\documentclass{article}

\usepackage{tikz}
\usepackage{bbm}
\usepackage{eqnarray}
\usepackage{color}
\usepackage{bm}
\usepackage{amssymb}
\usepackage{epsfig}
\usepackage{epsf}
\usepackage{float}
\usepackage{subfigure}
\usepackage{amsfonts,amsmath,amsthm,fancyhdr,multirow}
\usepackage{booktabs,longtable,authblk}
\usepackage{mathrsfs,hhline,enumerate}
\usepackage[authoryear,round]{natbib}
\usepackage{hyperref}

\usepackage[top=2.5cm, bottom=2.5cm, left=3cm, right=3cm]{geometry}
\newtheorem{theorem}{Theorem}[section]
\newtheorem{assumption}{Assumption}
\newtheorem{corollary}[theorem]{Corollary}

\newtheorem{lemma}[theorem]{Lemma}
\newtheorem{proposition}[theorem]{Proposition}
\newtheorem{remark}{Remark}

\allowdisplaybreaks[4]

\renewcommand{\H}{\mathcal{H}}

\newcommand{\A}{\mathcal{A}}
\newcommand{\T}{\mathcal{T}}
\newcommand{\B}{\mathcal{B}}

\newcommand{\E}{\mathcal{E}}

\newcommand{\X}{\mathcal{X}}
\newcommand{\Y}{\mathcal{Y}}
\newcommand{\K}{\mathcal{K}}
\newcommand{\Z}{\mathcal{Z}}
\renewcommand{\l}{\left}
\renewcommand{\r}{\right}
\newcommand{\HS}{\mathrm{HS}}

\newcommand{\be}{\mathbb{E}}
\newcommand{\bn}{\mathbb{N}}
\numberwithin{equation}{section}

\title{Learning Operators by Regularized Stochastic Gradient Descent with Operator-valued Kernels$^\dag$\footnotetext{\dag~The work described in this paper is supported by the National Natural Science Foundation of China under Grant Nos.~12171093 and 12571099. Email addresses: jqyang24@m.fudan.edu.cn (J.-Q. Yang), leishi@fudan.edu.cn (L. Shi). The corresponding author
is Lei Shi.}}

\author{Jia-Qi Yang}
\author{Lei Shi}

\affil{School of Mathematical Sciences and Shanghai Key Laboratory for
	Contemporary Applied Mathematics, Fudan University, Shanghai 200433, China.}
\date{}

\begin{document}

\maketitle

\begin{abstract}
We consider a class of statistical inverse problems involving the estimation of a regression operator from a Polish space to a separable Hilbert space, where the target lies in a vector-valued reproducing kernel Hilbert space induced by an operator-valued kernel. To address the associated ill-posedness, we analyze regularized stochastic gradient descent (SGD) algorithms in both online and finite-horizon settings. The former uses polynomially decaying step sizes and regularization parameters, while the latter adopts fixed values. 
Under suitable structural and distributional assumptions,
we establish prediction and estimation error bounds with
no explicit dependence on the dimension of the output space.
The resulting convergence rates are near-optimal in expectation, and we also derive high-probability estimates that imply almost sure convergence. Our analysis introduces a general technique for obtaining high-probability guarantees in infinite-dimensional settings. 
We illustrate the scope of our framework through applications
to structured prediction and a class of parametric
elliptic PDEs. 
For the latter, we construct kernels for arcsine and
uniform sampling, verify the assumptions of our high-probability
results, and obtain bounds uniform over finite parameter
truncations under suitable conditions on coefficient decay.
\end{abstract}


\section{Introduction} \label{introduction}

    In this work, we consider the nonlinear regression
\begin{equation} \label{model}
    y = h^\dagger(x) + \epsilon,
\end{equation}
where \( x \) belongs to a Polish space \(\mathcal{X}\), and \( y \) lies in a separable Hilbert space \(\mathcal{Y}\). The pair \((x,y)\) is distributed according to an unknown probability measure \(\rho\) on \(\mathcal{X} \times \mathcal{Y}\). The regression operator \( h^\dagger: \mathcal{X} \to \mathcal{Y} \) is defined as the conditional expectation
\[
h^\dagger(x) = \mathbb{E}[y|x],
\]
and the noise \(\epsilon\) is a centered \(\mathcal{Y}\)-valued random variable with finite variance $\sigma^2:=\be[\|\epsilon\|^2_{\Y}]$, assumed independent of \(x\). The objective is to recover \( h^\dagger \) from independent, identically distributed samples of \((x,y)\).

This formulation encompasses a broad class of statistical inverse problems in infinite-dimensional settings, where the forward map \(h^\dagger\) must be recovered from noisy, indirect observations. Traditional approaches to inverse problems often employ explicit regularization techniques, such as Tikhonov regularization or spectral filtering, to stabilize the solution, typically under certain assumptions on the forward operator's structure. In our work, we also use explicit regularization by penalizing the norm in a vector-valued reproducing kernel Hilbert space (RKHS). However, unlike classical deterministic frameworks, our approach studies stochastic gradient descent algorithms in infinite-dimensional function spaces with streaming data \footnote{Here, “streaming data” refers to a setting in which samples arrive sequentially and are processed one at a time, so that the estimator can be updated online without needing repeated access to the full dataset.}. This requires new analytical tools to establish convergence and error bounds under operator regularity conditions.

A key motivation for this model arises from structured output prediction, where the outputs exhibit functional or combinatorial structure. In many practical scenarios, the output space is not finite-dimensional but consists of complex objects such as functions, sequences, or graphs. To address these tasks, surrogate learning approaches embed the structured output into a Hilbert space and learn a continuous operator mapping inputs to embedded outputs. Applications include image completion \citep{weston2002kernel}, label ranking \citep{korba2018structured}, graph prediction \citep{brouard2016fast}, and protein function prediction \citep{ciliberto2020general}.
This framework also includes \emph{functional output regression}, which has become increasingly relevant with the growing availability of functional data in science and engineering. Here, both the input and output may be infinite-dimensional, motivating the development of operator-theoretic regression approaches. These applications have spurred the growth of the operator learning paradigm, which aims to learn mappings between function spaces directly from data. A prominent example is learning solution operators for parameterized partial differential equations (PDEs), where the operator maps coefficient or boundary condition functions to PDE solutions \citep{li2020fourier, bhattacharya2021model, lu2021learning,lei2024efficient}.
In this paper, we focus on learning algorithms that operate directly on function spaces rather than on finite-dimensional discretizations. This functional viewpoint preserves the intrinsic geometry and infinite-dimensional structure of the problem, enabling discretization-invariant analysis and broad applicability. Our setting naturally leads to the study of regularized learning in vector-valued RKHSs, where the regression operator is modeled within a rich nonparametric hypothesis space induced by operator-valued kernels.

    We introduce a supervised learning framework based on model \eqref{model}. Consider a data set $\{(x_t,y_t)\}_{t=1}^T$ generated by model \eqref{model}, or equivalently, drawn independently from the distribution $\rho$. To estimate $h^\dagger$, we minimize the regularized functional $\E(h)+\lambda \|h\|_{\H}^2$ over all $h\in\H$, where $\H$ is some Hilbert space, $\E(h):=\be\left[\|h(x)-y\|^2_\Y\right]$ denotes the mean squared error, and $\lambda>0$ is a regularization parameter. In this paper, we adopt a non-parametric approach to solve the nonlinear model \eqref{model}, assuming that $\H$ is a vector-valued reproducing kernel Hilbert space (RKHS) induced by an operator-valued kernel $K$ \citep{kadri2012multiple,kadri2016operator,brouard2016input,brogat2022vector}.

    To illustrate our algorithm, we introduce some notations along with basic concepts from operator theory \citep{conway2000course}. Consider a linear operator $A:\mathcal{H}_1\to\mathcal{H}_2$, where both $(\mathcal{H}_1,\langle\cdot,\cdot\rangle_{\mathcal{H}_1},\|\cdot\|_{\mathcal{H}_1})$ and $(\mathcal{H}_2,\langle\cdot,\cdot\rangle_{\mathcal{H}_2},\|\cdot\|_{\mathcal{H}_2})$ are Hilbert spaces. The set of bounded linear operators from $\mathcal{H}_1$ to $\mathcal{H}_2$ forms a Banach space under the operator norm $\|A\|= \sup _{\| f\| _{\mathcal{H}_1}= 1}\| Af\| _{\mathcal{H}_2}$, denoted by $\B( \mathcal{H}_1 , \mathcal{H}_2)$, or simply $\mathcal{B}(\mathcal{H}_1)$ when $\mathcal{H}_1=\mathcal{H}_2$. We call an operator $A\in\B(\H_1,\H_2)$ Hilbert-Schmidt if it holds $\sum_{k\geq1}\l\|Ae_i\r\|_{\H_2}^2<\infty$ for some (equivalently, any) orthonormal basis $\{e_k\}_{k\geq1}$ of $\H_1$. The set of Hilbert-Schmidt operators from $\H_1$ to $\H_2$ forms a Hilbert space under the Hilbert-Schmidt inner product $\langle A,B\rangle_{\mathrm{HS}}=\sum_{k\geq1}\langle Ae_k,Be_k\rangle_{\H_2}$ and the induced norm $\|\cdot\|_{\mathrm{HS}}$, denoted by 
    $\B_{\mathrm{HS}}(\H_1,\H_2)$. The adjoint of $A$, denoted by $A^*$, is the unique operator satisfying $\langle Af,f^{\prime}\rangle_{\mathcal{H}_2}=\langle f,A^*f^{\prime}\rangle_{\mathcal{H}_1}$ for all $f\in\mathcal{H}_1$ and $f^\prime\in\mathcal{H}_2$. If $A\in\B(\mathcal{H}_1,\mathcal{H}_2)$, then $A^*\in\B(\mathcal{H}_2,\mathcal{H}_1)$ and $\|A\|=\|A^*\|$. An operator $A\in\B(\mathcal{H}_1)$ is called self-adjoint if $A^*=A$, and positive if it is self-adjoint and satisfies $\langle Af,f\rangle_{\mathcal{H}_1}\geq0$ for every $f\in\mathcal{H}_1$. Let $(\H_{\mathcal{K}},\langle\cdot,\cdot\rangle_{\H_{\mathcal{K}}},\|\cdot\|_{\H_{\mathcal{K}}})$ denote the RKHS generated by the scalar-valued kernel $\mathcal{K}:\X\times\X\rightarrow \mathbb{R}$. Here, we say $K$ is a scalar-valued kernel if it is a real, symmetric, and positive-definite bivariate function. A mapping $K:\X\times\X\rightarrow \B(\Y)$ is called an operator-valued kernel \citep{schwartz1964sous, micchelli2005learning} on $\X$ if:
            \begin{enumerate}
                \item[(1)] For any $x,x'\in\X$, $K(x,x')$ is the adjoint operator of $K(x',x)$, i.e., $K(x,x')^*=K(x',x)$;
                \item[(2)] For any $n\in\bn$, $\{x_i\}_{i=1}^n\subset\X$ and $\{y_i\}_{i=1}^n\subset\Y$, it holds that $\sum_{i,j=1}^n\langle K(x_i,x_j)y_i,y_j\rangle_{\Y}\geq0$.
            \end{enumerate}
        Note that the function $K(x,\cdot)y:\X\rightarrow\Y$ is well-defined for $x\in\X$ and $y\in\Y$. The vector-valued RKHS $\H$ is the completion of the linear span of $\l\{K(x,\cdot)y:x\in\X,y\in\Y\r\}$ with inner product $\langle K(x,\cdot)y,K(x',\cdot)y'\rangle_{\H}=\langle K(x,x')y,y'\rangle_{\Y}$. Moreover, the reproducing property holds: 
        \[
        \langle K(x,\cdot)y, f\rangle_{\H}=\langle y,f(x)\rangle_{\Y}, \quad \forall (x,y)\in\X\times\Y \mbox{ and } f\in\H. 
        \]
        For more details on vector-valued RKHSs, see \citet{micchelli2005learning, carmeli2006vector, carmeli2010vector}. Furthermore, when $\Y=\mathbb{R}$, $K$ reduces to a scalar-valued kernel. 

     The construction of operator-valued kernels plays an important role in our setting.  A common choice is 
    \begin{equation} \label{kernel}
        K(x,x')=\K(x,x')W,
    \end{equation}
    where $\K$ is a scalar-valued kernel and $W\in\B(\Y)$ is a positive linear operator.  In multi-task learning, $W$ is typically a finite-dimensional matrix that facilitates information sharing among tasks \citep{evgeniou2005learning,caponnetto2008universal}. For some functional output learning problems, $W$ is selected to be a multiplication or an integral operator \citep{kadri2010nonlinear,kadri2011functional}. Additionally, some works on functional regression \citep{lian2007nonlinear} and structured output learning \citep{brouard2011semi,ciliberto2016consistent,ciliberto2020general,brogat2022vector} directly construct operator-valued kernels by setting $W$ to be the identity operator. In \citet{kadri2012multiple},  the kernels are taken to be the finite combinations of operator-valued kernels.  For other constructions, see \citet{kadri2016operator}.

        We briefly outline some algorithms for solving model \eqref{model}. 
        \citet{kadri2016operator} address model \eqref{model} using spectral decomposition of block operator matrices, while \citet{kadri2012multiple} propose a block-coordinate descent method. To handle limited training data, \citet{brogat2022vector} leverage the structure of the target output and propose a reduced-rank method to solve model \eqref{model}. Note that all these existing algorithms follow the batch learning \footnote{Here, “batch learning” refers to a setting in which the estimator is trained from a fixed dataset.} paradigms. In this paper, we adopt a stochastic gradient descent (SGD) algorithm derived from the Tikhonov regularization scheme to solve the model \eqref{model}, aiming to learn the nonlinear operators from streaming data. This algorithm is well-suited for real-time operator learning, enabling continuous adaptation without retaining historical data. This challenge has also been addressed in previous work on operator learning \citep{8815339, HOI2021249}. The performance of the resulting estimator $h$ can be evaluated using the prediction error $\E(h)-\E(h^\dagger)=\be\l[\l\|h(x)-h^\dagger(x)\r\|_\Y^2\r]$ and the estimation error $\|h-h^\dagger\|_{\H}^2$, where $\E(h):=\be[\|h(x)-y\|_{\Y}^2]$. To illustrate our algorithm, we define the minimizer of the regularized least squares problem as 
    \begin{equation} \label{lambda}
        h_\lambda:=\mathop{\arg\min}_{h\in\H}\l\{\E(h)+\lambda\|h\|_\H^2\r\},
    \end{equation}
    where $\lambda>0$ is the regularization parameter. This paper focuses on two important settings of the SGD algorithm: one with constant step sizes and regularization parameters, and the other with decaying step sizes and regularization parameters. Hereinafter, we use $\mathbf{0}$ to denote the zero element in a Hilbert space.
    
    \textbf{The ﬁnite-horizon setting.} In this setting, we assume access to finite i.i.d. samples $\{z_t=(x_t,y_t)\}_{t=1}^T$,  where the sample size $T<\infty$ is known in advance. We aim to solve the regularized problem \eqref{lambda}, where the parameter $\lambda$ depends on $T$. The SGD algorithm proceeds by updating the current estimator $h_{t}$ to $h_{t+1}$ using a single sample at the $t$-th iterate with a constant step size and regularization parameter. Specifically, the iteration begins with $h_1=\mathbf{0}$ and is recursively defined as
    \begin{equation} \label{iter3}
        h_{t+1}=h_t-\eta_T\left(K(x_t,\cdot)(h_t(x_t)-y_t)+\lambda_T h_t\right),  \quad t=1,\cdots T,
    \end{equation}
    where the step size (learning rate) $\eta_T$ and the regularization parameter $\lambda_T$ are appropriately chosen based on the sample size $T$. The update in iteration \eqref{iter3} arises from a one-sample stochastic approximation of $2\be\l[K(x,\cdot)(h(x)-y)\r]+2\lambda h$, which corresponds to the Fr$\acute{e}$chet derivative \citep{dunford1988linear} of $\E(h)+\lambda\|h\|_\H^2$. Implementing an efficient warm start can be non-trivial when new data points become available in the future.
    
    \textbf{The online setting.} In this setting, the sample size $T$ may be unknown in advance or even infinite, which is well-suited for scenarios that require real-time iterative updates.  To accommodate this setting, we update the regularization parameter $\lambda_t$ such that $h_{t+1}$ follows the regularization path \citep{zhu2021algorithmic} $h_{\lambda_t}$,\footnote{Regularization path refers to the trajectory of solutions $h_\lambda$ as the regularization parameter $\lambda$ varies, characterizing how the learned model evolves under different levels of regularization.} ensuring that $h_t-h_{\lambda_t}\rightarrow\mathbf{0}$ and $h_{\lambda_t}\rightarrow h^\dagger$ in the norm $\|\cdot\|_\H$ (or in the semi-norm associated with prediction error) as $t$ increases. This leads to the following iterative scheme, initialized with $h_1=\mathbf{0}$:
    \begin{equation} \label{iter2}
        h_{t+1}=h_t-\eta_t\left(K(x_t,\cdot)(h_t(x_t)-y_t)+\lambda_t h_t\right), \quad t\geq 1.
    \end{equation}
    Moreover, we let both $\eta_t$ and $\lambda_t$ decay polynomially with respect to $t$, enabling stability and convergence of the solution while adapting to streaming data and mitigating overfitting.

In this paper, we study both settings of the SGD algorithm. We express the iterative forms \eqref{iter3} and \eqref{iter2} in a unified manner as the form given in \eqref{iter2}. For decaying step sizes adopted in the online setting, we set the step size as $\eta_t = \bar{\eta}(t + t_0)^{-\theta_1}$ for all $t \geq 1$, where $\theta_1 \in (0,1)$, $\bar{\eta} > 0$, and $t_0 > 0$. The regularization parameter is defined as $\lambda_t = \bar{\lambda}(t + t_0)^{-\theta_2}$, where $\theta_2 \in (0,1)$, $\bar{\lambda} > 0$. We emphasize that in the online setting, both $\bar{\eta}$ and $\bar{\lambda}$ are constants independent of $t$ and the total number of iterations (e.g., the sample size) $T$. For the constant step sizes and regularization parameters adopted in the finite-horizon setting, we set $\eta_t = \eta_\star T^{-\theta_3}$ and $\lambda_t = \lambda_\star T^{-\theta_4}$ for $t = 1, 2, \dots, T$, where $\theta_3 \in (0,1)$, $\theta_4 > 0$, and $\eta_\star, \lambda_\star> 0$. In this finite-horizon setting, the step sizes and regularization parameters explicitly depend on the total number of iterations $T$. 

Throughout the paper, we impose the following assumption on the operator-valued kernels.
    \begin{assumption} \label{a1}
            The vector-valued RKHS $\H$ is generated by the operator-valued kernel $K(x,x')=\K(x,x')I$, where
            $\K$ is a continuous scalar-valued kernel with $\|\K\|_\infty:=\sup_{x\in\X}\K(x,x)\leq\kappa^2$ for some constant $\kappa>0$, and $I$ is the identity operator on $\Y$.
    \end{assumption}
In particular,
$\|\mathcal K(x,\cdot)\|_{\mathcal H_{\mathcal K}}\le\kappa$
for all $x\in\mathcal X$. This simple construction of operator-valued kernels has been adopted in previous works, e.g., \citep{brogat2022vector,batlle2024kernel}. Note that all elements $h$ in the vector-valued RKHS $\H$ are measurable. Furthermore, when choosing kernels as in \eqref{kernel}, it follows from \citet[Example 5]{carmeli2010vector} that $K$ is a Mercer [resp. $\mathcal{C}_0$]\footnote{That is, the Banach space of continuous functions vanishing at infinity with the uniform norm.} kernel if $\K$ is Mercer [resp. $\mathcal{C}_0$], implying that all operators in $\H$ are continuous. This choice of kernel, $K(x,x') = \mathcal K(x,x') I$, provides a mathematically tractable and commonly adopted setting in which the vector-valued RKHS structure allows for a complete convergence analysis of the regularized SGD algorithm in infinite-dimensional operator learning. In particular, each output component is associated with the same scalar kernel $\mathcal K$, which facilitates the analysis while preserving the infinite-dimensional nature of the operator-valued RKHS.

At the same time, this construction does not explicitly encode correlations between different output directions. More general operator-valued kernels, such as those incorporating a compact operator $W$, can, in principle, capture such dependencies. We discuss these generalizations in Appendix~\ref{Section 3.1}, though a complete convergence analysis in these settings is beyond the scope of the present work. This discussion emphasizes both the practical relevance and the inherent limitations of the identity-operator construction used throughout the main text.

The two types of step sizes considered in this paper have been extensively studied in the previous literature on SGD in various settings. The seminal work of \citet{smale2006online} shows that the step size serves as an implicit form of regularization, thus improving the algorithm's generalization and robustness. In \citet{shi2024learning}, we investigate operator learning via the SGD algorithm between Hilbert spaces, deriving bounds for both prediction and estimation errors in expectation. However, the SGD algorithm in \citet{shi2024learning} does not incorporate the regularization term. On the other hand, in the context of operator learning, research on the almost-sure convergence of SGD algorithms is still scarce. Only a few works, including \citet{tarres2014online, berthier2020tight, varre2021last}, have considered the almost-sure convergence in finite-dimensional output settings but either assume a noise-free scenario or provide convergence results that do not directly extend to operator learning problems.  This clearly identifies a significant gap in the existing literature. Consequently, while advancements have occurred in finite-dimensional settings, substantial challenges related to operator learning and the fundamental role of regularization remain largely unexplored.

This paper aims to fill this gap by rigorously analyzing the regularized SGD algorithm applied to the nonlinear operator regression problem described in \eqref{model}. Our main contributions are summarized as follows: First, we introduce specific regularity assumptions on $h^\dagger$ (or its associated Hilbert-Schmidt operator, as defined in Proposition \ref{transform}), which effectively capture the intrinsic features of infinite-dimensional regression problems. Under these assumptions, we derive bounds in expectation for the prediction and estimation errors of the regularized SGD algorithm, showing improvements compared to the unregularized SGD algorithm studied in \citet{shi2024learning}. Second, we propose a novel technique for establishing high-probability bounds, ensuring almost-sure convergence via the Borel–Cantelli lemma. High-probability convergence provides a stronger guarantee than expectation-based bounds alone, moving beyond average-case performance. Crucially, we demonstrate that the introduction and careful tuning of regularization parameters are essential not only for achieving these high-probability bounds but also for significantly enhancing the convergence behaviors of the SGD algorithm. This underscores the superiority of our regularized approach over the unregularized framework considered in \citet{shi2024learning}, highlighting the necessity of regularization for robust probabilistic guarantees. Lastly, the resulting convergence rates are demonstrated to be near-optimal, aligning closely with the minimax lower bounds established in \citet{shi2024learning}, thus reinforcing the theoretical soundness and effectiveness of our proposed algorithm. 
The bounds have no explicit dependence on the dimension of $\Y$,
with constants determined by the structural and
distributional quantities in the assumptions.
For the class of parametric elliptic PDEs in
Section~\ref{sec:pde-application}, we quantify the dependence
on the number of parameters and obtain bounds uniform over
finite parameter truncations under suitable conditions on
coefficient decay.

The rest of the paper is organized as follows. Section \ref{results} introduces the main theoretical results and assumptions, and briefly discusses the key error decomposition and the relation to the unregularized analysis. Section \ref{new dis} provides further discussion of the framework and its broader context. In Section \ref{discussion}, we illustrate the practical scope of our framework through applications to structured prediction and parametric PDEs. Section \ref{Numerical experiments} presents a brief numerical experiment. Section \ref{decomposition} performs an error decomposition tailored to the regularized SGD algorithm. Building on this, Section \ref{section:basic} provides essential intermediate estimates used in the subsequent analysis. Sections \ref{expected} and \ref{prob} are devoted to establishing bounds on the prediction and estimation errors, first in expectation and then with high probability. For clarity and conciseness, some technical proofs are presented in the appendix. Appendix \ref{Section 3.1} discusses more general operator-valued kernels, while Appendix \ref{Section 3.3} addresses the connection with the PCA encoder-decoder framework.

\section{Main Results} \label{results}
        
This section introduces regularity conditions on the structure and smoothness of the target operator and the input random variables. We then present our main theorems. We begin with some notations for further statements. Denote by $\mathbb{N}_T$ the set $\{1,2,\cdots,T\}$. The rank-one operator $f \otimes g \in \B(\H_1, \H_2)$ is defined by $f \otimes g(g') := \langle g, g' \rangle_{\H_1} f$, where $g, g' \in \H_1$ and $f \in \H_2$. We denote by $\mathrm{Tr}(A)$ the trace of a self-adjoint and compact operator $A \in \B(\H_1)$. Let $\be$ and $\be_{z_t}$ denote the expectation with respect to the distribution $\rho$ and the sample $z_t := (x_t, y_t)$, respectively. For $k \in \mathbb{N}_T$, let $\mathbb{E}_{z_1, \cdots, z_k}$ denote the expectation with respect to $\{z_i\}_{i=1}^k$, abbreviated as $\mathbb{E}_{z^k}$. Recall that $\K$ is the scalar-valued kernel. Since $\X$ is separable and $K$ is continuous, the RKHS $\H_\K$ induced by $\mathcal{K}$ is also separable. The operator $C=\be[\phi(x)\otimes\phi(x)]$, defined by  $\phi(x):=\mathcal{K}(x,\cdot)\in\H_\K$, is self-adjoint, compact, and satisfies $\|C\|\leq\|C\|_{\mathrm{HS}}\leq\kappa^2$. Thus, for any $r > 0$, the operator $C^r$ is also self-adjoint and compact. Moreover, it is straightforward to verify that
    \[
    \|C^{1/2}\|^2_{\mathrm{HS}}=\mathrm{Tr}(C)=\be\left[\|\phi(x)\|_{\H_\K}^2\right]\leq\kappa^2.
    \] 
With the aid of the following proposition, the iterative process in the vector-valued RKHS $\H$ can be equivalently reformulated as an iterative process in $\B_{\mathrm{HS}}(\H_\K, \Y)$.

\begin{proposition} \label{transform}
		The vector-valued RKHS $\H$, associated with the operator-valued kernel $K(x,x^\prime)=\mathcal{K}(x,x^\prime)W$, where $W$ is a positive operator and $\mathcal{K}$ is a scalar-valued kernel, is isometrically isomorphic to $\B_{\mathrm{HS}}(\H_\mathcal{K},\overline{W^{1/2}\Y})\subset\B_{\mathrm{HS}}(\H_\mathcal{K},\Y)$. Specifically, for each $h\in\H$, there exists a unique $H\in\B_{\mathrm{HS}}(\H_\mathcal{K},\overline{W^{1/2}\Y})$ such that 
		\begin{equation*}
			h(x)=W^{1/2}H\phi(x),\quad \forall x\in\X,
		\end{equation*}
		and $\|h\|_{\H}=\|H\|_{\mathrm{HS}}$.
	\end{proposition} 
The proof of Proposition \ref{transform} is deferred to Appendix \ref{tran}. When Proposition \ref{transform} is applied with $W = I$, the kernel coincides with that in Assumption \ref{a1}, and the iteration \eqref{iter2} can be equivalently expressed as
        \begin{equation} \label{iter}
            \begin{cases}
			H_1=\mathbf{0}, \\
			H_{t+1}=H_t-\eta_t\left(\l(H_t\phi(x_t)-y_t\r)\otimes \phi(x_t)+\lambda_t H_t\right), \\
			h_t(\cdot)=H_t\l(\phi(\cdot)\r).
		\end{cases}
        \end{equation}
Hereinafter, we assume that $h^\dagger(x)=H^\dagger\phi(x)$, where $H^\dagger\in\B_{\mathrm{HS}}(\H_\K,\Y)$ is a Hilbert-Schmidt operator. Under this assumption, the nonlinear operator learning model \eqref{model}  reduces to an infinite-dimensional linear model:
\begin{equation} \label{model'}
    y=H^\dagger\phi(x)+\epsilon,
\end{equation}
where the input and output are $\phi(x)$ and $y$, respectively. For $H\in\B_{\mathrm{HS}}(\H_\K,\Y)$, let $\E(H)=\mathbb{E}\left[\|y-H\phi(x)\|_{\Y}^{2}\right]$, where $h(x)=H\phi(x)$, so that $\E(H)=\E(h)$. The prediction error is defined as $\E(H)-\E(H^\dagger)=\E(h)-\E(h^\dagger)$, and the estimation error is defined as $\be[\|H-H^\dagger\|^2_{\mathrm{HS}}]$.
According to Proposition \ref{transform}, we have $\be[\|h-h^\dagger\|^2_{\H}]=\be[\|H-H^\dagger\|^2_{\mathrm{HS}}]$, which implies that the prediction error and estimation error for the estimator $h(\cdot)=H\phi(\cdot)$ in the original model \eqref{model} coincide with those of $H$ in the linearized model \eqref{model'}. Therefore, it suffices to analyze the convergence rates of the errors of $H_t$ associated with the SGD iteration \eqref{iter} in $\B_{\mathrm{HS}}(\H_\mathcal{K}, \Y)$.  Although this is not directly required for our theoretical analysis, we emphasize for clarity that the iterative form of SGD derived from minimizing the regularized objective functional $\E(H) + \lambda \|H\|_{\mathrm{HS}}^2 $ corresponds exactly to the iteration given in \eqref{iter}, which is equivalent to \eqref{iter2}.

 \subsection{Assumptions} \label{sub1}

To conduct the convergence analysis, we need the following assumptions. 
	\begin{assumption}[Regularity condition of $H^\dagger$] \label{a2} There exists a Hilbert-Schmidt operator $S^\dagger\in\B_{\mathrm{HS}}(\H_\K,\Y)$ and a positive parameter $r>0$, such that:
		\begin{equation*} 
			H^\dagger=S^\dagger C^r.
		\end{equation*}
	\end{assumption}
This assumption, introduced in \citet{shi2024learning}, characterizes the regularity of the target operator $H^\dagger$ via its relation to the operator $C$. The parameter $r$ serves as a smoothness index—larger values of $r$ indicate higher regularity of $H^\dagger$.
In the special case where $\Y = \mathbb{R}$, the Riesz representation theorem implies that $H^\dagger$ corresponds to an element $g^\dagger = C^r g$ in $\H_\K$ for some $g \in \H_\K$. 
This is exactly the regularity condition widely adopted in the convergence analysis of non-parametric regression in real-valued RKHSs \citep{ying2008online, dieuleveut2016nonparametric, berthier2020tight, guo2023capacity}.

\begin{assumption}[Spectral decay condition of $C$] \label{a3} There exists $s\in(0,1]$ such that:
	\begin{equation*}
		\mathrm{Tr}(C^s) < +\infty.
	\end{equation*}
\end{assumption}
This condition is automatically satisfied for any $s\geq1$ (as $\mathrm{Tr}(C)<\kappa^2$), and it imposes constraints on the decay rate of the eigenvalues of the operator $C$. Let $\{a_k\}_{k\geq1}$ denote
the non-increasing sequence of eigenvalues of $C$. Under this condition, the eigenvalues exhibit polynomial decay, specifically satisfying
\[
a_k\leq\mathrm{Tr}(C^s)^{\frac{1}{s}}k^{-\frac{1}{s}}.
\]
A sufficient (though not necessary) condition for this assumption is that
$a_k=O(k^{-\frac{1}{s}-\epsilon})$ for some $\epsilon>0$. For a detailed discussion on this condition, we refer the reader to \citet{guo2023capacity}. In this paper, Assumption \ref{a3} with some $0<s<1$ is introduced to derive sharper error bounds. Together with Assumption \ref{a2},
Assumption \ref{a3} leads to improved convergence rates. This condition, commonly known as the capacity condition, was first introduced in \citet{dieuleveut2017harder} and has since been widely adopted in the literature, including \citet{pillaud2018statistical, brogat2022vector, guo2023capacity, shi2024learning}, as a way to capture the intrinsic complexity of infinite-dimensional learning problems. 
The convergence rates depend explicitly on the regularity
parameter $r$ in Assumption~\ref{a2} and the capacity parameter
$s$ in Assumption~\ref{a3}.
In Section~\ref{sec:pde-application}, we verify these assumptions
for a class of parametric elliptic PDEs using
suitable kernels.

The following assumption is only required for establishing error bounds in expectation. Recall that  $\phi(x):=\mathcal{K}(x,\cdot)\in\H_\K$ for some scalar-valued kernel $\K$.
\begin{assumption}[Moment condition of $\phi(x)$] \label{a4}
There exists a constant $c>0$ such that for any bounded linear operator $A\in\B(\H_\K)$,
\begin{equation*}
	\mathbb{E}\left[\left\|A \phi(x)\right\|^4_{\H_\K}\right]\leq c\left(\mathbb{E}\left[\|A \phi(x)\|^2_{\H_\K}\right]\right)^2.
\end{equation*}				
\end{assumption}
The following scalar moment condition is equivalent to
Assumption~\ref{a4}: 
\begin{equation} \label{a4'}
   \be\l[\l\langle\phi(x),f\r\rangle_{\H_\K}^4\r]\leq c\l(\be\l[\l\langle\phi(x),f\r\rangle_{\H_\K}^2\r]\r)^2, \quad \forall f\in\H_\K.
\end{equation}
The equivalence follows from Proposition~\ref{lemma in Appendix} in
Appendix~\ref{Appendix lemma}; see also \citet[Proposition 2.1]{shi2024learning}
for the corresponding formulation used there.

Condition \eqref{a4'} holds, for example, when $\phi(x)$ is strictly sub-Gaussian, implying that all linear functionals of $\phi(x)$ have
bounded kurtosis. Similar assumptions have been adopted in several papers \citep{yuan2010reproducing,cai2012minimax,guo2023capacity,shi2024learning}. To further deepen our understanding, we now present a novel sufficient condition for Assumption \ref{a4}, which is analogous to the idea discussed in \citet{liu2024statistical}.
\begin{proposition} \label{prop14}
Consider the principal component decomposition of $\phi(x)$:
    \begin{equation} \label{temp94}
        \phi(x)=\overline\phi+\sum_{k\geq1}\sqrt{\lambda_k}\xi_k\phi_k,
    \end{equation}
    where $\overline\phi:=\be[\phi(x)]$, and $\l\{\l(\lambda_k,\phi_k\r)\r\}_{k\geq1}$ are the eigenvalue-eigenvector pairs of the covariance operator $\Sigma:=\be\l[\l(\phi(x)-\overline\phi\r)\otimes\l(\phi(x)-\overline\phi\r)\r]$. The sequence $\{\xi_k\}_{k\geq1}$  consists of zero-mean, uncorrelated real-valued random variables with $\be[\xi_k^2] = 1$.
   If, in addition, $\{\xi_k\}_{k \geq 1}$ are independent, then Assumption \ref{a4} (or equivalently, \eqref{a4'}) holds provided that $\{\be[\xi_k^4]\}_{k \geq 1}$ are uniformly bounded. That is, there exists a constant $C > 0$ such that
    \[
    \be\l[\xi_k^4\r]\leq C,\quad \forall k\geq 1.
    \]
\end{proposition} The proof of the above proposition is presented in Appendix \ref{Proof of prop14}. The next assumption is used to derive high-probability error bounds.
\begin{assumption}[Boundedness condition of $y$] \label{a5}
    There exists some constant $M_\rho>0$ such that 
    \[\|y\|_\Y \leq M_\rho\]
    almost surely.
\end{assumption}

Assumption \ref{a5} applies, in particular, to the surrogate approach to structured prediction
\citep{brouard2016input,ciliberto2016consistent,korba2018structured,ciliberto2020general,brogat2022vector},
see Subsection \ref{Section 3.2}. It is also relevant to PDE operator-learning settings, including
benchmark examples such as Darcy-type flow and Navier--Stokes equations
\citep{kovachki2021universal,lanthaler2022error,lanthaler2023operator}, provided that the
corresponding solutions satisfy suitable uniform a priori bounds.

\subsection{Error Bounds in Expectation} \label{sub2}

In this subsection, we assume that Assumption \ref{a1} holds, Assumption \ref{a2} holds with $S^\dagger\in \B_{\mathrm{HS}}(\H_\K,\Y)$ and $r>0$, Assumption \ref{a3} holds with $0<s\leq1$, and Assumption \ref{a4} holds with $c>0$. Theorem \ref{thm1} and Theorem \ref{thm2} provide the convergence rates of prediction error and estimation error in expectation for the online setting. In contrast, Theorem \ref{thm3} and \ref{thm4} focus on the finite-horizon setting.


For notational simplicity, we specify here the parameter
dependence of the constants appearing in the following theorems.
The constants $c_{1,1}$ and $c_{1,2}$ depend
only on $\bar\eta$, $\bar\lambda$, $\theta_1$, $t_0$, $r$, $s$,
$\kappa$, $\mathrm{Tr}(C^s)$,
$\|S^\dagger\|_{\mathrm{HS}}$, $\sigma^2$, and $c$.
The constants $c_{1,3}$ and $c_{1,4}$ depend
only on $\eta_\star$, $\lambda_\star$, $\theta_3$, $\theta_4$,
$r$, $s$, $\kappa$, $\mathrm{Tr}(C^s)$,
$\|S^\dagger\|_{\mathrm{HS}}$, $\sigma^2$, and $c$.
All these constants are independent of $T$.

\begin{theorem} \label{thm1}
    Suppose that Assumption \ref{a1}, Assumption \ref{a2}, Assumption \ref{a3} and Assumption \ref{a4} are satisfied. Define $\{h_t\}_{t\geq1}$ through \eqref{iter} with step sizes $\{\eta_t=\bar\eta (t+t_0)^{-\theta_1}\}_{t\geq1}$ and regularization parameters $\{\lambda_t=\bar\lambda (t+t_0)^{-\theta_2}\}_{t\geq1}$, where $0<\theta_1<1$, $0<\theta_2<1$ and $\bar\eta\bar\lambda>\theta_2\min\{r,1\}$. Additionally, let $t_0$ satisfy $(t_0+1)^{\theta_1}\geq\bar\eta(\kappa^2+\bar\lambda)$, $t_0\geq\exp\{\frac{1}{\theta_1}\}$, and
    \begin{equation*}
        c_4\sqrt{c}t_0^{-\theta_1}\log t_0<1,
    \end{equation*}
    where $c_4=c_4(\bar\eta,\bar\lambda,\theta_1,s,c,\kappa,\mathrm{Tr}(C^s))$ is a constant independent of $t_0$, as specified in Proposition \ref{prop5}. Choose $\theta_1=\min\l\{\frac{2r+1}{2r+2},\frac{2}{3}\r\}$ and $\theta_2=1-\theta_1$. Then for any $T\geq1$,
    \begin{equation*}
        \be_{z^T}\left[\E(h_{T+1})-\E(h^\dagger)\right]\leq c_{1,1}
        \begin{cases}
            (T+t_0)^{-\theta_1}, & \text{when } s<1, \\
            (T+t_0)^{-\theta_1}\log(T+t_0), & \text{when } s=1. \\
        \end{cases}
    \end{equation*}
    Here the constant $c_{1,1}$ is independent of\/ $T$.
\end{theorem}

\begin{remark}
    In the theorem above, we set $\theta_1 + \theta_2 = 1$, as this choice leads to the most favorable convergence rates achievable within our framework. The condition on $t_0$ is necessary for the proof, while the constraint on $\bar\eta\bar\lambda$ serves to accelerate convergence. When $\theta_1 + \theta_2 \neq 1$, the resulting rates are slower. In such cases, one can set $t_0 = 0$ and choose a small $\bar\eta\bar\lambda$; the corresponding analysis is similar and more straightforward, so we omit it for brevity.
\end{remark}

In Theorem \ref{thm1}, since the constant $c_4$ is independent of $t_0$, one can choose $t_0$ sufficiently large to satisfy the required conditions. Compared to Assumption \ref{a3} with $s=1$, the stronger assumption with $0<s<1$ only removes a logarithmic factor in the convergence rate. It is also clear that the convergence rate saturates at $r = 1/2$, i.e., increasing $r$ beyond $1/2$ does not yield further improvement. This saturation phenomenon under increasing source regularity is not specific to the present setting and has also been observed in related kernel-learning literature; see, for example, \citet{li2024saturation}. According to \citet[Theorem 2.9]{shi2024learning}, the result is minimax optimal (up to a logarithmic term) when $s=1$ and $r<1/2$. Compared to the unregularized SGD algorithm analyzed in \citet[Theorem 2.4]{shi2024learning}, adding a regularization term here leads to faster convergence. Specifically, while the prediction error rate of unregularized SGD in \citet{shi2024learning} saturates at $r = (1-s)/2$, the regularized SGD in our work improves the saturation level to $r = 1/2$.

The following theorem provides the convergence rate for the estimation error.

\begin{theorem} \label{thm2}
    Under the conditions of Theorem \ref{thm1}, choose $\theta_1=\min\l\{\frac{s+2r}{1+s+2r},\frac{2+s}{3+s}\r\}$ and $\theta_2=1-\theta_1$. Then for any $T\geq1$,
    \begin{equation*}
                \be_{z^T}\left[\left\|h_{T+1}-h^\dagger\right\|_{\H}^2\right]\leq c_{1,2}
                (T+t_0)^{-\min\l\{\frac{2r}{1+s+2r},\frac{2}{3+s}\r\}}.
    \end{equation*}
    Here the constant $c_{1,2}$ is independent of\/ $T$.
\end{theorem}

The convergence rate of the estimation error saturates at  $r=1$, which improves the convergence of unregularized SGD in \citet[Theorem 2.6]{shi2024learning}, where the rate saturates at 
$r=\frac{1-s}{2}$. Moreover, in \citet[Theorem 3]{guo2023capacity} and \citet[Theorem 2.6]{shi2024learning}, we cannot guarantee the convergence of the estimation error with decaying step sizes for $s=1$, whereas adding a regularization term addresses this issue. According to \citet[Theorem 2.9]{shi2024learning}, the convergence rate with decaying step sizes is minimax optimal when $r<1$.

Next, we present the convergence rates for prediction and estimation errors with constant step sizes and regularization parameters, where both depend on the total number of iterations $T$ (i.e., the total sample size).

    \begin{theorem}\label{thm3}
        Suppose that Assumption \ref{a1}, Assumption \ref{a2}, Assumption \ref{a3} and Assumption \ref{a4} are satisfied. Define $\{h_t\}_{t\in\bn_T}$ through \eqref{iter} with step sizes $\{\eta_t=\eta_\star T^{-\theta_3}\}_{t\in\bn_{T}}$ and regularization parameters $\{\lambda_t=\lambda_\star T^{-\theta_4}\}_{t\in{\bn_T}}$, where $T\geq2$,
        $\eta_\star(\kappa^2+\lambda_\star)\leq1$, and
            \[\eta_\star<\frac{1}{6c\kappa^2\left(1+\frac{1}{2e\theta_3}\right)}.
            \]
        Choose $\theta_3=\frac{2r+1}{2r+2}$ and $\theta_4\geq\frac{2r+1}{(2r+2)\min\{2r+1,2\}}$. Then
            \begin{equation*}
                \be_{z^{T}}[\mathcal{E}(h_{T+1})-\mathcal{E}(h^\dagger)]\leq
                c_{1,3}\begin{cases}
                    T^{-\frac{2r+1}{2r+2}}, & \text{ when }s<1, \\
                    T^{-\frac{2r+1}{2r+2}}\log T , & \text{ when } s=1.
                \end{cases}
            \end{equation*}
            Here the constant $c_{1,3}$ is independent of \/$T$.
    \end{theorem}

    \begin{theorem} \label{thm4}
    Under the conditions of Theorem \ref{thm3}, choose 
    \[
    \theta_3=\frac{2r+s}{1+2r+s}\ \text{ and }\ \theta_4\geq\frac{2r}{(1+2r+s)\min\{2r,2\}}.
    \]
    Then
            \[
            \be_{z^{T}}\left[\left\|h_{T+1}-h^\dagger\right\|_{\H}^2\right]
		\leq c_{1,4}T^{-\frac{2r}{1+2r+s}}.
            \]
        Here the constant $c_{1,4}$ is independent of \/$T$.
    \end{theorem}

In the case of constant step sizes and regularization parameters, the prediction error achieves the minimax optimal rate when $s=1$, and the estimation error performs so for any $s>0$, as established by the minimax lower bounds in \citet{shi2024learning}. Unlike the scenario with decaying step sizes and regularization parameters, no saturation occurs when these parameters are held constant. It is also noteworthy that the convergence rates and the choice of $\theta_3$ in the above two theorems align with those in Theorems 2.5 and 2.7 of \citet{shi2024learning}, which analyze the unregularized SGD algorithm with constant step sizes. This contrasts with the case of decaying step size, where adding a regularization term leads to improved rates. When employing constant step sizes, introducing regularization does not improve the convergence rates; in fact, an improperly chosen 
$\theta_4$ (not sufficiently large) may degrade performance. The unregularized SGD can be viewed as the limiting case corresponding to $\theta_4=\infty$.

\subsection{High-probability Error Bounds} \label{sub3}

In this subsection, we assume Assumption \ref{a1} holds, Assumption \ref{a2} holds with $S^\dagger\in \B_{\mathrm{HS}}(\H_\K,\Y)$ and $r>0$, Assumption \ref{a3} holds with $0<s\leq1$, and Assumption \ref{a5} holds with $M_\rho>0$. We derive high-probability error bounds for both the prediction and estimation errors in both the online and finite-horizon settings. These error bounds guarantee almost-sure convergence of the regularized SGD algorithm, providing a stronger guarantee than bounds in expectation, as convergence is ensured with high probability across all realizations. The notation $a \lesssim b$ denotes $a \leq Cb$ for some constant $C$ independent of $t$, $T$, and the confidence level $\delta$.


For notational simplicity, we specify here the parameter
dependence of the constants appearing in the following
theorems and corollaries.
The constants $c_{2,1}$, $\widetilde{c}_{2,1}$,
$c_{2,2}$, and $\widetilde{c}_{2,2}$ depend
only on $\bar\eta$, $\bar\lambda$, $\theta_1$, $t_0$, $r$, $s$,
$\kappa$, $\mathrm{Tr}(C^s)$,
$\|S^\dagger\|_{\mathrm{HS}}$, and $M_\rho$.
The constants $c_{2,3}$ and $c_{2,4}$ depend
only on $\eta_\star$, $\lambda_\star$, $\theta_3$, $\theta_4$,
$r$, $s$, $\kappa$, $\mathrm{Tr}(C^s)$,
$\|S^\dagger\|_{\mathrm{HS}}$, and $M_\rho$.
All these constants are independent of $t$, $T$, and $\delta$.

    The following theorem establishes the prediction error bounds in the online setting.

    \begin{theorem} \label{thm5}
        Suppose that Assumption \ref{a1}, Assumption \ref{a2}, Assumption \ref{a3} and Assumption \ref{a5} are satisfied. Define $\{h_t\}_{t\geq1}$ through \eqref{iter} with step sizes $\{\eta_t=\bar\eta (t+t_0)^{-\theta_1}\}_{t\geq1}$ and regularization parameters $\{\lambda_t=\bar\lambda (t+t_0)^{-\theta_2}\}_{t\geq1}$, where $\bar\eta\bar\lambda>\max\{\theta_2\min\{r,1\},\theta_1,2\theta_1-\frac12\}$ and $(t_0+1)^{\theta_1}\geq\bar\eta(\kappa^2+\bar\lambda)$.
        Choose \begin{equation*}
            \theta_1=
            \begin{cases}
                \frac{2r+1}{2r+2}, & \text{ when } r<\frac12, \\[5pt]
                \frac23, & \text{ when } r\geq\frac12,
            \end{cases}
        \end{equation*}
        and $\theta_2=1-\theta_1$. Then for any $T\geq1$ and $\delta\in(0,2/e)$, with probability at least $1-2\delta$, the following holds:
        \begin{itemize}
            \item[(1)] If $s<1$,
            \begin{equation*}
                \begin{aligned}
                    \E(h_{T+1})-\E(h^\dagger)
            &\leq c_{2,1}\l((T+t_0)^{-\theta_1}+(T+t_0)^{1-3\theta_1}\log^2(T+t_0)\log^2\frac2\delta\r)\log^2\frac2\delta
            \\ &\lesssim (T+t_0)^{-\theta_1}\log^4\frac2\delta.
                \end{aligned}
            \end{equation*}
            \item[(2)] If $s=1$,
            \begin{equation*}
                \begin{aligned}
                    \E(h_{T+1})-\E(h^\dagger)
            &\leq c_{2,1}\l((T+t_0)^{-\theta_1}+(T+t_0)^{1-3\theta_1}\log^2(T+t_0)\log^2\frac2\delta\r)\log(T+t_0)\log^2\frac2\delta
            \\ &\lesssim (T+t_0)^{-\theta_1}\log(T+t_0)\log^4\frac2\delta.
                \end{aligned}
            \end{equation*}
        \end{itemize}
         Here the constant $c_{2,1}$ is independent of \/$T$ and $\delta$.
    \end{theorem}
    The following corollary, as a natural extension of Theorem~\ref{thm5}, establishes a uniform high-probability bound that holds simultaneously for all  $t \geq 1$.

    \begin{corollary} \label{coro1}
        Under the conditions of Theorem \ref{thm5},
        choose \begin{equation*}
            \theta_1=
            \begin{cases}
                \frac{2r+1}{2r+2}, & \text{ when } r<\frac12, \\[5pt]
                \frac23, & \text{ when } r\geq\frac12,
            \end{cases}
        \end{equation*}
        and $\theta_2=1-\theta_1$. Then, with probability at least $1-2\delta$ , for all $1\leq t<\infty$, the following holds:
        \begin{equation*}
            \E(h_{t+1})-\E(h^\dagger)
            \leq \widetilde{c}_{2,1}
            \begin{cases}
                (t+t_0)^{-\theta_1}\log^4(t+t_0)\log^4\frac2\delta, & \text{ when } s<1, \\[5pt]
                (t+t_0)^{-\theta_1}\log^5(t+t_0)\log^4\frac2\delta, & \text{ when } s=1.
            \end{cases}
        \end{equation*}
         Here the constant $\widetilde{c}_{2,1}$ is independent of\/ $t$ and $\delta$.
    \end{corollary}

    In Theorem \ref{thm8} and Corollary \ref{coro2}, we focus on the estimation error in the online setting.

    \begin{theorem} \label{thm8}
        Under the conditions of Theorem \ref{thm5}, choose 
        \begin{equation*}
            \theta_1=
            \begin{cases}
                \frac{1+2r+s}{3+2r+s}, &\text{ when } r\leq\frac{1-s}{2}, \\
                \frac{2\min\{r,1\}+s}{1+2\min\{r,1\}+s}, &\text{ when } r>\frac{1-s}{2},
            \end{cases}
        \end{equation*}
        and $\theta_2=1-\theta_1$. Then for any $T\geq1$ and $\delta\in(0,2/e)$, with at least $1-2\delta$ probability, the following holds:
        \begin{itemize}
            \item[(1)] If $r\leq\frac{1-s}{2}$,
            \begin{equation*}
                \begin{aligned}
                    \left\|h_{T+1}-h^{\dagger}\right\|_{\H}^2
            &\leq c_{2,2}(T+t_0)^{-\frac{4r}{3+2r+s}}\log^2(T+t_0)\log^4\frac2\delta.
                \end{aligned}    
            \end{equation*}
            \item[(2)] If $r>\frac{1-s}{2}$,
            \begin{equation*}
                \begin{aligned}
                    \left\|h_{T+1}-h^{\dagger}\right\|_{\H}^2
            &\leq c_{2,2}\l((T+t_0)^{-\frac{2\min\{r,1\}}{1+2\min\{r,1\}+s}}+(T+t_0)^{-\frac{4\min\{r,1\}+s-1}{1+2\min\{r,1\}+s}}\log^2(T+t_0)\log^2\frac2\delta\r)\log^2\frac2\delta \\
            &\lesssim (T+t_0)^{-\frac{2\min\{r,1\}}{1+2\min\{r,1\}+s}}\log^4\frac2\delta.
                \end{aligned}    
            \end{equation*}
        \end{itemize}        
         Here the constant $c_{2,2}$ is independent of\/ $T$ and $\delta$.
    \end{theorem}

    \begin{corollary} \label{coro2}
        Under the conditions of Theorem \ref{thm5}, choose $\theta_1$ and $\theta_2$ as in Theorem \ref{thm8}. Then, for any $\delta\in(0,2/e)$, with probability at least $1-2\delta$ , for all $1\leq t<\infty$, the following holds:
        \begin{equation*}
            \left\|h_{t+1}-h^{\dagger}\right\|_{\H}^2
            \leq
            \begin{cases}
                \widetilde{c}_{2,2}(t+t_0)^{-\frac{4r}{3+2r+s}}\log^6(t+t_0)\log^4\frac2\delta, & \text{ when }r\leq\frac{1-s}{2}, \\
                \widetilde{c}_{2,2}(t+t_0)^{-\frac{2\min\{r,1\}}{1+2\min\{r,1\}+s}}\log^4(t+t_0)\log^4\frac2\delta, & \text{ when }r>\frac{1-s}{2}.             
            \end{cases}
        \end{equation*}
         Here the constant $\widetilde{c}_{2,2}$ is independent of\/ $t$ and $\delta$.
    \end{corollary}

  The following two theorems provide high-probability convergence rates for the prediction and estimation errors, respectively, in the finite-horizon setting.

    \begin{theorem} \label{thm6}
        Suppose that Assumption \ref{a1}, Assumption \ref{a2}, Assumption \ref{a3} and Assumption \ref{a5} are satisfied. Define $\{h_t\}_{t\in\bn_T}$ through \eqref{iter} with step sizes $\{\eta_t=\eta_\star T^{-\theta_3}\}_{t\in\bn_{T}}$ and regularization parameters $\{\lambda_t=\lambda_\star T^{-\theta_4}\}_{t\in{\bn_T}}$, where $T\geq2$ and $\eta_\star(\kappa^2+\lambda_\star)\leq1$. Choose $\theta_3=\frac{2r+1}{2r+2}$ and $\theta_4\geq\frac{2r+1}{(2r+2)\min\{2r+1,2\}}$. Then, for any $\delta\in(0,2/e)$, with probability at least $1-2\delta$,
            \begin{equation*}
                \begin{aligned}
                    \mathcal{E}(h_{T+1})-\mathcal{E}(h^\dagger)&\leq
                c_{2,3}
            \begin{cases}
                T^{-\theta_3}\log^2\frac2\delta+T^{1-3\theta_3}\log^2T\log^4\frac2\delta, & \text{ when } s<1, \\
                T^{-\theta_3}\log T\log^2\frac2\delta+T^{1-3\theta_3}\log^3T\log^4\frac2\delta, & \text{ when } s=1, \\
            \end{cases} \\
            &\lesssim\log^4\frac2\delta
            \begin{cases}
                T^{-\theta_3}, & \text{ when } s<1, \\
                T^{-\theta_3}\log T, & \text{ when } s=1.
            \end{cases}
                \end{aligned}
            \end{equation*}
            Here the constant $c_{2,3}$ is independent of \/ $T$ and $\delta$.
    \end{theorem}

    \begin{theorem} \label{thm7}
        Under the conditions of Theorem \ref{thm6},
        choose 
        \begin{equation*}
            \theta_3=
            \begin{cases}
                \frac{1+2r+s}{3+2r+s}, &\text{ when }r\leq\frac{1-s}{2}, \\[5pt]
                \frac{2r+s}{1+2r+s}, &\text{ when }r>\frac{1-s}{2},
            \end{cases}
        \end{equation*}
        and
        \begin{equation*}
            \theta_4\geq
            \begin{cases}
                \frac{2r}{(3+2r+s)r}, &\text{ when }r\leq\frac{1-s}{2}, \\
                \frac{r}{(1+2r+s)\min\{r,1\}}, &\text{ when }r>\frac{1-s}{2}.
            \end{cases}
        \end{equation*}
        Then, for any $\delta\in(0,2/e)$, with probability at least $1-2\delta$, the following holds:
        \begin{itemize}
            \item[(1)] If $r\leq\frac{1-s}{2}$,
            \begin{equation*}
                \begin{aligned}
                    \left\|h_{T+1}-h^{\dagger}\right\|_{\H}^2
                    &\leq c_{2,4}\l(T^{-\frac{1+2r-s}{3+2r+s}}+T^{-\frac{4r}{3+2r+s}}\log^2 T\log^2\frac{2}{\delta}\r)\log^2\frac{2}{\delta} \\
                    &\lesssim T^{-\frac{4r}{3+2r+s}}\log^2 T\log^4\frac{2}{\delta}.
                \end{aligned}
            \end{equation*}
            \item[(2)] 
            If $r>\frac{1-s}{2}$,
            \begin{equation*}
                \begin{aligned}
                    \left\|h_{T+1}-h^{\dagger}\right\|_{\H}^2
                    &\leq c_{2,4}\l(T^{-\frac{2r}{1+2r+s}}+T^{-\frac{4r+s-1}{1+2r+s}}\log^2 T\log^2\frac{2}{\delta}\r)\log^2\frac{2}{\delta} \\
                    &\lesssim T^{-\frac{2r}{1+2r+s}}\log^4\frac{2}{\delta}.
                \end{aligned}
            \end{equation*}
        \end{itemize}
        Here the constant $c_{2,4}$ is independent of\/ $T$ and $\delta$.
    \end{theorem}

Building upon the results from the previous theorems, we note that Theorem \ref{thm6} for $s=1$ and Theorem \ref{thm7} for $r\geq\frac{1-s}{2}$ achieve the minimax lower bound derived in Theorem 2.9 of \citet{shi2024learning}, up to a logarithmic factor.

\medskip
\noindent\textbf{A key decomposition.}
To help the reader better understand the proof strategy, we briefly indicate the common quantity underlying our error analysis, namely
\[
\left\|\left(H_{T+1}-H^\dagger\right)C^\alpha\right\|_{\mathrm{HS}}^2,
\qquad 0\le \alpha\le \frac12.
\]
By the identity
\[
\mathcal{E}(H)-\mathcal{E}(H^\dagger)
=
\|(H-H^\dagger)C^{1/2}\|_{\mathrm{HS}}^2,
\]
the choice $\alpha=\frac12$ corresponds to the prediction error, while $\alpha=0$ corresponds to the estimation error. The full details are given in Section~\ref{decomposition}; here we only record the corresponding decompositions and refer the reader there for the precise definitions of the auxiliary quantities.

In the expectation analysis, we derive
\begin{equation*}
\be_{z^{T}}\left[
\left\|\left(H_{T+1}-H^{\dagger}\right)C^\alpha\right\|_{\mathrm{HS}}^2
\right]
\leq
\T_1+\T_2+\T_3+\T_4,
\end{equation*}
where
\begin{equation*}
\begin{split}
&\T_1:=2\left\|(H_{\lambda_T}-H^\dagger)C^\alpha\right\|_{\mathrm{HS}}^{2},\\
&\T_2:=6\left\|H_{\lambda_0}C^\alpha\prod_{t=1}^{T}\left(I-\eta_{t}(C+\lambda_{t}I)\right)\right\|^2_{\mathrm{HS}},\\
&\T_3:=6\left\|\sum_{t=1}^{T}(H_{\lambda_{t-1}}-H_{\lambda_{t}})C^\alpha
\prod_{j=t}^{T}(I-\eta_j(C+\lambda_j I))\right\|^2_{\mathrm{HS}},\\
&\T_4:=6\sqrt{c}\sum_{t=1}^{T}\eta_t^{2}\left(\sqrt{c}\be_{z^{t-1}}\mathbb E_{x_t}\left\|(H_t-H^\dagger)\phi(x_t)\right\|_{\Y}^2+\sigma^2\right)
\mathrm{Tr}\left(C^{1+2\alpha}\prod_{j=t+1}^{T}(I-\eta_j(C+\lambda_j I))^2\right).
\end{split}
\end{equation*}
Here $\T_1$, $\T_2$, $\T_3$, and $\T_4$ are referred to as the approximation term, the initial term, the drift term, and the sample term, respectively. In the finite-horizon setting, the drift term $\T_3$ vanishes since the regularization parameter is constant.

For the high-probability analysis, the corresponding decomposition takes the form
\begin{equation*}
\left\|\left(H_{T+1}-H^\dagger\right)C^\alpha\right\|_{\mathrm{HS}}^2
\leq
\T_1+\T_2+\T_3+6\left\|\sum_{t=1}^T \chi_t\right\|_{\mathrm{HS}}^2,
\end{equation*}
where $\{\chi_t\}$ is the stochastic fluctuation sequence appearing in the high-probability analysis; see Section~\ref{prob}. The deterministic terms $\T_1,\T_2,\T_3$ are estimated in the same way as in the expectation analysis, whereas the last stochastic term is handled differently. In expectation, it is controlled through $\T_4$ using Assumption~\ref{a4}. In the high-probability setting, it is treated by a truncation argument together with martingale techniques.

\medskip
\noindent\textbf{Relation to the unregularized analysis.}
The decomposition above also clarifies the relation between the present paper and our earlier unregularized SGD analysis \citep{shi2024learning}. While the basic operator-learning setting and the overall expectation-based framework are similar, the explicit regularization term changes the recursion in an essential way. In particular, it introduces the regularization path $\{H_{\lambda_t}\}$ and the drift term $\T_3$, and it changes the treatment of $\T_2$ and $\T_4$, especially in the critical regime $\theta_1+\theta_2=1$.

These changes are reflected both in the analysis and in the resulting guarantees. At the level of results, explicit regularization yields sharper conclusions in the online setting, especially for the estimation error. At the methodological level, it leads to a different decomposition and a different treatment of the stochastic recursion. At the probabilistic level, it also enables high-probability bounds and hence almost sure convergence in the infinite-dimensional setting.

\section{Discussion} \label{new dis}

This section places the abstract framework studied in this paper in a broader operator-learning context and further interprets its functional-analytic meaning. 
We first show that, for a suitable class of scalar kernels, the induced vector-valued RKHS is norm-equivalent to a Bochner--Sobolev space, thereby giving a concrete interpretation of both the hypothesis space and the estimation error in our analysis. 
We then discuss the relation of our results to recent operator learning literature and conclude with several remarks on possible extensions and on the scope and limitations of the present theory.

\medskip
\noindent\textbf{Bochner--Sobolev spaces.}
In the discussion below, we further assume that the input space \(\X\subset\mathbb{R}^d\) is a bounded Lipschitz domain, equipped with the Lebesgue measure. For a Banach space \(E\) and \(1\le p\le\infty\), let \(W^{k,p}(\X;E)\) denote the Bochner--Sobolev space of \(E\)-valued functions on \(\X\) whose weak derivatives up to order \(k\) belong to \(L^p(\X;E)\); see \citet{amann1995linear,aubin2000applied,cardanobile2009parabolic}. When \(E\) is a Hilbert space, we write
\[
H^k(\X;E)=W^{k,2}(\X;E).
\]
Since our output space \(\Y\) is a separable Hilbert space, the space \(H^k(\X;\Y)\) is well defined.

The next proposition shows that, for a suitable class of scalar kernels, the induced vector-valued RKHS is norm-equivalent to a Bochner--Sobolev space.

\begin{proposition}\label{sobolev}
Suppose that \(\mathcal{K}\) is a translation-invariant scalar-valued kernel on \(\mathbb{R}^{d}\) whose Fourier transform satisfies
\[
\widehat{\mathcal{K}}(\xi)\asymp (1+\|\xi\|_2^2)^{-k},
\qquad \xi\in\mathbb{R}^{d},
\]
for some \(k>d/2\). Then the RKHS induced by the restriction of \(\mathcal{K}\) to \(\X\times\X\) is norm-equivalent to \(H^k(\X)\).

Define the operator-valued kernel
\[
K:\X\times\X\to\mathcal{B}(\Y),
\qquad
K(x,x')=\mathcal{K}(x,x')\,I,
\]
where \(I\) denotes the identity operator on \(\Y\). Then the corresponding vector-valued RKHS \(\mathcal{H}\) is norm-equivalent to \(H^k(\X;\Y)\).
\end{proposition}

This proposition is essentially a specialization of the vector-valued Sobolev RKHS characterization developed in \citet[Section 6]{li2024towards}. The underlying argument combines three ingredients: the norm equivalence between scalar Mat\'ern-type RKHSs and Sobolev spaces on \(\X\) \citep[Corollary 10.13 and Theorem 10.46]{wendland2004scattered}; the Hilbert-valued Sobolev representation $H^k(\X;E)\cong E\widehat{\otimes}H^k(\X),$ where \(\widehat{\otimes}\) denotes the Hilbert tensor product, from \citet[Theorem 12.7.1]{aubin2000applied}; and the fact that an operator-valued kernel of the form \(K(x,x')=\mathcal{K}(x,x')I\) induces a vector-valued RKHS canonically identified with \(\Y\widehat{\otimes}\mathcal{H}_{\mathcal K}\); see \citet[Example 5(i)]{carmeli2010vector}. A concrete example covered by this proposition is provided by Mat\'ern kernels with Sobolev smoothness index \(k\). Consequently, in such settings, the estimation error in our analysis can be interpreted in terms of a Sobolev norm.

\medskip
\noindent\textbf{Relation to operator learning.}
The present work studies operator learning from the perspective of nonparametric kernel methods. In the setting \(K(x,x')=\mathcal{K}(x,x')I\), the operator-valued kernel \(K\) is \(L^2\)-universal if and only if the scalar kernel \(\mathcal K\) is \(L^2\)-universal \footnote{Here $L^2$-universality means that the RKHS is dense in
$L^2(\mathcal X,\rho_{\mathcal X};\mathcal Y)$ for every Borel probability measure
$\rho_{\mathcal X}$ on $\mathcal X$; for scalar-valued kernels, take
$\mathcal Y=\mathbb R$.
By \citet[Example 5(i)]{carmeli2010vector},
$\mathcal{H}\cong\mathcal{Y}\widehat{\otimes}\mathcal{H}_{\mathcal K}$,
and by \citet[Theorem 12.6.1]{aubin2000applied},
$L^2(\mathcal{X},\rho_{\mathcal X};\mathcal{Y})
\cong\mathcal{Y}\widehat{\otimes}L^2(\mathcal{X},\rho_{\mathcal X})$.
These two identifications imply the equivalence.
}. This provides a qualitative parallel to the universal approximation results established for neural operator architectures such as FNO \citep{kovachki2021universal}, DeepONet \citep{lanthaler2022error}, and PCA-Net \citep{lanthaler2023operator}.

Much of the existing quantitative theory for operator learning concerns approximation, parameter complexity, and sample complexity for achieving accuracy \(\epsilon\). A substantial part of this theory either treats highly structured operator classes, such as infinite-dimensional holomorphic maps \citep{adcock2024optimal}, Fréchet differentiable operators learned via prespecified neural operators \citep{cheng2026learning}, or derives rates for concrete PDEs by showing that neural architectures can emulate classical numerical schemes such as Fourier--Galerkin discretizations and fixed-point iterations \citep{kovachki2021universal,lanthaler2022error,lanthaler2023operator}. Related statistical learning theory has also been developed for neural operators from the viewpoint of empirical risk minimization and statistical learning theory \citep{reinhardt2024statistical}, while recent theory-to-practice gap results reveal substantial sampling barriers for neural operators in relevant approximation spaces  \citep{grohs2025theory}. For broader classes, algebraic rates are generally unavailable: for Lipschitz operator classes and Gaussian--Sobolev settings, both parameter and sample complexity may exhibit curse-of-complexity behavior \citep{adcock2024sample,lanthaler2023operator,kovachki2024data}, while near-sharp minimax risk bounds for Lipschitz operators were recently established in \citet{adcock2025towards}. 

In contrast to works that focus primarily on approximation, architectural complexity, or sample complexity for specific operator-learning models, the present paper studies a structured nonparametric kernel framework and analyzes regularized SGD as a concrete learning algorithm in infinite-dimensional spaces. Under RKHS-based regularity assumptions, we establish algebraic convergence rates for both prediction and estimation errors, in expectation and with high probability, in both online and finite-horizon settings. Our results therefore complement the existing operator-learning literature by highlighting the roles of optimization and regularization, and by providing almost sure convergence guarantees.

\medskip
\noindent\textbf{Remarks on extensions.}
The regularized SGD analysis also extends beyond Assumption~\ref{a1} to more general operator-valued kernels; see Appendix~\ref{Section 3.1}. Since this extension does not include the noncompact case \(K(x,x')=\mathcal K(x,x')I\), the two settings are complementary. Together they cover a broader class of operator learning problems.

We next describe a viewpoint that links the prediction error, the estimation error, and Assumption~\ref{a2}. As noted in \citet[Remark 3]{yang2025kernel}, Assumption~\ref{a2} is equivalent to \(h^\dagger \in \mathrm{ran}(L_K^{\,r+1/2})\), where \(\rho_{\X}\) denotes the marginal distribution of \(\rho\) on \(\X\), and \(L_K:L^2(\X,\rho_{\X};\Y)\to L^2(\X,\rho_{\X};\Y)\) is the integral operator induced by \(K\). Restricting \(L_K\) to \(\ker(L_K)^\perp\), we see that for \(h\in\H\), the prediction error is \(\|h-h^\dagger\|_{L^2(\X,\rho_{\X};\Y)}^2\), while the estimation error can be written as \(\|h-h^\dagger\|_{\H}^2=\|L_K^{-1/2}(h-h^\dagger)\|_{L^2(\X,\rho_{\X};\Y)}^2\). More generally, for \(\beta\in(0,1)\), the quantity \(\|L_K^{-\beta/2}(h-h^\dagger)\|_{L^2(\X,\rho_{\X};\Y)}^2\) measures the error in the interpolation space \([L^2(\X,\rho_{\X};\Y),\H]_{\beta,2}\); see \citet[Theorem 2.3]{yang2025kernel}. By the interpolation inequality,
\[
\|L_K^{-\beta/2}(h-h^\dagger)\|_{L^2(\X,\rho_{\X};\Y)}^2
\le
\|h-h^\dagger\|_{L^2(\X,\rho_{\X};\Y)}^{2(1-\beta)}
\|h-h^\dagger\|_{\H}^{2\beta}.
\]
Applying this to the estimator \(h_{T+1}\) and taking expectation with respect to the sample, Hölder's inequality yields
\[
\text{interpolation-space error}
\le
(\text{prediction error})^{1-\beta}\times(\text{estimation error})^{\beta}.
\]
Under the setting of Proposition~\ref{sobolev}, if \(\rho_{\X}\) is equivalent to the uniform measure on \(\X\), so that \(L^2(\X,\rho_{\X})\) and \(L^2(\X)\) coincide with equivalent norms, then the interpolation space \([L^2(\X,\rho_{\X};\Y),\H]_{\beta,2}\) corresponds to the Bochner--Sobolev space \(H^{\beta k}(\X;\Y)\); see \citet[Section 6]{li2024towards}.

\section{Applications and Extensions} \label{discussion}
This section discusses several applications and related connections of our results. We begin with structured prediction problems, then turn to a class of parametric PDEs. A brief discussion of the connection with the PCA encoder-decoder framework is deferred to Appendix \ref{Section 3.3}.

\subsection{Applications to Structured Prediction} \label{Section 3.2}

Now, we formulate the surrogate approach for structured prediction as an application example of our model \eqref{model}. In structured prediction, the input takes values in $\X$ and the output takes values in $\Z$, where $\X$ is a Polish space and $\Z$ represents the structured output space. A structured loss function $\mathscr{D}:\Z \times \Z \to \mathbb{R}$ is defined on $\Z$ to measure the discrepancy between the true and the predicted outputs. Let $x$ denote the input random variable and $z$ the output random variable. Given a set of independent and identically distributed input-output samples, our goal is to learn a mapping from the inputs to the structured outputs. To this end, we minimize the prediction error defined by
\[
    \mathcal{R}(f):=\be\l[\mathscr{D}(f(x),z)\r],
\]
where $f$ is an estimator of $f^\dagger$. The function $f^\dagger:\X\to\Z$ is the minimizer of $\mathcal{R}$, i.e., $f^\dagger=\arg\min_f \mathcal{R}(f)$. 
    
We focus on the case where the loss function $\mathscr{D}$ is induced by a scalar-valued kernel $k_\Z : \Z \times \Z \to \mathbb{R}$. Specifically, denote the RKHS induced by $k_\Z$ by $\left(\Y, \|\cdot\|_{\Y}\right)$, and embed $\Z$ into $\Y$ via the canonical feature map $\phi(z):=k_{\Z}(z,\cdot)$. We then define the structured loss as $\mathscr{D}(z,z')=\|\phi(z)-\phi(z')\|_{\Y}^2$. Building on extensive research on kernels for structured objects \citep{gartner2003survey}, this class of loss functions addresses various structured prediction problems. Instead of directly learning $f^\dagger$, we adopt a surrogate model $h^\dagger : \X \rightarrow \Y$, where $h^\dagger(x) := \mathbb{E}[y|x]$ and $y := \phi(z)$ is a random variable taking values in $\Y$. This reduces the original structured prediction task to the least-squares regression problem underlying model \eqref{model}. We then reformulate the original structured prediction problem as the following surrogate nonlinear operator learning problem:
    \[
    \min_{h:\X\rightarrow\Y}\be[\|h(x)-y\|_{\Y}^2].
    \]

    We solve this problem using the SGD algorithm presented in this paper, which yields an approximation of $h^\dagger$, denoted by $\hat{h}$. During prediction, we use a decoding operator $D$ defined as
    \[
    D(h)(\cdot):=\mathop{\arg\min}_{z\in\Z}\l\{\|h(\cdot)-\phi(z)\|_{\Y}\r\}
    \]
        for any estimator $h$, as detailed in \citet{ciliberto2016consistent, brogat2022vector}. Let $\hat{f} = D(\hat{h})$, denote the estimator for $f^\dagger$ obtained via the algorithm. The surrogate approach for structured prediction is illustrated in Figure \ref{fig:enter-label}.

    According to \citet{ciliberto2016consistent}, the following properties hold:
    \begin{itemize}
        \item[(1)] Fisher Consistency: $D(h^\dagger) = f^\dagger$ almost surely.
        \item[(2)] Comparison Inequality:
        \[
    \mathcal{R}(\hat f)-\mathcal{R}(f^\dagger)\lesssim
    \l(\be\l[\|\hat{h}(x)-h^\dagger(x)\|_{\Y}^2\r]\r)^{\frac12}.
    \]
    \end{itemize}
Thus, to bound $\mathcal{R}(\hat{f}) - \mathcal{R}(f^\dagger)$, it suffices to bound $\mathbb{E}[\| \hat{h}(x) - h^\dagger(x) \|_{\Y}^2]=\mathcal{E}(\hat{h})-\mathcal{E}(h^\dagger)$, as conducted in this paper. This guarantees decay rates of the prediction error under mild assumptions.

In many structured prediction tasks, such as those in natural language processing (e.g., sequence labeling and machine translation) or time series forecasting, data often arrive sequentially or in streams. In such cases, SGD is particularly well suited, as it allows for incremental model updates with each new data point. This makes it an effective tool for structured prediction in streaming or time-dependent environments.

\tikzset{every picture/.style={line width=0.75pt}} 

\begin{figure}
    \centering
    
\begin{tikzpicture}[x=0.75pt,y=0.75pt,yscale=-1,xscale=1]


\draw    (224.5,205) -- (344.5,205) ;
\draw [shift={(346.5,204)}, rotate = 179.53] [color={rgb, 255:red, 0; green, 0; blue, 0 }  ][line width=0.75]    (10.93,-3.29) .. controls (6.95,-1.4) and (3.31,-0.3) .. (0,0) .. controls (3.31,0.3) and (6.95,1.4) .. (10.93,3.29)   ;
\draw    (359.5,193) -- (359.5,93) ;
\draw [shift={(359.5,91)}, rotate = 90] [color={rgb, 255:red, 0; green, 0; blue, 0 }  ][line width=0.75]    (10.93,-3.29) .. controls (6.95,-1.4) and (3.31,-0.3) .. (0,0) .. controls (3.31,0.3) and (6.95,1.4) .. (10.93,3.29)   ;
\draw    (224.5,195) -- (348.97,90.29) ;
\draw [shift={(350.5,89)}, rotate = 139.93] [color={rgb, 255:red, 0; green, 0; blue, 0 }  ][line width=0.75]    (10.93,-3.29) .. controls (6.95,-1.4) and (3.31,-0.3) .. (0,0) .. controls (3.31,0.3) and (6.95,1.4) .. (10.93,3.29)   ;
\draw    (374.5,87) .. controls (392.61,102.98) and (403.06,123.22) .. (402.83,144.99) .. controls (402.62,163.71) and (394.53,183.56) .. (376.61,202.82) ;
\draw [shift={(375.5,204)}, rotate = 313.6] [color={rgb, 255:red, 0; green, 0; blue, 0 }  ][line width=0.75]    (10.93,-3.29) .. controls (6.95,-1.4) and (3.31,-0.3) .. (0,0) .. controls (3.31,0.3) and (6.95,1.4) .. (10.93,3.29)   ;

\draw (209,197.5) node [anchor=north west][inner sep=0.75pt]    {$\mathcal{X}$};
\draw (351,197.5) node [anchor=north west][inner sep=0.75pt]    {$\mathcal{Z}$};
\draw (352,69.4) node [anchor=north west][inner sep=0.75pt]    {$\mathcal{Y}$};
\draw (269,123.4) node [anchor=north west][inner sep=0.75pt]    {$\hat{h}$};
\draw (366,134.4) node [anchor=north west][inner sep=0.75pt]    {$\phi $};
\draw (258,209.4) node [anchor=north west][inner sep=0.75pt]    {$\hat{f} =D(\hat{h})$};
\draw (412,137.4) node [anchor=north west][inner sep=0.75pt]    {$D$};

\end{tikzpicture}

\caption{Surrogate approach for structured prediction}

\label{fig:enter-label}

\end{figure}

\subsection{Applications to Parametric PDEs}
\label{sec:pde-application}

In this subsection, we apply regularized SGD to a class of
 parametric elliptic PDEs. We specify the sampling distribution
and construct a kernel using bounds on the diffusion coefficients. Proofs and further results on uniform sampling and coefficient decay
are given in Appendix~\ref{proofs in sec:pde-application}.

Let $\Omega_x\subset\mathbb R^d$ be a bounded Lipschitz domain and
set $\Y=H_0^1(\Omega_x)$, equipped with the norm
$\|v\|_\Y=\|\nabla v\|_{L^2(\Omega_x)}$.
For a fixed $N=N_\zeta\ge1$, let $\X=\Omega_\zeta=[-1,1]^N$ and consider
\begin{equation}\label{PDE}
 -\nabla\!\cdot\!\bigl(a(x,\zeta)\nabla u(x,\zeta)\bigr)=f,
 \qquad
 a(x,\zeta)=a_0(x)+\sum_{j=1}^N\zeta_j a_j(x),
\end{equation}
with homogeneous Dirichlet boundary conditions.
Assume that $f\in H^{-1}(\Omega_x)$, that
$a_0,\ldots,a_N$ are real functions in $L^\infty(\Omega_x)$, and that
$a(x,\zeta)\ge a_{\min}>0$ for almost every $x$ and every
$\zeta\in\Omega_\zeta$.
The Lax--Milgram theorem therefore gives a unique weak solution
$u(\cdot,\zeta)\in\Y$ for each $\zeta\in\Omega_\zeta$ (see also
\citet[Theorem~2.1]{cohen2015approximation}).
We define the solution map by $h^\dagger(\zeta)=u(\cdot,\zeta)$.

To choose the kernel parameters, let $B_j\ge\|a_j\|_{L^\infty(\Omega_x)}$ be known bounds on the
coefficient functions. Choose radii $\rho_j>1$ such that
\begin{equation}\label{pde:radius-budget}
 \sum_{j=1}^N
 \left(\frac{\rho_j+\rho_j^{-1}}2-1\right)B_j
 \le \frac{a_{\min}}2.
\end{equation}
For fixed finite $N$, such radii always exist by choosing them
sufficiently close to $1$. The map $h^\dagger$ then admits a bounded
holomorphic extension to a neighborhood of the corresponding
product of Bernstein ellipses
\citep[Corollary~2.4]{cohen2015approximation}.

We sample the parameter vector from the product arcsine probability
measure $\mu_N$ on $\Omega_\zeta=[-1,1]^N$, given by
\[
 d\mu_N(\zeta)=\prod_{j=1}^N
 \frac{d\zeta_j}{\pi\sqrt{1-\zeta_j^2}}.
\]
Equivalently, $\zeta_j=\cos\vartheta_j$, where the $\vartheta_j$
are independent and uniform on $[0,\pi]$.
The spatial domain $\Omega_x$ is fixed.
The training samples are independent copies of $(\zeta,y)$, where
\[
 y=h^\dagger(\zeta)+\epsilon,
 \qquad \zeta\sim\mu_N,
\]
and $\epsilon$ is a centered random variable in $\Y$, independent of
$\zeta$, with $\|\epsilon\|_\Y\le\sigma_b$ almost surely for some
$\sigma_b>0$.

Write $p_0=1$ and $p_n=\sqrt2\,T_n$ for $n\ge1$, where $T_n$ is the
Chebyshev polynomial of the first kind. Then
$p_\nu(\zeta)=\prod_{j=1}^Np_{\nu_j}(\zeta_j)$,
$\nu\in\mathbb N_0^N$, form an orthonormal basis of $L^2(\mu_N)$.
Define
\begin{equation}\label{pde:kernel}
 \lambda_\nu=\prod_{j=1}^N\rho_j^{-\nu_j/2},
 \qquad
 \K_N(\zeta,\zeta')=
 \sum_{\nu\in\mathbb N_0^N}
 \lambda_\nu p_\nu(\zeta)p_\nu(\zeta'),
 \qquad K_N=\K_N I,
\end{equation}
where $I$ is the identity on $\Y$.
Let $\H_{\K_N}$ and $\H_N$ be the RKHSs associated with $\K_N$
and $K_N$, respectively, and define
\[
 C_N=\int_{\Omega_\zeta}
 \phi_N(\zeta)\otimes\phi_N(\zeta)\,d\mu_N(\zeta),
 \qquad \phi_N(\zeta)=\K_N(\zeta,\cdot).
\]

The kernel can be evaluated without truncating the series in
\eqref{pde:kernel}. For $q\in(0,1)$, set
\[
 P_q(t)=\frac{1-q^2}{(1-q)^2+4q\sin^2(t/2)}.
\]
Using $T_n(\cos\vartheta)=\cos(n\vartheta)$ and summing geometric
series gives (see Appendix \ref{Proofs for arcsine sampling})
\begin{equation}\label{pde:kernel-evaluation}
 \K_N(\zeta,\zeta')
 =\prod_{j=1}^N
 \frac{
 P_{q_j}(\vartheta_j-\vartheta_j')
 +P_{q_j}(\vartheta_j+\vartheta_j')}{2},
 \qquad q_j=\rho_j^{-1/2},
\end{equation}
where $\vartheta_j=\arccos\zeta_j$ and
$\vartheta_j'=\arccos\zeta_j'$.
Thus each kernel evaluation requires $O(N)$ elementary operations.

\begin{proposition}\label{pde:verification}
The above model satisfies Assumptions~\ref{a1}, \ref{a2},
\ref{a3}, and~\ref{a5} for every $0<r<3/2$ and $0<s\le1$. 
More precisely, let $M=2\|f\|_{H^{-1}(\Omega_x)}/a_{\min}$, 
where the dual norm is induced by $\|\cdot\|_\Y$. Then
\begin{equation}\label{pde:structural-bounds}
 \kappa_N^2:=\sup_{\zeta\in\Omega_\zeta}\K_N(\zeta,\zeta)
 =\prod_{j=1}^N\frac{1+\rho_j^{-1/2}}{1-\rho_j^{-1/2}},
 \qquad
 \mathrm{Tr}(C_N^s)
 =\prod_{j=1}^N(1-\rho_j^{-s/2})^{-1}.
\end{equation}
For each $0<r<3/2$, the operator $H^\dagger$ representing
$h^\dagger$ can be written as $H^\dagger=S_r^\dagger C_N^r$, with $h^\dagger=H^\dagger\phi_N$, $S_r^\dagger\in\B_{\mathrm{HS}}(\H_{\K_N},\Y)$, and
\begin{equation}\label{pde:source-radius}
 \|S_r^\dagger\|_{\mathrm{HS}}^2
 \le R_{N,r}^2
 :=M^2\prod_{j=1}^N
 \frac{1+\rho_j^{r-3/2}}{1-\rho_j^{r-3/2}}.
\end{equation}
Moreover, Assumption~\ref{a5} holds with $M_\rho=M+\sigma_b$.
\end{proposition}

Assumption~\ref{a4} does not hold for the kernel $\K_N$ and
sampling measure $\mu_N$, as shown in
Proposition~\ref{pde:moment-failure} in
Appendix~\ref{proofs in sec:pde-application}.
We apply Theorems~\ref{thm5} and~\ref{thm8}, which do not
require this assumption, with $r=1$ and any fixed $s\in(0,1)$.
Larger admissible values of $r$ do not improve the convergence
exponents in these theorems.

\begin{corollary}\label{pde:rates}
Under the conditions of Proposition~\ref{pde:verification}, fix
$s\in(0,1)$ and set $t_{0,N}\geq (\kappa_N^2+1)^{3/2}$. For each parameter choice below, let $\{h_t\}_{t\ge1}$ be the
corresponding iteration~\eqref{iter2}, with $h_1=0$, $K=K_N$,
and samples $(\zeta_t,y_t)$.
For every $T\ge1$ and $\delta\in(0,1/2)$, each bound holds with
probability at least $1-2\delta$:
\begin{enumerate}
\item[(1)] If $\eta_t=(t+t_{0,N})^{-2/3}$ and
$\lambda_t=(t+t_{0,N})^{-1/3}$, then
\[
 \|h_{T+1}-h^\dagger\|_{L^2(\mu_N;\Y)}^2
 \le A_{N,s}(T+t_{0,N})^{-2/3}\log^4\frac2\delta.
\]
\item[(2)] If $\eta_t=(t+t_{0,N})^{-(2+s)/(3+s)}$ and
$\lambda_t=(t+t_{0,N})^{-1/(3+s)}$, then
\[
 \|h_{T+1}-h^\dagger\|_{\H_N}^2
 \le D_{N,s}(T+t_{0,N})^{-2/(3+s)}\log^4\frac2\delta.
\]
\end{enumerate}
Here $A_{N,s}$ and $D_{N,s}$ can be chosen to depend only on
$t_{0,N}$, $\kappa_N$, $\mathrm{Tr}(C_N^s)$, $M$, $\sigma_b$,
and $s$, and are independent of $T$ and $\delta$.
\end{corollary}

Theorems~\ref{thm6} and~\ref{thm7} also apply to this model for
fixed $T\ge2$, with the respective step sizes and regularization
parameters specified in those theorems. The prediction bound averages the squared error in $\Y$ over $\mu_N$. 
For $g=\sum_\nu g_\nu p_\nu\in\H_N$, the RKHS norm satisfies
\[
 \|g\|_{\H_N}^2=\sum_\nu\frac{\|g_\nu\|_\Y^2}{\lambda_\nu},
 \qquad
 \sup_{\zeta\in\Omega_\zeta}\|g(\zeta)\|_\Y^2
 \le\kappa_N^2\|g\|_{\H_N}^2.
\]
The estimation bound therefore also controls the error in $\Y$
uniformly over $\Omega_\zeta$.

\medskip
\noindent\textbf{Extension to uniform sampling.} The same convergence exponents hold when the training parameters
are sampled from the uniform probability measure $m_N$ on
$\Omega_\zeta$, defined by $dm_N(\zeta)=2^{-N}\,d\zeta$.
For this sampling measure, we replace the Chebyshev basis in
\eqref{pde:kernel} with the orthonormal Legendre basis in
$L^2(m_N)$ and retain the eigenvalues $\lambda_\nu$.
The verification and the corresponding prediction and RKHS
estimation bounds are given in Appendix~\ref{pde:uniform}.

\medskip
\noindent\textbf{Uniform bounds under coefficient decay.}
The products in \eqref{pde:structural-bounds} and
\eqref{pde:source-radius} make the dependence of the structural
bounds on $N$ explicit.
These quantities may not remain bounded as $N$ increases.
For example, if the chosen radii satisfy
$\rho_1=\cdots=\rho_N=\rho>1$, then
\[
 \kappa_N^2
 =
 \left(\frac{1+\rho^{-1/2}}{1-\rho^{-1/2}}\right)^N,
 \qquad
 \mathrm{Tr}(C_N^s)
 =
 (1-\rho^{-s/2})^{-N}.
\]
Thus, convergence exponents that are independent of $N$
do not by themselves imply error bounds uniform in $N$.
We now give a sufficient condition on coefficient decay
and a corresponding choice of radii under which the error
bounds are uniform over finite parameter truncations. 
Consider the family of problems \eqref{PDE} obtained by
truncating a fixed sequence $(a_j)_{j\ge1}$, with
$\Omega_x$, $a_0$, $f$, and $\sigma_b$ independent of $N$. Suppose that all finite truncations satisfy the ellipticity
condition with the same $a_{\min}>0$, and that
\begin{equation}\label{pde:coefficient-decay}
 \|a_j\|_{L^\infty(\Omega_x)}
 \le B_j\le B j^{-\beta},
 \qquad j\ge1,
\end{equation}
for some $B>0$ and $\beta>3$.
Fix $s\in(2/(\beta-1),1)$ and choose
$2/s<\gamma<\beta-1$.
Set
\begin{equation}\label{pde:decay-radii}
 \rho_j=1+\tau j^\gamma,
 \qquad
 0<\tau\le
 \frac{a_{\min}}
 {B\sum_{k=1}^{\infty}k^{\gamma-\beta}},
 \qquad j\ge1.
\end{equation}
These radii satisfy \eqref{pde:radius-budget} for every $N$.
For both sampling measures $\mu_N$ and $m_N$, the corresponding kernel bounds, capacity quantities for this $s$,
and $\|S_1^\dagger\|_{\mathrm{HS}}$ admit finite upper bounds
independent of $N$. Since the error constants can be controlled using only upper bounds on $\kappa$, $\mathrm{Tr}(C^s)$,
$\|S^\dagger\|_{\mathrm{HS}}$, and $M_\rho$, with a common choice of $t_0$, the bounds of
Corollaries~\ref{pde:rates} and~\ref{pde:uniform-rates}
therefore hold for each finite truncation with constants
independent of $N$, $T$, and $\delta$.
The proof and a common choice of $t_0$ are given in
Appendix~\ref{pde:decay}. This uniformity concerns finite truncations of the fixed
coefficient sequence satisfying \eqref{pde:coefficient-decay},
with the radii chosen as in \eqref{pde:decay-radii}.

For affine-parametric elliptic problems such as \eqref{PDE},
uniform ellipticity and
$(\|a_j\|_{L^\infty(\Omega_x)})_{j\ge1}\in\ell^p$, $0<p<1$,
yield sparse Taylor approximation rates even for countably
many parameters \citep[Theorem~1.2]{cohen2011analytic}.
Best $n$-term Legendre approximation has $L^2$ error
$O(n^{-(1/p-1/2)})$ under the product uniform probability
measure, where $n$ counts polynomial terms
\citep[Corollary~3.11]{cohen2015approximation}.
For suitable classes of Hilbert-valued holomorphic functions
on the infinite parameter cube,
\citet[Theorems~4.5 and~4.6]{adcock2025optimalII} establish
high-probability $L^2$ recovery bounds with algebraic rate
$T^{-(1/p-1/2)}$, up to logarithmic factors, from exact
function values at $T$ i.i.d. sampling points, using least
squares and compressed sensing under their respective
summability and sampling assumptions.
Our analysis complements these results by providing
high-probability bounds for the squared $L^2$ prediction
error and squared RKHS estimation error of the last iterate
of online regularized SGD under the bounded, centered noise
model specified above (see Corollaries~\ref{pde:rates}
and~\ref{pde:uniform-rates} for the respective parameter choices).
The coefficient-decay condition
\eqref{pde:coefficient-decay} with $\beta>3$, together with
the parameter choices above, yields error constants uniform
over finite truncations of the fixed coefficient sequence. Whether the sufficient condition $\beta>3$ can be relaxed while
retaining error bounds uniform over finite parameter truncations
is an interesting question for future work.
A related direction is to explore kernel constructions better
adapted to the structure of parametric PDEs, with the aim of
further improving learning guarantees.

\section{Numerical Experiments} \label{Numerical experiments}
\begin{figure}[htbp]
    \centering
    \includegraphics[width=0.65\textwidth]{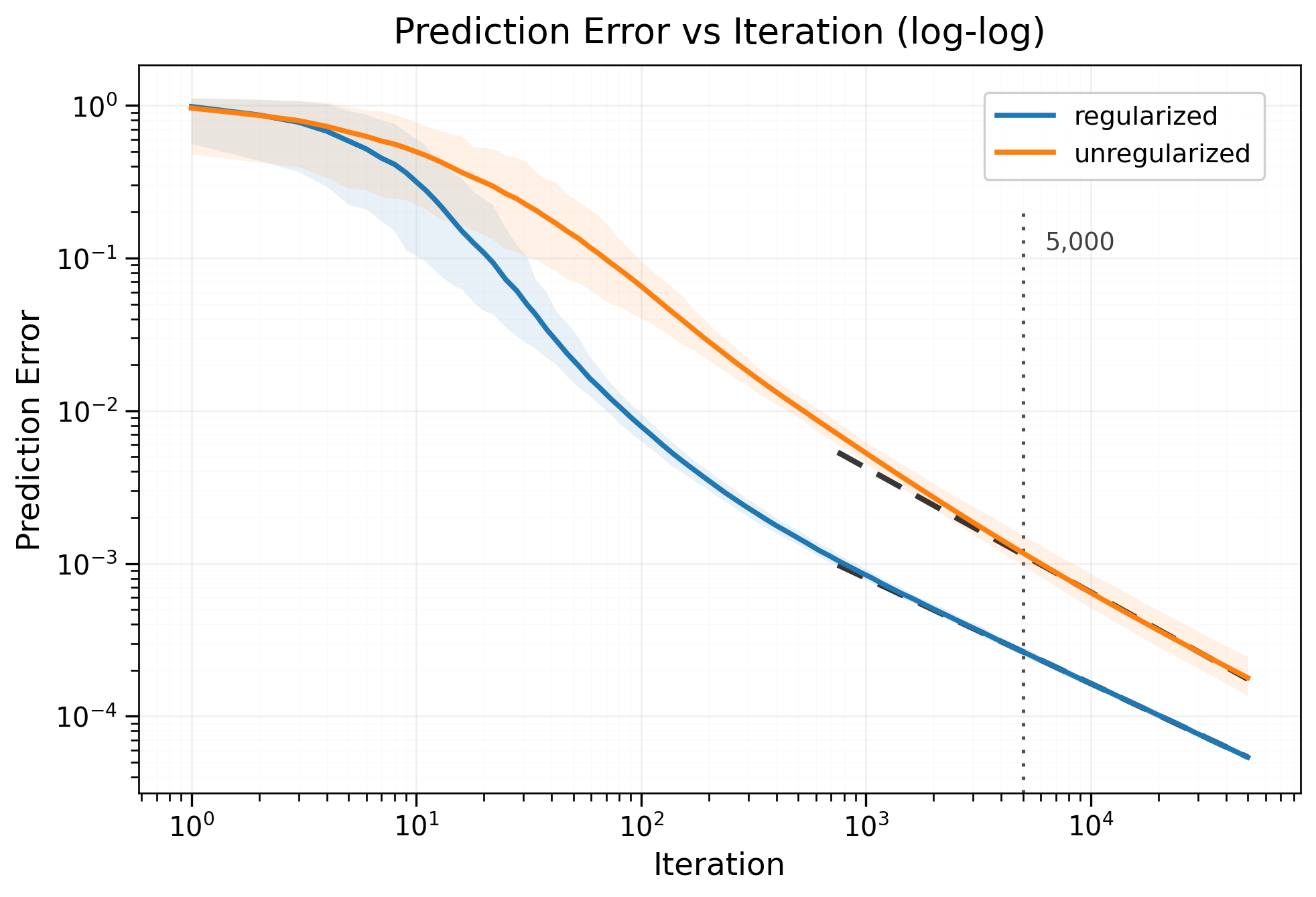}

    \vspace{0.8em}

    \includegraphics[width=0.65\textwidth]{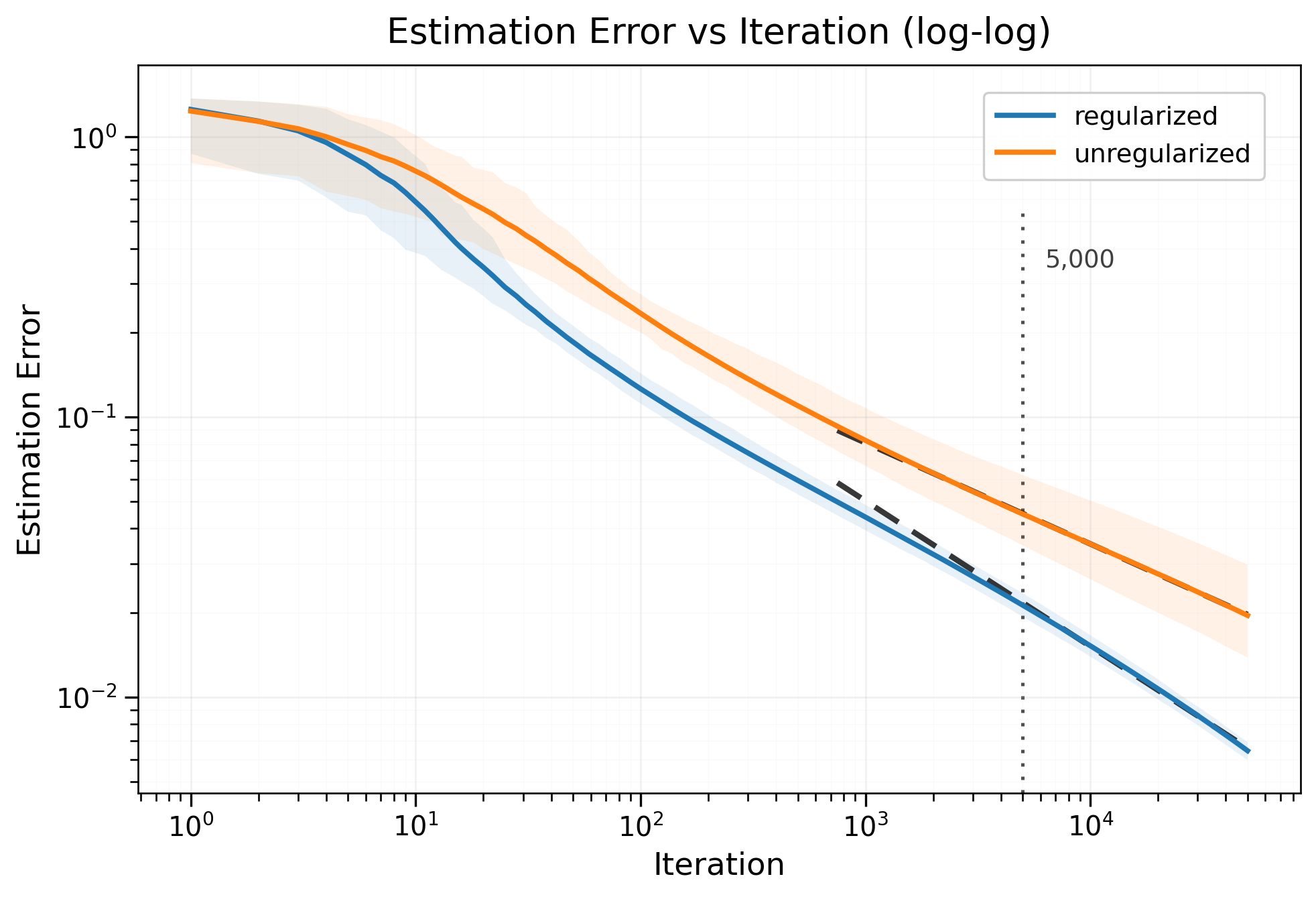}
    \caption{Log--log plots of the prediction and estimation errors versus the iteration number in the online setting, comparing regularized SGD with unregularized SGD. We set $m=200$, $\alpha=0.3$, $\beta=0.9$, $r=1.0$, and $\sigma=0.1$. Each experiment is run for $50{,}000$ iterations and repeated $100$ times. The shaded regions represent the empirical $90\%$ confidence bands. Both axes are shown on a logarithmic scale to highlight the polynomial decay behavior, and linear interpolation is used for $t>5000$.}
    \label{fig:twofigs_vertical}
\end{figure}
The following experiment is a synthetic illustration of the effect of explicit regularization. In the experiments, we represent inputs through the first $m$ basis functions of $\X=L^2[0,1]$,
\[
e_k(u)=\sqrt{2}\cos(k\pi u),\qquad u\in[0,1],\quad 1\le k\le m,
\]
and generate
\[
x=\sum_{k=1}^m \sqrt{\mu_k}\,\xi_k e_k,
\qquad
\mu_k=k^{-\alpha},
\]
where $\alpha>0$, $m\gg1$, and $\xi_k\sim\mathcal N(0,1)$ are independent. We consider the scalar kernel
\[
\mathcal K(x,x')
=
\sum_{k\ge1}\nu_k\langle x,e_k\rangle_{L^2}\langle x',e_k\rangle_{L^2},
\qquad
\nu_k=k^{-\beta},
\]
and define the operator-valued kernel $K(x,x')=\mathcal K(x,x')I_{\Y}$ with $\Y=L^2[0,1]$. We choose the target operator in the form $H^\dagger=S^\dagger C^r,$ where $C$ is the covariance operator associated with $\mathcal K$. The observations are generated by
\[
y_t=h^\dagger(x_t)+\epsilon_t,
\qquad
\epsilon_t=\frac{\sigma}{\sqrt m}\sum_{k=1}^m \zeta_k^t e_k,
\]
where $\zeta_k^t\sim\mathcal N(0,1)$ are independent. The experimental results are displayed in Figure~\ref{fig:twofigs_vertical}; see Appendix~\ref{Detailed setup} for additional details. The results suggest that regularization yields smaller errors and improved stability. In particular, regularized SGD consistently outperforms unregularized SGD in prediction error and shows an even more pronounced advantage in estimation error.

In practice, both the step sizes and the regularization parameters can be tuned according to validation performance, and the output iterate can be determined by validation-based early stopping. This stopping strategy can be regarded as a form of implicit regularization. Our theoretical results also indicate that the benefit of explicit regularization is especially pronounced under a decaying step-size schedule.

\section{Error Decomposition} \label{decomposition}
       
In this section, we present the error decomposition employed in the convergence analysis of upper bounds.  We begin with several useful observations.
 
        For any $H\in\B_{\mathrm{HS}}(\H_\K,\Y)$, by the definition of $\E(H)$ in the main text,
        \begin{equation*}
            \begin{aligned}
                \mathcal{E}(H)-\mathcal{E}(H^{\dagger})
		&=\mathbb{E}\left[\|y-H\phi(x)\|_{\Y}^{2}\right]-\mathbb{E}\left[\|y-H^\dagger\phi(x)\|_{\Y}^{2}\right]
		\\&=\mathbb{E}\left[\|(H^\dagger-H)\phi(x)+\epsilon\|_{\Y}^{2}\right]-\sigma^2 
		\\&=\mathbb{E}\left[\|(H-H^\dagger)\phi(x)\|_{\Y}^{2}\right]+2\mathbb{E}\left[\langle \epsilon,(H^\dagger-H)\phi(x)\rangle_{\Y}\right]. 
            \end{aligned}
        \end{equation*}
		Since $\epsilon$ is a centered noise independent of $x$, we have $\mathbb{E}\left[\langle \epsilon,(H^\dagger-H)\phi(x)\rangle_{\Y}\right]=0$. Therefore,
	\begin{equation*}
		\mathcal{E}(H)-\mathcal{E}(H^{\dagger})=\mathbb{E}\left[\|(H-H^\dagger)\phi(x)\|_{\Y}^{2}\right].
	\end{equation*}
	Furthermore, suppose that $\{f_j\}_{j\geq1}$ is an orthonormal basis of the separable Hilbert space $\Y$. We express $(H-H^\dagger)\phi(x)$ using a Fourier expansion: 
	\begin{equation}
		\begin{aligned}
			\mathcal{E}(H)-\mathcal{E}(H^{\dagger})&=\mathbb{E}\left[\sum_{j\geq1}\langle(H-H^\dagger)\phi(x),f_j\rangle_\Y^2\right]
			\\ &=\sum_{j\geq1}\mathbb{E}\left[\left\langle(H-H^\dagger)\phi(x)\otimes \phi(x)(H-H^\dagger)^* f_j,f_j\right\rangle_\Y\right]
			\\ &=\left\|(H-H^\dagger)C^{\frac{1}{2}}\right\|_{\mathrm{HS}}^{2}.
		\end{aligned}	\label{form2}
	\end{equation}
	
	Our goal in this paper is to estimate $\left\|(H-H^\dagger)C^{\alpha}\right\|_{\mathrm{HS}}^{2}$ for $\alpha=0$ or $1/2$, corresponding respectively to the estimation error and the prediction error, for a given estimator $H$.
     
     We define the regularizing operator as	
    \begin{align} \label{temp1}
		H_\lambda&:=\mathop{\arg\min}_{H\in\mathcal{B}_{\mathrm{HS}}(\H_\K,\Y)} \mathcal{E}(H)+\lambda\|H\|_\mathrm{HS}^2 
	\end{align}
        for some $\lambda>0$. By computing the Fr$\acute{e}$chet derivative on $H$, we obtain
        \begin{equation} \label{regular}
            H_\lambda=H^\dagger C(C+\lambda I)^{-1}=S^\dagger C^{1+r}(C+\lambda I)^{-1},
        \end{equation}
	where the final equality follows from Assumption \ref{a2}.

    As introduced in Section \ref{introduction}, we consider two types of step sizes and regularization parameters. Both can be uniformly expressed in the following form:
    \begin{equation} \label{setting}
        \begin{cases}
            \eta_t=\bar{\eta}(t+t_0)^{-\theta_1},&\\
            \lambda_t=\bar{\lambda}(t+t_0)^{-\theta_2},&
        \end{cases}
    \end{equation}
    where $t_0\geq0$, $\eta_t$ is the step size with $\theta_1\in[0,1)$ and $\lambda_t$ denotes the regularization parameter with $\theta_2\in[0,1)$.
    To avoid confusion, we clarify the parameter settings below:
    \begin{enumerate}
        \item The online setting. In this setting, in \eqref{setting} we require that $\theta_1, \theta_2 \in (0,1)$, $t_0 > 0$, and $\bar{\eta} , \bar{\lambda}> 0$ be constants independent of the current iteration $t$.
        
        \item The finite-horizon setting. In this setting, we set
\[
\eta_t \equiv \bar{\eta}, \qquad \lambda_t \equiv \bar{\lambda},
\qquad t=1,2,\dots,T,
\]
where $\bar{\eta}=\eta_\star T^{-\theta_3}$ and $\bar{\lambda}=\lambda_\star T^{-\theta_4}$, with $t_0=\theta_1=\theta_2=0$, $\theta_3\in(0,1)$, and $\theta_4>0$. Unlike the decaying case, here $\bar{\eta} = \bar{\eta}(T)$ and $\bar{\lambda} = \bar{\lambda}(T)$ depend on $T$, while $\eta_\star$ and $\lambda_\star$ are constants independent of $T$.
    \end{enumerate}

	\begin{lemma}
		Let $\{H_t\}_{t\geq1}$ be defined as \eqref{iter}. Then, we have
		\begin{equation} \label{temp3}
			\begin{aligned}
				H_{t+1}-H_{\lambda_t}=&(H_t-H_{\lambda_{t-1}})(I-\eta_t(C+\lambda_t I))
				\\&+(H_{\lambda_{t-1}}-H_{\lambda_{t}})(I-\eta_t(C+\lambda_t I))+\eta_{t}\B_t, 
			\end{aligned}
		\end{equation}
		where $I$ denotes the identity operator, and $\B_t$ is defined by
		\begin{equation*}
			\B_t=(H_t-H^\dagger)C+(y_t-H_t\phi(x_t))\otimes \phi(x_t). 
		\end{equation*}
		Moreover, for any $t\in\bn_T$, it holds that $\mathbb{E}_{z_t}[\B_t]=0.$
	\end{lemma}	
	
	\begin{proof}
		From \eqref{regular}, we have $H^\dagger C=H_{\lambda_t}(C+\lambda_t I)$. Combining this with the update rule in algorithm \eqref{iter}, we obtain the equality in \eqref{temp3}, which can be directly verified.
    
	Note that $H_t$ depends on $z^{t-1}$ and is independent of $z_t$. Therefore, we have
        \begin{equation*}
            \begin{aligned}
                \mathbb{E}_{z_t}[\B_t]&=(H_t-H^\dagger)C+\mathbb{E}_{z_t}[(y_t-H_t\phi(x_t))\otimes \phi(x_t)]\\
			&=(H_t-H^\dagger)C+\mathbb{E}_{z_t}[\left((H^\dagger-H_t)\phi(x_t)+\epsilon_t\right)\otimes \phi(x_t)].
            \end{aligned}
        \end{equation*}
		Since $\epsilon_t$ is a centered noise independent of $x_t$, it follows that $\mathbb{E}_{z_t}[\epsilon_t\otimes \phi(x_t)]=0$. Hence,
		\[
		\mathbb{E}_{z_t}[\B_t]=(H_t-H^\dagger)C+(H^\dagger-H_t)C=0.
		\]
		The proof is then completed.
	\end{proof}

     We set $\lambda_0=\bar\lambda t_0^{-\theta_2}$ in the online setting, and $\lambda_0=\bar\lambda$ in the finite-horizon setting. Let $\prod_{j=T+1}^{T}(I-\eta_j(C+\lambda_j I))=I$. By applying induction to the equality \eqref{temp3}, we derive the following key identity used in the error decomposition:
	\begin{equation} \label{induction}
		\begin{aligned}
			H_{T+1}-H_{\lambda_T}=&(H_T-H_{\lambda_{T-1}})(I-\eta_T(C+\lambda_T I)) 
			\\&+(H_{\lambda_{T-1}}-H_{\lambda_{T}})(I-\eta_T(C+\lambda_T I))+\eta_{T}\B_T 
			\\=&\cdots 
			\\=&-H_{\lambda_0}\prod_{t=1}^{T}\left(I-\eta_{t}(C+\lambda_{t}I)\right)
			+\sum_{t=1}^{T}(H_{\lambda_{t-1}}-H_{\lambda_{t}})\prod_{j=t}^{T}(I-\eta_j(C+\lambda_j I))
			\\&+\sum_{t=1}^{T}\eta_{t}\B_{t}\prod_{j=t+1}^{T}(I-\eta_j(C+\lambda_j I)).
		\end{aligned}
	\end{equation}

    In the next proposition, we decompose the expectation of the prediction error (when \(\alpha=1/2\)) and the estimation error (when \(\alpha=0\)), $\be_{z^{T}}\!\left[
\left\|\left(H_{T+1}-H^{\dagger}\right)C^\alpha\right\|_{\mathrm{HS}}^2
\right],$
into four terms that can be estimated separately.

Its derivation, as well as the subsequent expectation estimates, builds on ideas from earlier work. In the scalar case, \citet{ying2008online} studied kernel SGD for functional regression. For Hilbert-space-valued outputs, \citet{ciliberto2020general,brogat2022vector} used the isometric identification between RKHSs of the form \(K=\mathcal{K}I\) and Hilbert--Schmidt operator spaces (an identification already available in \citet{carmeli2010vector}) to analyze least-squares-type methods through an equivalent Hilbert--Schmidt operator formulation in that setting. Later, \citet{guo2023capacity,shi2024learning} extended the analysis of \citet{ying2008online} to the linear-operator setting and studied SGD in that framework.

The present proposition builds on this line of work, but the introduction of regularization under Assumption~\ref{a1}  requires additional terms and estimates that are specific to the regularized SGD analysis developed in this paper. In this sense, while the overall decomposition strategy is connected to the existing literature, the parts of the argument involving regularization are new in the present paper.
    \begin{proposition}\label{Proposition error 1}
		Let $\{H_t\}_{t\in\bn_T}$ be defined as \eqref{iter}. Suppose that Assumption \ref{a4} holds with some $c>0$. Then, for any $T\geq1$ and $0\leq\alpha\leq\frac{1}{2}$, the following inequality holds:
            \begin{equation}
                \be_{z^{T}}\left[
		\left\|\left(H_{T+1}-H^{\dagger}\right)C^\alpha\right\|_{\mathrm{HS}}^2\right]
            \leq
            \T_1+\T_2+\T_3+\T_4,
            \end{equation}
            where
            \begin{equation} \label{t}
                \begin{split}
                    &\T_1:=2\left\|(H_{\lambda_T}-H^\dagger)C^\alpha\right\|_{\mathrm{HS}}^{2},  \\
                    &\T_2:=6\left\|H_{\lambda_0}C^\alpha\prod_{t=1}^{T}\left(I-\eta_{t}(C+\lambda_{t}I)\right)\right\|^2_{\mathrm{HS}}, \\
                    &\T_3:=6\left\|\sum_{t=1}^{T}(H_{\lambda_{t-1}}-H_{\lambda_{t}})C^\alpha\prod_{j=t}^{T}(I-\eta_j(C+\lambda_j I))\right\|^2_{\mathrm{HS}}, \\
                    &\T_4:=6\sqrt{c}\sum_{t=1}^{T}\eta_t^{2}\left(\sqrt{c}\be_{z^{t-1}}\mathbb E_{x_t}\left\|(H_t-H^\dagger)\phi(x_t)\right\|_{\Y}^2+\sigma^2\right)\mathrm{Tr}\left(C^{1+2\alpha}\prod_{j=t+1}^{T}(I-\eta_j(C+\lambda_j I))^2\right).
                \end{split}
            \end{equation}
	\end{proposition}
	
	\begin{proof}
		Since $(H_{T+1}-H^\dagger)C^\alpha=(H_{T+1}-H_{\lambda_T})C^\alpha+(H_{\lambda_T}-H^\dagger)C^\alpha$, we have 
		\begin{equation*}
			\begin{aligned}
				\be_{z^{T}}\left[
				\left\|\left(H_{T+1}-H^{\dagger}\right)C^\alpha\right\|_{\mathrm{HS}}^2\right]
				&=\be_{z^{T}}\left[\left\|(H_{T+1}-H_{\lambda_T})C^\alpha+(H_{\lambda_T}-H^\dagger)C^\alpha\right\|_{\mathrm{HS}}^{2}\right]
				\\&\leq2\be_{z^{T}}\left[\left\|(H_{T+1}-H_{\lambda_T})C^\alpha\right\|_{\mathrm{HS}}^{2}\right]+2\left\|(H_{\lambda_T}-H^\dagger)C^\alpha\right\|_{\mathrm{HS}}^{2}.
			\end{aligned}
		\end{equation*}
	We aim to bound $\be_{z^{T}}\left[\left\|(H_{T+1}-H_{\lambda_T})C^\alpha\right\|_{\mathrm{HS}}^{2}\right]$.
		From the equality \eqref{induction}, it follows that
			\begin{align}
				(H_{T+1}-H_{\lambda_T})C^\alpha=&-H_{\lambda_0}C^\alpha\prod_{t=1}^{T}\left(I-\eta_{t}(C+\lambda_{t}I)\right)
				+\sum_{t=1}^{T}(H_{\lambda_{t-1}}-H_{\lambda_{t}})C^\alpha\prod_{j=t}^{T}(I-\eta_j(C+\lambda_j I)) \nonumber
				\\&+\sum_{t=1}^{T}\eta_{t}\B_{t}C^\alpha\prod_{j=t+1}^{T}(I-\eta_j(C+\lambda_j I))=:J_1+J_2+J_3. \label{temp31}
			\end{align}
            Then, 
            \begin{equation*} 
                \be_{z^{T}}\left[\left\|(H_{T+1}-H_{\lambda_T})C^\alpha\right\|_{\mathrm{HS}}^{2}\right]
                \leq 3\|J_1\|_{\mathrm{HS}}^2+3\|J_2\|_{\mathrm{HS}}^2+
                3\be_{z^T}\left[\|J_3\|_{\mathrm{HS}}^2\right].
            \end{equation*}

		We express $\be_{z^{T}}\left[\|J_3\|^2_{\mathrm{HS}}\right]=\be_{z^{T}}\left[\left\|\sum_{t=1}^{T}\eta_{t}\B_{t}C^\alpha\prod_{j=t+1}^{T}(I-\eta_j(C+\lambda_j I))\right\|_{\mathrm{HS}}^{2}\right]$ as
		\[
 \sum_{t=1}^{T}\sum_{t^\prime=1}^{T}\eta_{t}\eta_{t^\prime}\be_{z^{T}}\left[\left\langle\B_{t}C^\alpha\prod_{j=t+1}^{T}(I-\eta_j(C+\lambda_j I)),\B_{t^\prime}C^\alpha\prod_{j=t^\prime+1}^{T}(I-\eta_j(C+\lambda_j I))\right\rangle_{\mathrm{HS}}\right].
		\]
		Using the property $\mathbb{E}_{z_t}[\B_t]=0$, for $t>t^\prime$,  we obtain
        \begin{equation*} 
            \begin{aligned}
			&\be_{z^{T}}\left[\left\langle\B_{t}C^\alpha\prod_{j=t+1}^{T}(I-\eta_j(C+\lambda_j I)),\B_{t^\prime}C^\alpha\prod_{j=t^\prime+1}^{T}(I-\eta_j(C+\lambda_j I))\right\rangle_{\mathrm{HS}}\right]
			\\ &=\be_{z^{t-1}}\be_{z_{t}}\left[\left\langle\B_{t}C^\alpha\prod_{j=t+1}^{T}(I-\eta_j(C+\lambda_j I)),\B_{t^\prime}C^\alpha\prod_{j=t^\prime+1}^{T}(I-\eta_j(C+\lambda_j I))\right\rangle_{\mathrm{HS}}\right]
			\\ &=\be_{z^{t-1}}\left[\left\langle\be_{z_{t}}\B_{t}C^\alpha\prod_{j=t+1}^{T}(I-\eta_j(C+\lambda_j I)),\B_{t^\prime}C^\alpha\prod_{j=t^\prime+1}^{T}(I-\eta_j(C+\lambda_j I))\right\rangle_{\mathrm{HS}}\right]=0.
		\end{aligned}
        \end{equation*}
		Similarly, the above equality also holds for $t<t^\prime$.
		Consequently, there holds	
            \begin{equation} \label{temp5}
                \be_{z^T}\left[\|J_3\|_{\mathrm{HS}}^2\right]
=\sum_{t=1}^{T}\be_{z^{T}}\left[\left\|\eta_{t}\B_{t}C^\alpha\prod_{j=t+1}^{T}(I-\eta_j(C+\lambda_j I))\right\|_{\mathrm{HS}}^{2}\right].
            \end{equation}
		
		Using the property $\mathbb{E}_{z_t}[\B_t]=0$ again, we have
		\[\B_{t}=-\mathbb{E}_{z_t}[(y_t-H_t\phi(x_t))\otimes \phi(x_t)]+(y_t-H_t\phi(x_t))\otimes \phi(x_t).\]
		Denote $\eta_t\left[(y_t-H_t\phi(x_t))\otimes \phi(x_t)\right]C^\alpha\prod_{j=t+1}^{T}(I-\eta_j(C+\lambda_j I))$ by $\A$, then substituting $\A$ into (\ref{temp5}) yields that
        \begin{equation} \label{temp-1}
            \begin{gathered}
			\be_{z^{t}}\left[\left\|\eta_{t}\B_{t}C^\alpha\prod_{j=t+1}^{T}(I-\eta_j(C+\lambda_j I))\right\|_{\mathrm{HS}}^{2}\right]=\be_{z^{t-1}}\be_{z_{t}}\left[\left\|-\mathbb{E}_{z_t}[\A]+\A\right\|_{\mathrm{HS}}^{2}\right]
			\\ \leq \be_{z^{t}}\left[\left\|\A\right\|_{\mathrm{HS}}^{2}\right]
			=\be_{z^{t}}\left[\left\|\eta_t\left[(y_t-H_t\phi(x_t))\otimes \phi(x_t)\right]C^\alpha\prod_{j=t+1}^{T}(I-\eta_j(C+\lambda_j I))\right\|_{\mathrm{HS}}^{2}\right].
		\end{gathered}
        \end{equation}
		Take  $\{e_i\}_{i\geq1}$ to be an orthonormal basis of Hilbert space $\H_\K$. Since $C$ is self-adjoint, by \eqref{temp5}, \eqref{temp-1} and the definition of the Hilbert-Schmidt norm, there holds
        \begin{equation} \label{temp6}
            \begin{aligned}
&\be_{z^T}\left[\|J_3\|_{\mathrm{HS}}^2\right]
\\ \leq&\sum_{t=1}^{T}\be_{z^{t}}\left[\sum_{i\geq1}\left\|\eta_t\left[(y_t-H_t\phi(x_t))\otimes \phi(x_t)\right]C^\alpha\prod_{j=t+1}^{T}(I-\eta_j(C+\lambda_j I))e_i\right\|_{\H_\K}^{2}\right]
			\\ =&\sum_{t=1}^{T}\be_{z^{t}}\left[\sum_{i\geq1}\left\|\eta_t(y_t-H_t\phi(x_t))\right\|_{\Y}^2\left\langle \phi(x_t),C^\alpha\prod_{j=t+1}^{T}(I-\eta_j(C+\lambda_j I))e_i\right\rangle_{\H_\K}^2\right]
			\\ =&\sum_{t=1}^{T}\eta_t^{2}\be_{z^{t}}\left[\left\|y_t-H_t\phi(x_t)\right\|_{\Y}^{2}\left\|C^\alpha\prod_{j=t+1}^{T}(I-\eta_j(C+\lambda_j I))\phi(x_t)\right\|_{\H_\K}^{2}\right]
			\\=&\sum_{t=1}^{T}\eta_t^{2}\be_{z^{t-1}}\be_{x_t}\left[\be_{\epsilon_t}\left\|(H^\dagger-H_t)\phi(x_t)+\epsilon_t\right\|_{\Y}^{2}\left\|C^\alpha\prod_{j=t+1}^{T}(I-\eta_j(C+\lambda_j I))\phi(x_t)\right\|_{\H_\K}^{2}\right],
		\end{aligned}
        \end{equation}
	where we use $y_t=H^\dagger\phi(x_t)+\epsilon_t$ in the last equality.
		It is obvious that 
        $$\be_{\epsilon_t}\left[\left\|(H^\dagger-H_t)\phi(x_t)+\epsilon_t\right\|_{\Y}^{2}\right]
		=\|(H_t-H^\dagger)\phi(x_t)\|_{\Y}^2+\sigma^2,$$ where $\sigma^2=\be[\|\epsilon\|_{\Y}^2]$ is the variance of $\epsilon$. Substitute it back into \eqref{temp6} and use the Cauchy-Schwarz inequality. Then we obtain
		\begin{align*}
\be_{z^T}\left[\|J_3\|_{\mathrm{HS}}^2\right]&\leq\sum_{t=1}^{T}\eta_t^{2}\be_{z^{t-1}}\be_{x_{t}}\left[\left(\left\|(H^\dagger-H_t)\phi(x_t)\right\|_{\Y}^{2}+\sigma^2\right)\left\|C^\alpha\prod_{j=t+1}^{T}(I-\eta_j(C+\lambda_j I))\phi(x_t)\right\|_{\H_\K}^{2}\right]
			\\&\leq\sum_{t=1}^{T}\eta_t^{2}\left(\be_{z^{t-1}}\sqrt{\be_{x_t}\left\|(H_t-H^\dagger)\phi(x_t)\right\|_{\Y}^4}+\sigma^2\right)
			\\ &\qquad \times\left(\be_{x_t}\left\|C^\alpha\prod_{j=t+1}^{T}(I-\eta_j(C+\lambda_j I))\phi(x_t)\right\|_{\H_\K}^{4}\right)^{1/2}
			\\ &\leq \sqrt{c}\sum_{t=1}^{T}\eta_t^{2}\left(\sqrt{c}\be_{z^{t-1}}\mathbb E_{x_t}\left\|(H_t-H^\dagger)\phi(x_t)\right\|_{\Y}^2+\sigma^2\right)
            \\ &\qquad\times\be_{x_t}\left\|C^\alpha\prod_{j=t+1}^{T}(I-\eta_j(C+\lambda_j I))\phi(x_t)\right\|_{\H_\K}^{2},
		\end{align*}
	where the last inequality is due to Assumption \ref{a4} and Proposition \ref{lemma in Appendix} in Appendix \ref{Appendix lemma}. Since 
 \begin{align*}
     \be_{x_t}\left\|C^\alpha\prod_{j=t+1}^{T}(I-\eta_j(C+\lambda_j I))\phi(x_t)\right\|_{\H_\K}^{2}
     &=\sum_{i\geq1}\be_{x_t}\left\langle C^\alpha\prod_{j=t+1}^{T}(I-\eta_j(C+\lambda_j I))\phi(x_t), e_i\right\rangle_{\H_\K}^2
     \\&=\sum_{i\geq1}\left\langle C^{1+2\alpha}\prod_{j=t+1}^{T}(I-\eta_j(C+\lambda_j I))^2 e_i, e_i\right\rangle_{\H_\K}
     \\&=\mathrm{Tr}\left(C^{1+2\alpha}\prod_{j=t+1}^{T}(I-\eta_j(C+\lambda_j I))^2\right),
 \end{align*}
	there holds
    \begin{align*}
        \be_{z^T}\left[\|J_3\|_{\mathrm{HS}}^2\right]
 &\leq\sqrt{c}\sum_{t=1}^{T}\eta_t^{2}\left(\sqrt{c}\be_{z^{t-1}}\mathbb E_{x_t}\left\|(H_t-H^\dagger)\phi(x_t)\right\|_{\Y}^2+\sigma^2\right)
 \\&\qquad\times\mathrm{Tr}\left(C^{1+2\alpha}\prod_{j=t+1}^{T}(I-\eta_j(C+\lambda_j I))^2\right),
    \end{align*}
 which finishes our proof.
	\end{proof}

Hereafter, we refer to $\T_1$ as the approximation error, $\T_2$ as the initial error, $\T_3$ as the drift error, and $\T_4$ as the sample error, respectively.

We now present the error decomposition of $\left\|\left(H_{T+1}-H^{\dagger}\right)C^\alpha\right\|_{\mathrm{HS}}^2$, which serves to establish a high-probability upper bound. For any random variable $\mu$ taking values in $\B_{\mathrm{HS}}(\H_\K,\Y)$, we denote the $L^{\infty}$ norm of $\|\mu\|_{\mathrm{HS}}$ by $\|\mu\|_{L^\infty_{\mathrm{HS}}}$.

\begin{proposition} \label{proposition error2}
    Let $\{H_t\}_{t\in\bn_T}$ be defined as \eqref{iter}. Suppose that Assumption \ref{a5} holds for some $M_\rho>0$. Then, for any $T\geq1$ and $0\leq\alpha\leq\frac{1}{2}$, the quantity 
    $\left\|\left(H_{T+1}-H^{\dagger}\right)C^\alpha\right\|_{\mathrm{HS}}^2$
		admits the decomposition
            \begin{equation}
		\left\|\left(H_{T+1}-H^{\dagger}\right)C^\alpha\right\|_{\mathrm{HS}}^2
            \leq
            \T_1+\T_2+\T_3+6\left\|\sum_{t=1}^T\chi_t\right\|^2_{\mathrm{HS}},
            \end{equation}
        where $\T_1$, $\T_2$, and $\T_3$ are defined in \eqref{t}, and $\chi_t=\eta_{t}\B_{t}C^\alpha\prod_{j=t+1}^{T}(I-\eta_j(C+\lambda_j I))$
        satisfies
        \begin{equation} \label{temp63}
        \|\chi_t\|_{\mathrm{HS}}\leq2\eta_t\kappa\l(M_\rho+\kappa\l\|H_t\r\|_{L^\infty_{\mathrm{HS}}}\r)\left\|C^\alpha\prod_{j=t+1}^{T}(I-\eta_j(C+\lambda_j I))\right\|, \quad  \forall t\in\bn_T.
        \end{equation}
\end{proposition}

\begin{proof}
    The proof follows a similar strategy to the previous proposition. As in the proof of Proposition \ref{Proposition error 1}, we readily obtain
    \[
    \left\|\left(H_{T+1}-H^{\dagger}\right)C^\alpha\right\|_{\mathrm{HS}}^2
    \leq2\left\|(H_{T+1}-H_{\lambda_T})C^\alpha\right\|_{\mathrm{HS}}^{2}+2\left\|(H_{\lambda_T}-H^\dagger)C^\alpha\right\|_{\mathrm{HS}}^{2}
    \]
    and 
    \begin{equation*}
        \begin{aligned}
            \left\|(H_{T+1}-H_{\lambda_T})C^\alpha\right\|_{\mathrm{HS}}^{2}\leq
            3\|J_1\|_{\mathrm{HS}}^2+3\|J_2\|_{\mathrm{HS}}^2+
                3\|J_3\|_{\mathrm{HS}}^2,
        \end{aligned}
    \end{equation*}
    where $J_1$, $J_2$, and $J_3$ are defined in \eqref{temp31}. Defining $\chi_t=\eta_{t}\B_{t}C^\alpha\prod_{j=t+1}^{T}(I-\eta_j(C+\lambda_j I))$, we then have $J_3=\sum_{t=1}^{T}\chi_t$. Since $\B_t$ can be expressed as
    \[
    \B_t=(y_t-H_t\phi(x_t))\otimes \phi(x_t)-\be_{z_t}\l[(y_t-H_t\phi(x_t))\otimes \phi(x_t)\r],
    \]
    it follows from Assumption \ref{a5} that
    \[
    \|\B_t\|_\HS\leq2\l\|(y_t-H_t\phi(x_t))\otimes \phi(x_t)\r\|_{L^\infty_{\mathrm{HS}}}
    \leq2\kappa\l(M_\rho+\kappa\l\|H_t\r\|_{L^\infty_{\mathrm{HS}}}\r).
    \]
    Thus,
    \[
    \|\chi_t\|_{\mathrm{HS}}\leq2\eta_t\kappa\l(M_\rho+\kappa\l\|H_t\r\|_{L^\infty_{\mathrm{HS}}}\r)\left\|C^\alpha\prod_{j=t+1}^{T}(I-\eta_j(C+\lambda_j I))\right\|.
    \]
    
    The proof is then finished.
\end{proof}

	\section{Intermediate Estimates for Error Analysis} \label{section:basic}
	In this section, we derive bounds for $\T_1$, $\T_2$, $\T_3$, and $\T_4$. The bounds for $\T_1$ and $\T_2$ are presented in a unified form encompassing the finite-horizon setting. The term $\T_3$ arises exclusively in the online setting and is therefore analyzed only within that context. For $\T_4$, we provide separate bounds corresponding to these two settings. These intermediate results play an important role in the subsequent analysis of prediction and estimation errors, both in expectation and with high probability.

	\subsection{Bounding Approximation Error}
        We bound $\T_1$ in the following proposition.
	\begin{proposition} \label{prop2}
		Under the Assumption \ref{a2} with $r>0$, there exists a constant $c_1=c_1(\|S^\dagger\|_{\mathrm{HS}},r,\alpha,\kappa)$ independent of $t_0$, $T$, $\bar{\eta}$, and $\bar{\lambda}$, such that
		\begin{equation*}
			\T_1=2\left\|(H_{\lambda_T}-H^\dagger)C^{\alpha}\right\|_{\mathrm{HS}}^2
			\leq c_1\lambda_T^{\min\{2(r+\alpha),2\}}.
		\end{equation*}
	\end{proposition}
	
\begin{proof}
		We know from \eqref{regular} that 
		\begin{align*}
			H_{\lambda_T}-H^\dagger&=H^\dagger C(C+\lambda_T I)^{-1}-H^\dagger
			\\&=-\lambda_T H^\dagger(C+\lambda_T I)^{-1}=-\lambda_T S^\dagger C^r(C+\lambda_T I)^{-1},
		\end{align*}
        where the last identity uses Assumption \ref{a2}.
		It then follows that 
            \begin{equation} \label{temp7}
                \begin{aligned}
			\left\|(H_{\lambda_T}-H^\dagger)C^{\alpha}\right\|_{\mathrm{HS}}
			&=\|\lambda_T S^\dagger C^{r+\alpha}(C+\lambda_T I)^{-1}\|_{\mathrm{HS}}
			\\&\leq\lambda_T\|S^\dagger\|_{\mathrm{HS}}\left\|C^{r+\alpha}(C+\lambda_T I)^{-1}\right\|.
		\end{aligned}
            \end{equation}
		Since
        \begin{equation} \label{temp18}
            \begin{aligned}
		&\left\|C^{r+\alpha}(C+\lambda_T I)^{-1}\right\|\leq\sup_{0\leq x\leq\kappa^2}\frac{x^{r+\alpha}}{x+\lambda_{T}}
		\\&\leq\begin{cases}
			\kappa^{2(r+\alpha-1)}, & \text{ when } r+\alpha\geq1,\\
			(r+\alpha)^{r+\alpha}(1-r-\alpha)^{1-r-\alpha}\lambda_{T}^{r+\alpha-1}, & \text{ when } r+\alpha<1,
		\end{cases}
	\end{aligned}
        \end{equation}
		combining \eqref{temp7} with \eqref{temp18}, there exists a constant $c_1$ such that
	\begin{equation*}
		\left\|(H_{\lambda_T}-H^\dagger)C^{\alpha}\right\|_{\mathrm{HS}}^2
		\leq c_1\lambda_T^{\min\{2(r+\alpha),2\}}.
	\end{equation*}
        It is clear that $c_1$ is independent of $t_0$, $T$, $\bar{\eta}$, and $\bar{\lambda}$, which completes this proof.
\end{proof}
	
	\subsection{Bounding Initial Error}
    The following two lemmas will be used repeatedly throughout our analysis.
	\begin{lemma}  \label{lemma1}
		Let $\beta>0$ and let $\eta_t,\lambda_t$ be defined as  \eqref{setting}. Let $l,m$ be integers satisfying $1\leq l \leq m$. Suppose that $(t_0+1)^{\theta_1}\geq\bar{\eta}(\kappa^2+\bar{\lambda})$. Then, the following estimates hold:
        \begin{itemize}
            \item[(1)] $\left\|C^\beta\prod_{t=l}^m\left(I-\eta_t(C+\lambda_{t}I)\right)\right\|\leq\exp\left\{-\sum_{t=l}^{m}\eta_{t}\lambda_t\right\}\frac{2(\kappa^{2\beta}+(\beta/e)^\beta)}{1+\left(\sum_{t=l}^{m}\eta_{t}\right)^\beta}.$
            \item[(2)] $\left\|C^\beta\prod_{t=l}^m\left(I-\eta_t(C+\lambda_{t}I)\right)^2\right\|
			\leq\left(\frac{\beta}{2e}\right)^{\beta}\left(\sum_{t=l}^{m}\eta_{t}\right)^{-\beta}\exp\left\{-2\sum_{t=l}^{m}\eta_{t}\lambda_{t}\right\}$.
            \item[(3)] $\left\|C^\beta\prod_{t=l}^m\left(I-\eta_t(C+\lambda_{t}I)\right)^2\right\|
		\leq\exp\left\{-2\sum_{t=l}^{m}\eta_{t}\lambda_t\right\}\frac{2(\kappa^{2\beta}+(\beta/(2e))^\beta)}{1+\left(\sum_{t=l}^{m}\eta_{t}\right)^\beta}.$
        \end{itemize}
	\end{lemma}

	\begin{proof}
            Recall that $C$ (defined in Section \ref{results}) is self-adjoint and compact. By the definition of the operator norm, we have
            \begin{equation} \label{temp8}
                \begin{aligned}
			\left\|C^\beta\prod_{t=l}^m\left(I-\eta_t(C+\lambda_{t}I)\right)\right\|
			&\leq\sup_{0\leq x\leq\kappa^2}x^\beta\prod_{t=l}^m\left(1-\eta_t(x+\lambda_{t})\right)
			\\&\leq\sup_{0\leq x\leq\kappa^2}x^\beta\exp\left\{-\sum_{t=l}^{m}\eta_{t}(x+\lambda_{t})\right\}
			\\&=\left(\frac{\beta}{e}\right)^{\beta}\left(\sum_{t=l}^{m}\eta_{t}\right)^{-\beta}\exp\left\{-\sum_{t=l}^{m}\eta_{t}\lambda_{t}\right\},
		\end{aligned}
            \end{equation}
            where the first inequality follows from the fact
            $1-\eta_t(x+\lambda_{t})\geq0$ for all $t\geq 1$ and $0\leq x\leq\kappa^2$, which is ensured by the condition $(t_0+1)^{\theta_1}\geq\bar{\eta}(\kappa^2+\bar{\lambda})$.
		On the other hand, we also have
		\begin{align} \label{temp19}
			\left\|C^\beta\prod_{t=l}^m\left(I-\eta_t(C+\lambda_{t}I)\right)\right\|
			\leq\kappa^{2\beta}\prod_{t=l}^m(1-\eta_{t}\lambda_{t})
			\leq\kappa^{2\beta}\exp\left\{-\sum_{t=l}^{m}\eta_{t}\lambda_t\right\}.
		\end{align}
		Applying the inequality $\min\{a,b\}\leq\frac{2}{1/a+1/b},  \forall a,b>0$ and combining \eqref{temp8} with \eqref{temp19}, we obtain
		\begin{equation*}
			\left\|C^\beta\prod_{t=l}^m\left(I-\eta_t(C+\lambda_{t}I)\right)\right\|\leq\exp\left\{-\sum_{t=l}^{m}\eta_{t}\lambda_t\right\}\frac{2(\kappa^{2\beta}+(\beta/e)^\beta)}{1+\left(\sum_{t=l}^{m}\eta_{t}\right)^\beta}.
		\end{equation*}
            Now, using \eqref{temp19} once more, there holds
		\begin{align*}
			\left\|C^\beta\prod_{t=l}^m\left(I-\eta_t(C+\lambda_{t}I)\right)^2\right\|
			&=\left\|C^{\beta/2}\prod_{t=l}^m\left(I-\eta_t(C+\lambda_{t}I)\right)\right\|^2
			\\&\leq\left(\frac{\beta}{2e}\right)^{\beta}\left(\sum_{t=l}^{m}\eta_{t}\right)^{-\beta}\exp\left\{-2\sum_{t=l}^{m}\eta_{t}\lambda_{t}\right\}.
		\end{align*}
	Moreover, since $\left\|C^\beta\prod_{t=l}^m\left(I-\eta_t(C+\lambda_{t}I)\right)^2\right\|\leq\kappa^{2\beta}\exp\left\{-\sum_{t=l}^{m}\eta_{t}\lambda_t\right\}$, applying $\min\{a,b\}\leq\frac{2}{1/a+1/b}$ again yields
	\begin{align*}
		\left\|C^\beta\prod_{t=l}^m\left(I-\eta_t(C+\lambda_{t}I)\right)^2\right\|
		\leq\exp\left\{-2\sum_{t=l}^{m}\eta_{t}\lambda_t\right\}\frac{2(\kappa^{2\beta}+(\beta/(2e))^\beta)}{1+\left(\sum_{t=l}^{m}\eta_{t}\right)^\beta}.
	\end{align*}
        This completes the proof.
	\end{proof}

        The next lemma establishes lower bounds for $\sum_{t=l}^{m}\eta_{t}$ and $\sum_{t=1}^{T}\eta_{t}\lambda_{t}$.
	\begin{lemma} \label{lemma2}
		Let $0\leq\theta_1<1$, $0\leq\theta_2<1$, and $\eta_t,\lambda_t$ be defined as \eqref{setting}. Then the following bounds hold for $1\leq l\leq m$ with $l \in \mathbb{N}$:
            \begin{itemize}
                \item[(1)] $\sum_{t=l}^{m}\eta_{t}\geq\frac{\bar{\eta}}{1-\theta_1}\left[(m+t_0+1)^{1-\theta_1}-(l+t_0)^{1-\theta_1}\right].$
                \item[(2)] 
                \begin{equation*}
                    \sum_{t=l}^{m}\eta_{t}\lambda_{t}\geq
                    \begin{cases}
					\frac{\bar{\eta}\bar{\lambda}}{1-\theta_1-\theta_2}\left[(m+t_0+1)^{1-\theta_1-\theta_2}-(l+t_0)^{1-\theta_1-\theta_2}\right], & \text{ when } \theta_1+\theta_2\not=1, \\
					\bar{\eta}\bar{\lambda}\log\left(\frac{m+t_0+1}{t_0+l}\right), & \text{ when } \theta_1+\theta_2=1.
				    \end{cases}
                \end{equation*}
                In particular, when $l=1$ and $m=T$ with $T\geq t_0+1$, we have:
                \item[(3)] $\sum_{t=1}^{T}\eta_{t}\geq\frac{1-2^{\theta_1-1}}{1-\theta_1}\bar{\eta}(T+t_0)^{1-\theta_1}.$ 
                \item[(4)] \begin{equation*}
			\sum_{t=1}^{T}\eta_{t}\lambda_{t}\geq
			\begin{cases}
				\frac{\bar{\eta}\bar{\lambda}}{1-\theta_1-\theta_2}(1-2^{\theta_1+\theta_2-1})(T+t_0)^{1-\theta_1-\theta_2}, & \text{ when } 0\leq\theta_1+\theta_2<1, \\
				\bar{\eta}\bar{\lambda}\log\left(\frac{T+t_0}{t_0+1}\right), & \text{ when } \theta_1+\theta_2=1, \\
				\frac{\bar{\eta}\bar{\lambda}}{\theta_1+\theta_2-1}
				(1-2^{1-\theta_1-\theta_2})(t_0+1)^{1-\theta_1-\theta_2}, & \text{ when } \theta_1+\theta_2>1.
			\end{cases}
		\end{equation*}
            \end{itemize}
	\end{lemma}
	\begin{proof}
            We bound the summation $\sum_{t=l}^{m}\eta_{t}$ using
		\begin{equation*}
			\begin{aligned}
				\sum_{t=l}^{m}\eta_{t}&=\bar{\eta}\sum_{t=l}^{m}(t+t_0)^{-\theta_1}
				\geq\bar{\eta}\int_{l}^{m+1}(x+t_0)^{-\theta_1}dx
				\\&=\frac{\bar{\eta}}{1-\theta_1}\left[(m+t_0+1)^{1-\theta_1}-(l+t_0)^{1-\theta_1}\right].
			\end{aligned}
		\end{equation*}
		For the specific case where $l=1$, $m=T$, and $T\geq t_0+1$, we obtain
		\begin{align*}
			\sum_{t=1}^{T}\eta_{t}&\geq\frac{\bar{\eta}}{1-\theta_1}\left[(T+t_0+1)^{1-\theta_1}-(t_0+1)^{1-\theta_1}\right]
			\\ &\geq\frac{1-2^{\theta_1-1}}{1-\theta_1}\bar{\eta}(T+t_0)^{1-\theta_1}.
		\end{align*}
            Next, we analyze the summation involving $\lambda_t$ using the same estimate as before:
		\begin{equation*}
			\begin{aligned}
				\sum_{t=l}^{m}\eta_{t}\lambda_{t}&=\bar{\eta}\bar{\lambda}\sum_{t=l}^{m}(t+t_0)^{-\theta_1-\theta_2}\geq\bar{\eta}\bar{\lambda}\int_{l}^{m+1}(x+t_0)^{-\theta_1-\theta_2}dx
				\\&=
				\begin{cases}
					\frac{\bar{\eta}\bar{\lambda}}{1-\theta_1-\theta_2}\left[(m+t_0+1)^{1-\theta_1-\theta_2}-(l+t_0)^{1-\theta_1-\theta_2}\right], & \text{ when } \theta_1+\theta_2\not=1, \\
					\bar{\eta}\bar{\lambda}\log\left(\frac{m+t_0+1}{t_0+l}\right), & \text{ when } \theta_1+\theta_2=1.
				\end{cases}
			\end{aligned}
		\end{equation*}
        For the case $l=1$, $m=T$, and $T\geq t_0+1$, we obtain:
		\begin{equation*}
			\sum_{t=1}^{T}\eta_{t}\lambda_{t}\geq
			\begin{cases}
				\frac{\bar{\eta}\bar{\lambda}}{1-\theta_1-\theta_2}(1-2^{\theta_1+\theta_2-1})(T+t_0)^{1-\theta_1-\theta_2}, & \text{ when } \theta_1+\theta_2<1, \\
				\bar{\eta}\bar{\lambda}\log\left(\frac{T+t_0}{t_0+1}\right), & \text{ when } \theta_1+\theta_2=1, \\
				\frac{\bar{\eta}\bar{\lambda}}{\theta_1+\theta_2-1}
				(1-2^{1-\theta_1-\theta_2})(t_0+1)^{1-\theta_1-\theta_2}, & \text{ when } \theta_1+\theta_2>1.
			\end{cases}
		\end{equation*}

        The proof is then finished.
	\end{proof}

    Next, based on the two lemmas above, we provide a unified upper bound for $\T_2$ under the following two settings:
    \begin{itemize}
        \item[(1)] $0<\theta_1<1$, $0<\theta_2<1$, and $t_0>0$, corresponding to the online setting;
        \item[(2)] $\theta_1=\theta_2=t_0=0$, corresponding to the finite-horizon setting.
    \end{itemize}
    The bound in both cases is established by the following proposition.

	\begin{proposition} \label{prop3}
        Suppose that Assumption \ref{a2} holds with $S^\dagger\in \B_{\mathrm{HS}}(\H_\K,\Y)$ and $r>0$. Assume also that
$(t_0+1)^{\theta_1}\geq\bar\eta(\kappa^2+\bar\lambda)$. Then, for any $T\geq t_0+1$, $t_0\geq0$, $0\leq\theta_1<1$, and $0\leq\theta_2<1$, the quantity
		\[
			\T_2=6\left\|H_{\lambda_0}C^\alpha\prod_{t=1}^{T}\left(I-\eta_{t}(C+\lambda_{t}I)\right)\right\|_{\mathrm{HS}}^2
		\]
        admits the following bound:
        \begin{equation*}
            \T_2\leq c_2\bar{\eta}^{-2(r+\alpha)}
				\begin{cases}
					(T+t_0)^{-2(r+\alpha)(1-\theta_1)}\exp\{-\tau\bar{\eta}\bar{\lambda}(T+t_0)^{1-\theta_1-\theta_2}\}, & \text{ when } 0\leq\theta_1+\theta_2<1, \\
					(t_0+1)^{2\bar{\eta}\bar{\lambda}}(T+t_0)^{-2(r+\alpha)(1-\theta_1)-2\bar{\eta}\bar{\lambda}}, & \text{ when } \theta_1+\theta_2=1, \\
					(T+t_0)^{-2(r+\alpha)(1-\theta_1)}, & \text{ when }  \theta_1+\theta_2>1,
				\end{cases}
        \end{equation*}
        where $c_2=c_2(\|S^\dagger\|_{\mathrm{HS}},r,\alpha,\theta_1)$ and $\tau=\tau(\theta_1,\theta_2)$
are positive constants independent of $t_0$, $T$, $\bar{\eta}$, and $\bar{\lambda}$.
		
	\end{proposition}

	\begin{proof}
            According to equality \eqref{regular}, it follows that
		\begin{align}
			\left\|H_{\lambda_0}C^\alpha\prod_{t=1}^{T}\left(I-\eta_{t}(C+\lambda_{t}I)\right)\right\|_{\mathrm{HS}}^2
			&=\left\|S^\dagger C^{r+\alpha+1}(C+\lambda_0I)^{-1}\prod_{t=1}^{T}\left(I-\eta_{t}(C+\lambda_{t}I)\right)\right\|_{\mathrm{HS}}^2 \nonumber
			\\ &\leq\|S^\dagger\|_{\mathrm{HS}}^2\left\|C^{2(r+\alpha)}\prod_{t=1}^{T}\left(I-\eta_{t}(C+\lambda_{t}I)\right)^2\right\|.
		\end{align}
		By applying (2) in Lemma \ref{lemma1} with $\beta=2(r+\alpha)$ and Lemma \ref{lemma2} (3), the following inequality holds:
		\begin{equation}
			\begin{aligned}
				&\left\|H_{\lambda_0}C^\alpha\prod_{t=1}^{T}\left(I-\eta_{t}(C+\lambda_{t}I)\right)\right\|_{\mathrm{HS}}^2
				\\\leq&\|S^\dagger\|_{\mathrm{HS}}^2\left(\frac{r+\alpha}{e}\right)^{2(r+\alpha)}\left(\sum_{t=1}^{T}\eta_{t}\right)^{-2(r+\alpha)}\exp\left\{-2\sum_{t=1}^{T}\eta_{t}\lambda_{t}\right\}
				\\\leq&\|S^\dagger\|_{\mathrm{HS}}^2\left(\frac{(r+\alpha)(1-\theta_1)}{e(1-2^{\theta_1-1})}\right)^{2(r+\alpha)}\bar{\eta}^{-2(r+\alpha)}(T+t_0)^{-2(r+\alpha)(1-\theta_1)}\exp\left\{-2\sum_{t=1}^{T}\eta_{t}\lambda_{t}\right\}.
			\end{aligned}
		\end{equation}
		Next, using (4) in Lemma \ref{lemma2}, the exponential term can be bounded as:		
		\begin{equation*}
			\exp\left\{-2\sum_{t=1}^{T}\eta_{t}\lambda_{t}\right\}
			\leq\begin{cases}
				\exp\{-\tau\bar{\eta}\bar{\lambda}(T+t_0)^{1-\theta_1-\theta_2}\}, & \text{ when } 0\leq\theta_1+\theta_2<1, \\
				(t_0+1)^{2\bar{\eta}\bar{\lambda}}(T+t_0)^{-2\bar{\eta}\bar{\lambda}}, & \text{ when } \theta_1+\theta_2=1, \\
				1, & \text{ when } \theta_1+\theta_2>1,
			\end{cases}
		\end{equation*}
		where $\tau=\frac{2}{1-\theta_1-\theta_2}(1-2^{\theta_1+\theta_2-1})$. Therefore, the bound for $\T_2$ becomes:
		\begin{equation}
			\begin{aligned}
				\T_2&=6\left\|H_{\lambda_0}C^\alpha\prod_{t=1}^{T}\left(I-\eta_{t}(C+\lambda_{t}I)\right)\right\|_{\mathrm{HS}}^2
				\\ &\leq 
				c_2\bar{\eta}^{-2(r+\alpha)}
				\begin{cases}
					(T+t_0)^{-2(r+\alpha)(1-\theta_1)}\exp\{-\tau\bar{\eta}\bar{\lambda}(T+t_0)^{1-\theta_1-\theta_2}\}, & \text{ when }  0\leq\theta_1+\theta_2<1, \\
					(t_0+1)^{2\bar{\eta}\bar{\lambda}}(T+t_0)^{-2(r+\alpha)(1-\theta_1)-2\bar{\eta}\bar{\lambda}}, & \text{ when } \theta_1+\theta_2=1, \\
					(T+t_0)^{-2(r+\alpha)(1-\theta_1)}, & \text{ when } \theta_1+\theta_2>1,
				\end{cases}
			\end{aligned}
		\end{equation}
		where $c_2=6\|S^\dagger\|_{\mathrm{HS}}^2\left(\frac{(r+\alpha)(1-\theta_1)}{e(1-2^{\theta_1-1})}\right)^{2(r+\alpha)}$ is independent of $t_0$, $T$, $\bar{\eta}$, and $\bar{\lambda}$.

        The desired result is established and the proof is complete.
	\end{proof}

Proposition~\ref{prop3} is written in a unified form in order to cover both the online setting and the finite-horizon setting. In particular, in the finite-horizon case, the bound is obtained by applying the proposition with $\theta_1=\theta_2=0$, while $\theta_3$ and $\theta_4$ enter through the definitions of the constant step size $\bar\eta$ and regularization parameter $\bar\lambda$.
    
	\subsection{Bounding Drift Error}
        In the finite-horizon setting, where $\lambda_t = \bar{\lambda}$ is fixed depending on $T$, we have $\T_3 = 0$. Therefore, it is sufficient to bound $\T_3$ under the regime of decaying step sizes and regularization parameters. In what follows, we focus on the setting where $0 < \theta_1 < 1$ and $0 < \theta_2 < 1$.
	\begin{lemma} \label{lemma3}
     Suppose that Assumption \ref{a2} holds with $S^\dagger\in \B_{\mathrm{HS}}(\H_\K,\Y)$ and $r>0$ and let $t_0\geq1$. Then, for any $t\geq1$, the following bound holds:
		\begin{equation*}
			\left\|H_{\lambda_{t-1}}-H_{\lambda_{t}}\right\|_{\mathrm{HS}}
                \leq \widetilde{c}_3\bar{\lambda}^{\min\{r,1\}}(t+t_0)^{-\theta_2\min\{r,1\}-1},
		\end{equation*}
            where $\widetilde{c}_3=\widetilde{c}_3(\|S^\dagger\|_{\mathrm{HS}},r,\theta_2,\kappa)$ is a constant independent of $t_0$, $t$, $\bar\eta$, and $\bar\lambda$.
	\end{lemma}

	\begin{proof}
		Based on the expression for $H_\lambda$ in \eqref{regular}
        and under Assumption \ref{a2}, we deduce that
            \begin{equation} \label{temp10}
                \begin{aligned}
                    \left\|H_{\lambda_{t-1}}-H_{\lambda_{t}}\right\|_{\mathrm{HS}}
			&=\left\|H^\dagger C(C+\lambda_{t-1} I)^{-1}-H^\dagger C(C+\lambda_t I)^{-1}\right\|_{\mathrm{HS}}
			\\&=\lvert\lambda_t-\lambda_{t-1}\rvert\left\|S^\dagger C^{r+1}(C+\lambda_{t}I)^{-1}(C+\lambda_{t-1}I)^{-1}\right\|_{\mathrm{HS}}
			\\&\leq\|S^\dagger\|_{\mathrm{HS}}\l\lvert\lambda_t-\lambda_{t-1}\r\rvert\left\|C^r(C+\lambda_{t}I)^{-1}\right\|
			\\&\leq\|S^\dagger\|_{\mathrm{HS}}\bar{\lambda}\l\lvert(t+t_0-1)^{-\theta_2}-(t+t_0)^{-\theta_2}\r\rvert
			\begin{cases}
				\kappa^{2r-2}, & \text{ when } r\geq1, \\
				r^r(1-r)^{1-r}\lambda_{t}^{r-1}, & \text{ when } r<1,
			\end{cases}
                \end{aligned}
            \end{equation}
		where the last inequality uses the fact that
		\[
			\left\|C^r(C+\lambda_{t}I)^{-1}\right\|
			\leq\sup_{0\leq x\leq\kappa^2}\left\{x^r(x+\lambda_{t})^{-1}\right\}
			\leq\begin{cases}
				\kappa^{2r-2}, & \text{ when } r\geq1, \\
				r^r(1-r)^{1-r}\lambda_{t}^{r-1}, & \text{ when } r<1.
			\end{cases}
		\]
		Applying the mean value theorem, there exists $\xi\in(0,1)$ such that
		\begin{align}
			\lvert(t+t_0-1)^{-\theta_2}-(t+t_0)^{-\theta_2}\rvert \nonumber
			&=\theta_2(t+t_0-\xi)^{-(\theta_2+1)}\leq\theta_2(t+t_0-1)^{-(\theta_2+1)}
			\\&\leq2^{\theta_2+1}\theta_2(t+t_0)^{-(\theta_2+1)} \label{temp9},
		\end{align}
	where the last inequality uses $t+t_0-1\geq(t+t_0)/2$. Substituting \eqref{temp9} into \eqref{temp10}, we arrive at
		\begin{align*}
			\left\|H_{\lambda_{t-1}}-H_{\lambda_{t}}\right\|_{\mathrm{HS}}
			&\leq2^{\theta_2+1}\theta_2\|S^\dagger\|_{\mathrm{HS}}\bar{\lambda}(t+t_0)^{-(\theta_2+1)}
			\begin{cases}
				\kappa^{2r-2}, & \text{ when } r\geq1, \\
				r^r(1-r)^{1-r}\lambda_{t}^{r-1}, & \text{ when } r<1,
			\end{cases}
			\\&\leq \widetilde{c}_3\bar{\lambda}^{\min\{r,1\}}(t+t_0)^{-\theta_2\min\{r,1\}-1},
		\end{align*}
		where $\widetilde{c}_3$ is a constant independent of $t_0$, $t$, $\bar{\eta}$, and $\bar{\lambda}$.

        The proof is complete.
	\end{proof}

        Now, we derive an error bound for $\T_3$ in the case $\alpha=0$, applied to the analysis of estimation error.
	\begin{proposition} \label{prop7}
		Suppose that Assumption \ref{a2} holds with $S^\dagger\in \B_{\mathrm{HS}}(\H_\K,\Y)$ and $r>0$. Set $\alpha=0$ in $\T_3$ and assume $t_0\geq1$. Let $T\geq t_0+1$ when $\theta_1+\theta_2<1$, and $T\geq1$ otherwise. Additionally, assume $(t_0+1)^{\theta_1}\geq\bar{\eta}(\kappa^2+\bar{\lambda})$ and $\bar\eta\bar\lambda>\theta_2\min\{r,1\}$. Then, the following bound holds for $\T_3$:
		\begin{equation*}
		\begin{aligned}
			\T_3=6\left\|\sum_{t=1}^{T}(H_{\lambda_{t-1}}-H_{\lambda_{t}})\prod_{j=t}^{T}(I-\eta_j(C+\lambda_j I))\right\|^2_{\mathrm{HS}}
			\leq
			c_3
			\begin{cases}
				1, & \text{ when } \theta_1+\theta_2>1, \\
				(T+t_0)^{-2\theta_2\min\{r,1\}}, & \text{ when } \theta_1+\theta_2\leq1,
			\end{cases}
		\end{aligned}
	\end{equation*}
	where $c_3=c_3(t_0,\bar\lambda,\bar\eta,\|S^\dagger\|_{\mathrm{HS}},r,\theta_1,\theta_2,\kappa)$ is a constant independent of\/ $T$.
	\end{proposition}

	\begin{proof}
		By Lemma \ref{lemma3}, we deduce that
            \begin{equation}
                \begin{aligned}
                    &\left\|\sum_{t=1}^{T}(H_{\lambda_{t-1}}-H_{\lambda_{t}})\prod_{j=t}^{T}(I-\eta_j(C+\lambda_j I))\right\|_{\mathrm{HS}}
			\\ \leq&\widetilde{c}_3\bar{\lambda}^{\min\{r,1\}}\sum_{t=1}^{T}(t+t_0)^{-\theta_2\min\{r,1\}-1}\left\|\prod_{j=t}^{T}(I-\eta_j(C+\lambda_j I))\right\|.
                \end{aligned}
            \end{equation}
            Since $C$ is self-adjoint and compact, with the operator norm $\|C\|\leq\kappa^2$ (see Section \ref{results}), it follows that
            \begin{equation}
                \begin{aligned}
                    \left\|\prod_{j=t}^{T}(I-\eta_j(C+\lambda_j I))\right\|
			\leq\sup_{0\leq x\leq\kappa^2}\prod_{j=t}^{T}(1-\eta_j(x+\lambda_j))
			\\\leq\sup_{0\leq x\leq\kappa^2}\left\{\exp\left\{-\sum_{j=t}^{T}\eta_j(x+\lambda_j)\right\}\right\}
			\leq\exp\left\{-\sum_{j=t}^{T}\eta_j\lambda_j\right\}.
                \end{aligned}
            \end{equation}	
            Hence, we have the following bound for $\left\|\sum_{t=1}^{T}(H_{\lambda_{t-1}}-H_{\lambda_{t}})\prod_{j=t}^{T}(I-\eta_j(C+\lambda_j I))\right\|_{\mathrm{HS}}$:
            \[
            \widetilde{c}_3\bar{\lambda}^{\min\{r,1\}}\sum_{t=1}^{T}(t+t_0)^{-\theta_2\min\{r,1\}-1}\exp\left\{-\sum_{j=t}^{T}\eta_j\lambda_j\right\}.
            \]
		From (2) in Lemma \ref{lemma2}, we know that $\exp\left\{-\sum_{j=t}^{T}\eta_j\lambda_j\right\}$ is bounded by
        \[
        \begin{cases}
					\exp\left\{-\frac{\bar{\eta}\bar{\lambda}}{1-\theta_1-\theta_2}\left[(T+t_0+1)^{1-\theta_1-\theta_2}-(t+t_0)^{1-\theta_1-\theta_2}\right]\right\}
					, & \text{ when } \theta_1+\theta_2\not=1, \\
					\exp\left\{-\bar{\eta}\bar{\lambda}\log\left(\frac{T+t_0+1}{t+t_0}\right)\right\}, & \text{ when } \theta_1+\theta_2=1.
				\end{cases}
        \]
        Thus, we obtain the following bound for the exponential term:
        \begin{equation*}
            \begin{aligned}
                &\exp\left\{-\sum_{j=t}^{T}\eta_j\lambda_j\right\}
                \\\leq&
                \begin{cases}
				\exp\left\{-\frac{\bar{\eta}\bar{\lambda}}{1-\theta_1-\theta_2}\left[(T+t_0+1)^{1-\theta_1-\theta_2}-(t+t_0)^{1-\theta_1-\theta_2}\right]\right\}
				, & \text{ when } \theta_1+\theta_2<1, \\
				(t+t_0)^{\bar{\eta}\bar{\lambda}}(T+t_0)^{-\bar{\eta}\bar{\lambda}}, & \text{ when } \theta_1+\theta_2=1, \\
				1, & \text{ when } \theta_1+\theta_2>1.
			\end{cases}
            \end{aligned}
        \end{equation*}
        We next consider the three cases corresponding to $\theta_1 + \theta_2 > 1$, $\theta_1 + \theta_2 = 1$, and $\theta_1 + \theta_2 < 1$.
        \\
	\textbf{Case 1:} When $\theta_1+\theta_2>1$, we obtain
	\begin{align*}
		&\left\|\sum_{t=1}^{T}(H_{\lambda_{t-1}}-H_{\lambda_{t}})\prod_{j=t}^{T}(I-\eta_j(C+\lambda_j I))\right\|_{\mathrm{HS}}
		\\\leq&\widetilde{c}_3\bar{\lambda}^{\min\{r,1\}}\sum_{t=1}^{T}(t+t_0)^{-\theta_2\min\{r,1\}-1} \leq\widetilde{c}_3\frac{t_0^{-\theta_2\min\{r,1\}}}{\theta_2\min\{r,1\}}\bar{\lambda}^{\min\{r,1\}}.
	\end{align*}
	\textbf{Case 2:} When $\theta_1+\theta_2=1$, we derive the following bound:
	\begin{align*}
		&\left\|\sum_{t=1}^{T}(H_{\lambda_{t-1}}-H_{\lambda_{t}})\prod_{j=t}^{T}(I-\eta_j(C+\lambda_j I))\right\|_{\mathrm{HS}}
		\\ \leq&\widetilde{c}_3\bar{\lambda}^{\min\{r,1\}}(T+t_0)^{-\bar{\eta}\bar{\lambda}}\sum_{t=1}^{T}(t+t_0)^{-\theta_2\min\{r,1\}-1+\bar{\eta}\bar{\lambda}}
		\\\leq&\frac{\widetilde{c}_3}{\bar{\eta}\bar{\lambda}-\theta_2\min\{r,1\}}\bar{\lambda}^{\min\{r,1\}}(T+t_0)^{-\theta_2\min\{r,1\}},
	\end{align*}
        where the last inequality uses the condition that $\bar\eta\bar\lambda>\theta_2\min\{r,1\}$.\\   
	\textbf{Case 3:} When $\theta_1+\theta_2<1$, there holds that
	\begin{gather*}
		\left\|\sum_{t=1}^{T}(H_{\lambda_{t-1}}-H_{\lambda_{t}})\prod_{j=t}^{T}(I-\eta_j(C+\lambda_j I))\right\|_{\mathrm{HS}}
		\leq\widetilde{c}_3\bar{\lambda}^{\min\{r,1\}}\sum_{t=1}^{T}(t+t_0)^{-\theta_2\min\{r,1\}-1}
		\\\times\exp\left\{-\frac{\bar{\eta}\bar{\lambda}}{1-\theta_1-\theta_2}\left[(T+t_0+1)^{1-\theta_1-\theta_2}-(t+t_0)^{1-\theta_1-\theta_2}\right]\right\}.
	\end{gather*}
        Now, we estimate the summation in the last inequality. Since $T\geq t_0+1$, we have $t+t_0\leq \frac34(T+t_0+1)$ when $t\leq\frac T2$.  By splitting the summation into two parts, from $1$ to $T/2$ and from $T/2$ to $T$, we deduce that
	\begin{align*}
		&\sum_{t=1}^{T}(t+t_0)^{-\theta_2\min\{r,1\}-1}
		\exp\left\{-\frac{\bar{\eta}\bar{\lambda}}{1-\theta_1-\theta_2}\left[(T+t_0+1)^{1-\theta_1-\theta_2}-(t+t_0)^{1-\theta_1-\theta_2}\right]\right\}
		\\ \leq&\sum_{t=1}^{T/2}(t+t_0)^{-\theta_2\min\{r,1\}-1}\exp\left\{-\frac{\bar{\eta}\bar{\lambda}}{1-\theta_1-\theta_2}\left(1-\left(3/4\right)^{1-\theta_1-\theta_2}\right)
		(T+t_0)^{1-\theta_1-\theta_2}\right\}
		\\&+\sum_{t=T/2}^{T}(t+t_0)^{-\theta_2\min\{r,1\}-1}
		\\\leq&\frac{t_0^{-\theta_2\min\{r,1\}}}{\theta_2\min\{r,1\}}\exp\left\{-\frac{\bar{\eta}\bar{\lambda}}{1-\theta_1-\theta_2}\left(1-\left(3/4\right)^{1-\theta_1-\theta_2}\right)(T+t_0)^{1-\theta_1-\theta_2}\right\}
		\\&+\frac{4^{\theta_2\min\{r,1\}}-1}{\theta_2\min\{r,1\}}(T+t_0)^{-\theta_2\min\{r,1\}}.
	\end{align*}
	Using the fact that for any constants $k,\gamma>0$, there exists a constant $m$ such that  
    \[
    \exp\left\{-k(T+t_0)^{1-\theta_1-\theta_2}\right\}\leq m\left(T+t_0\right)^{-\gamma},
    \]
    we conclude from this that
	\begin{equation*}
		\begin{aligned}
			\T_3&=6\left\|\sum_{t=1}^{T}(H_{\lambda_{t-1}}-H_{\lambda_{t}})\prod_{j=t}^{T}(I-\eta_j(C+\lambda_j I))\right\|^2_{\mathrm{HS}}
			\\ &\leq
			c_3
			\begin{cases}
				1, & \text{ when } \theta_1+\theta_2>1, \\
				(T+t_0)^{-2\theta_2\min\{r,1\}}, & \text{ when } \theta_1+\theta_2\leq1, 
			\end{cases}
		\end{aligned}
	\end{equation*}
	where $c_3$ is a constant independent of $T$.
	
        We then finish the proof.
	\end{proof}

        \begin{remark}
            We will only use the bound for the case $\theta_1+\theta_2=1$ in the above proposition, as it provides better convergence rates than the other cases. Note that we cannot guarantee convergence of the estimation error when $\theta_1+\theta_2>1$. However, convergence of the prediction error is assured when $\theta_1+\theta_2>1$ with $0<\theta_1,\theta_2<1$, and both the prediction and estimation errors converge when $\theta_1+\theta_2\leq1$. The proofs for the remaining cases are similar and are omitted here to avoid repetition.
        \end{remark}

       The following proposition plays a key role in deriving upper bounds for the drift and sample errors. Its technical proof is provided in Appendix \ref{Appendix 1}.

	\begin{proposition} \label{prop1}
		Let $v>0$, $\theta\in\mathbb{R}$, $t_0\geq1$, and $T\geq t_0+1$. The step size $\eta_t$ is defined as \eqref{setting}. Suppose $\bar{\eta}\bar{\lambda}>\theta-1$ and $\theta_1+\theta_2=1$. Then, 
		\begin{align*}
			\sum_{t=1}^{T}\exp\left\{-\sum_{j=t+1}^{T}\eta_{j}\lambda_j\right\}\frac{(t+t_0)^{-\theta}}{1+\left(\sum_{j=t+1}^{T}\eta_{j}\right)^{v}}
			\leq\delta_1
			\begin{cases}
				(T+t_0)^{-\theta+\theta_1}, & \text{ when } v>1, \\
				(T+t_0)^{-\theta+\theta_1}\log(T+t_0), & \text{ when } v=1, \\
				(T+t_0)^{-\theta+1-v(1-\theta_1)}, & \text{ when } v<1,
			\end{cases}
		\end{align*}
		where $\delta_1=\delta_1(\bar\lambda,\bar\eta,\theta_1,\theta)$ is a constant independent of\/ $T$ and $t_0$.
	\end{proposition}

    \begin{remark} \label{remark1}
        The inequality in Proposition \ref{prop1} remains valid for all $T \geq 1$, not only when $T \geq t_0 + 1$, provided that the constant $\delta_1$ is allowed to depend on $t_0$ and is chosen sufficiently large. More specifically, there exists a constant $\delta_2 = \delta_2(t_0, \bar\lambda, \bar\eta)$ such that, for any $T\geq1$,
        \begin{align*}
			\sum_{t=1}^{T}\exp\left\{-\sum_{j=t+1}^{T}\eta_{j}\lambda_j\right\}\frac{(t+t_0)^{-\theta}}{1+\left(\sum_{j=t+1}^{T}\eta_{j}\right)^{v}}
			\leq\delta_2
			\begin{cases}
				(T+t_0)^{-\theta+\theta_1}, & \text{ when } v>1, \\
				(T+t_0)^{-\theta+\theta_1}\log(T+t_0), & \text{ when } v=1, \\
				(T+t_0)^{-\theta+1-v(1-\theta_1)}, & \text{ when } v<1.
			\end{cases}
		\end{align*}
        Note that $\delta_2=\delta_2(t_0,\bar\lambda,\bar\eta,\theta_1,\theta)$ depends on $t_0$, whereas $\delta_1$ does not.
    \end{remark}

    The following proposition establishes a bound on $\T_3$ for the case $\alpha = \frac{1}{2}$, which is instrumental in analyzing the prediction error.
	\begin{proposition} \label{prop4}
		Suppose that Assumption \ref{a2} holds with $S^\dagger\in \B_{\mathrm{HS}}(\H_\K,\Y)$ and $r>0$. Set $\alpha=1/2$ in $\T_3$. Let $\theta_1+\theta_2=1$, $t_0\geq1$ and $T\geq t_0+1$. Suppose that $(t_0+1)^{\theta_1}\geq\bar{\eta}(\kappa^2+\bar{\lambda})$ and $\bar\eta\bar\lambda>\theta_2\min\{r,1\}$. Then, there holds 
		\begin{equation}
			\T_3=6\left\|\sum_{t=1}^{T}(H_{\lambda_{t-1}}-H_{\lambda_{t}})C^{1/2}\prod_{j=t}^{T}(I-\eta_j(C+\lambda_j I))\right\|^2_{\mathrm{HS}}\leq c_3(T+t_0)^{-2\theta_2\min\{r,1\}+\theta_1-1},
		\end{equation}
        where $c_3=c_3(\bar\lambda,\bar\eta,\|S^\dagger\|_{\mathrm{HS}},r,\theta_1,\kappa)$ is a constant independent of\/ $T$ and $t_0$.
	\end{proposition}

	\begin{proof}
		Applying Lemma \ref{lemma3}, we have
            \begin{equation*}
                \begin{aligned}
                    &\left\|\sum_{t=1}^{T}(H_{\lambda_{t-1}}-H_{\lambda_{t}})C^{1/2}\prod_{j=t}^{T}(I-\eta_j(C+\lambda_j I))\right\|_{\mathrm{HS}}
                \\\leq&\widetilde{c}_3\bar{\lambda}^{\min\{r,1\}}\sum_{t=1}^{T}(t+t_0)^{-\theta_2\min\{r,1\}-1}\left\|C^{1/2}\prod_{j=t}^{T}(I-\eta_j(C+\lambda_j I))\right\|.
                \end{aligned}
            \end{equation*}
            Using Lemma \ref{lemma1} (1) with $\beta=\frac12$, the above inequality is further bounded as
		\begin{align*}
			&\left\|\sum_{t=1}^{T}(H_{\lambda_{t-1}}-H_{\lambda_{t}})C^{1/2}\prod_{j=t}^{T}(I-\eta_j(C+\lambda_j I))\right\|_{\mathrm{HS}}
			\\
			\leq&
			2(\kappa+(1/(2e))^{1/2})\widetilde{c}_3\bar{\lambda}^{\min\{r,1\}}\sum_{t=1}^{T}\exp\left\{-\sum_{j=t}^{T}\eta_{j}\lambda_j\right\}\frac{(t+t_0)^{-\theta_2\min\{r,1\}-1}}{1+\left(\sum_{j=t}^{T}\eta_{j}\right)^{1/2}}
			\\ \leq&
			2(\kappa+(1/(2e))^{1/2})\widetilde{c}_3\bar{\lambda}^{\min\{r,1\}}\delta_1(T+t_0)^{-\theta_2\min\{r,1\}-(1-\theta_1)/2},
		\end{align*}
            where in the last inequality we use Proposition \ref{prop1} with 
            $\theta=\theta_2\min\{r,1\}+1$ and $v=1/2<1$. As a consequence,
		\begin{equation*}
			\T_3=6\left\|\sum_{t=1}^{T}(H_{\lambda_{t-1}}-H_{\lambda_{t}})C^{1/2}\prod_{j=t}^{T}(I-\eta_j(C+\lambda_j I))\right\|^2_{\mathrm{HS}}
			\leq c_3(T+t_0)^{-2\theta_2\min\{r,1\}+\theta_1-1},
		\end{equation*}
	    where $c_3$ is a constant independent of $T$ and $t_0$. 
     
        The proof is then finished. \end{proof}

	\subsection{Bounding Sample Error}
         Let $\be_{z^0}[\xi]=\xi$ for any random variable $\xi$.
         The next proposition applies to the online setting.
	\begin{proposition} \label{prop5}
		Suppose that Assumption \ref{a3} holds with $0<s\leq1$. If $t_0\geq1$, $\theta_1+\theta_2=1$, $0<\theta_1<1$, $0<\theta_2<1$, $T\geq t_0+1$, $2\bar\eta\bar\lambda>2\theta_1-1,$ $(t_0+1)^{\theta_1}\geq\bar\eta(\kappa^2+\bar\lambda)$, and the following condition holds for any 
        $t\in\mathbb{N}_T$:
            \begin{equation} \label{temp-2}
                \be_{z^{t-1}}\be_{x_t}\left\|(H_t-H^\dagger)\phi(x_t)\right\|_{\Y}^2\left(=\be_{z^{t-1}}\left[
		\left\|\left(H_{t}-H^{\dagger}\right)C^{\frac12}\right\|_{\mathrm{HS}}^2\right]\right)\leq M,
            \end{equation}
            where $M$ is independent of $T$. Then, the following bound holds for $\T_4$:
		\begin{align*}
			\T_4&=6\sqrt{c}\sum_{t=1}^{T}\eta_t^{2}\left(\sqrt{c}\be_{z^{t-1}}\be_{x_t}\left\|(H_t-H^\dagger)\phi(x_t)\right\|_{\Y}^2+\sigma^2\right)\mathrm{Tr}\left(C^{1+2\alpha}\prod_{j=t+1}^{T}(I-\eta_j(C+\lambda_j I))^2\right)
			\\ &\leq c_4\left(\sqrt{c}M+\sigma^2\right)
			\begin{cases}
				(T+t_0)^{-\theta_1}, & \text{ when } 2\alpha>s, \\
				(T+t_0)^{-\theta_1}\log(T+t_0), & \text{ when } 2\alpha=s, \\
				(T+t_0)^{-(1+s-2\alpha)\theta_1+s-2\alpha}, & \text{ when } 2\alpha<s,
			\end{cases}
		\end{align*}
		where $c_4=c_4(\bar\lambda,\bar\eta,\theta_1,\alpha,s,c,
\kappa,\mathrm{Tr}(C^s))$ is a constant independent of \/$T$, $t_0$, and $M$.
	\end{proposition}

	\begin{proof}
            Assumption \ref{a3} on $C$ guarantees that
            \[\mathrm{Tr}\left(C^{1+2\alpha}\prod_{j=t+1}^{T}(I-\eta_j(C+\lambda_j I))^2\right)\leq\mathrm{Tr}\left(C^s\right)\left\|C^{1+2\alpha-s}\prod_{j=t+1}^{T}(I-\eta_j(C+\lambda_j I))^2\right\|.\]
            For $1+2\alpha-s>0$, we use (3) in Lemma \ref{lemma1} with $\beta=1+2\alpha-s$ and \eqref{temp-2} to bound the operator norm as
		\begin{align*}
			&\sum_{t=1}^{T}\eta_t^{2}\left(\sqrt{c}\be_{z^{t-1}}\be_{x_t}\left\|(H_t-H^\dagger)\phi(x_t)\right\|_{\Y}^2+\sigma^2\right)\mathrm{Tr}\left(C^{1+2\alpha}\prod_{j=t+1}^{T}(I-\eta_j(C+\lambda_j I))^2\right)
			\\\leq&
			\sum_{t=1}^{T}\eta_t^{2}\left(\sqrt{c}M+\sigma^2\right)\mathrm{Tr}\left(C^s\right)\left\|C^{1+2\alpha-s}\prod_{j=t+1}^{T}(I-\eta_j(C+\lambda_j I))^2\right\|
			\\\leq&
			2\left(\kappa^{2+4\alpha-2s}+((1+2\alpha-s)/(2e))^{1+2\alpha-s}\right)\left(\sqrt{c}M+\sigma^2\right)\mathrm{Tr}\left(C^s\right)
			\\&
			\times\sum_{t=1}^{T}\exp\left\{-2\sum_{j=t+1}^{T}\eta_{j}\lambda_j\right\}
			\frac{\eta_t^{2}}{1+\left(\sum_{j=t+1}^{T}\eta_{j}\right)^{1+2\alpha-s}}.
		\end{align*}
        Now, multiplying the preceding bound by $6\sqrt c$
and applying Proposition \ref{prop1} with $\theta=2\theta_1$,
$v=1+2\alpha-s$, and $2\bar\lambda$ in place of
$\bar\lambda$, we get
            \begin{equation*}
                \begin{aligned}
			\T_4&=6\sqrt{c}\sum_{t=1}^{T}\eta_t^{2}\left(\sqrt{c}\be_{z^{t-1}}\be_{x_t}\left\|(H_t-H^\dagger)\phi(x_t)\right\|_{\Y}^2+\sigma^2\right)\mathrm{Tr}\left(C^{1+2\alpha}\prod_{j=t+1}^{T}(I-\eta_j(C+\lambda_j I))^2\right)
			\\ &\leq c_4\left(\sqrt{c}M+\sigma^2\right)
			\begin{cases}
				(T+t_0)^{-\theta_1}, & \text{ when } 2\alpha>s, \\
				(T+t_0)^{-\theta_1}\log(T+t_0), & \text{ when } 2\alpha=s, \\
				(T+t_0)^{-(1+s-2\alpha)\theta_1+s-2\alpha}, & \text{ when } 2\alpha<s,
			\end{cases}
		\end{aligned}
            \end{equation*}
            where, for $1+2\alpha-s>0$, $c_4=12\sqrt{c}\left(\kappa^{2+4\alpha-2s}+((1+2\alpha-s)/(2e))^{1+2\alpha-s}\right)\mathrm{Tr}\left(C^s\right)\delta_1\bar\eta^2$.

For $\alpha=0$ and $s=1$,
since $2\bar\eta\bar\lambda-2\theta_1>-1$, we have
\[
\begin{aligned}
&\sum_{t=1}^{T}\eta_t^2
\exp\left(-2\sum_{j=t+1}^{T}\eta_j\lambda_j\right)\\
&\quad\le
2^{2\bar\eta\bar\lambda}\bar\eta^2
(T+t_0)^{-2\bar\eta\bar\lambda}
\sum_{t=1}^{T}(t+t_0)^{2\bar\eta\bar\lambda-2\theta_1}\\
&\quad\le
2^{2\bar\eta\bar\lambda}\bar\eta^2
\max\left\{1,
\frac{1}{1+2\bar\eta\bar\lambda-2\theta_1}\right\}
(T+t_0)^{1-2\theta_1}.
\end{aligned}
\]
The first inequality follows from $\sum_{j=t+1}^{T}\eta_j\lambda_j
\ge \bar\eta\bar\lambda
\log\frac{T+t_0+1}{t+t_0+1}$ and $t+t_0+1\le 2(t+t_0)$.
The second uses
\[
\sum_{t=1}^{T}(t+t_0)^p
\le \max\left\{1,\frac{1}{p+1}\right\}(T+t_0)^{p+1},
\qquad p>-1,
\]
with $p=2\bar\eta\bar\lambda-2\theta_1$.
Thus, the same bound for $\mathcal T_4$ holds with
\[
c_4=
6\sqrt c\,\mathrm{Tr}(C)\,
2^{2\bar\eta\bar\lambda}\bar\eta^2
\max\left\{1,
\frac{1}{1+2\bar\eta\bar\lambda-2\theta_1}\right\}.
\]

The proof is thus complete.
	\end{proof}

        The next proposition is used to bound $\T_4$ in the finite-horizon setting. We define $0^0:=0$ for convenience.

        \begin{proposition} \label{prop12}
        Let $v\geq0$, $\bar\eta=\eta_\star T^{-\theta_3}$, $0<\theta_3<1$, $\theta_4>0$, and $\eta_\star$, $\lambda_\star$ be constants independent of $T$. Then, there exists a constant $\delta_3=\delta_3(\lambda_\star,\eta_\star,\theta_3,\theta_4,v)$, independent of \/$T$, such that for any $T\geq2$,
            \begin{equation*}
                \begin{aligned}
                    \sum_{t=0}^{T-1}\frac{\exp\left\{-2\lambda_\star\eta_\star t T^{-\theta_4-\theta_3}\right\}}{1+\left(t\bar\eta\right)^{v}}
                    \leq\delta_3
                    \begin{cases}
                    T^{v\theta_3+(1-v)\min\{1,\theta_3+\theta_4\}}, & \text{ when } 0\leq v<1, \\
                    T^{\theta_3}\log T, & \text{ when } v=1, \\
                    T^{\theta_3}, & \text{ when } v>1.
                    \end{cases}
                \end{aligned}                
            \end{equation*}
        \end{proposition}

        \begin{proof}
            It suffices to prove the stated bounds for all sufficiently large $T$,
since the finitely many remaining integer values $T\geq2$ can be
covered by increasing $\delta_3$.
As $0<\theta_3<1$ and $\theta_4>0$, both
$\bar\eta(T-1)=\eta_\star T^{-\theta_3}(T-1)$ and
$\bar\eta T^{\theta_3+\theta_4}=\eta_\star T^{\theta_4}$
tend to infinity.
We therefore assume that both quantities are at least $1$
in the estimates below. 

We divide the proof into three cases: $v=0$, $v>0$, and $0<v<1$ with $\theta_3+\theta_4<1$. The third case is an improvement upon the analysis in the second case. \\
            \textbf{Case 1: }$v=0$\\ We first apply the inequality $1-\exp\{-x\}\geq\exp\{-2\lambda_\star\eta_\star\}x$ for $0\leq x\leq 2\lambda_\star\eta_\star$, yielding:
                \begin{equation*}
                    \begin{aligned}
                        \sum_{t=0}^{T-1}\frac{\exp\left\{-2\lambda_\star\eta_\star t T^{-\theta_4-\theta_3}\right\}}{1+\left(t\bar\eta\right)^{v}} &\leq
                    \frac{1-\exp\left\{-2\lambda_\star\eta_\star T^{1-\theta_4-\theta_3}\right\}}{1-\exp\left\{-2\lambda_\star\eta_\star T^{-\theta_4-\theta_3}\right\}}
                    \\ &\leq\frac{\exp\{2\lambda_\star\eta_\star\}}{2\lambda_\star\eta_\star}T^{\theta_3+\theta_4}\l(1-\exp\left\{-2\lambda_\star\eta_\star T^{1-\theta_4-\theta_3}\right\}\r).
                    \end{aligned}
                \end{equation*}
                When $\theta_3+\theta_4\leq1$, this term simplifies as 
                \[
                \sum_{t=0}^{T-1}\frac{\exp\left\{-2\lambda_\star\eta_\star t T^{-\theta_4-\theta_3}\right\}}{1+\left(t\bar\eta\right)^{v}} \leq\frac{\exp\{2\lambda_\star\eta_\star\}}{2\lambda_\star\eta_\star}T^{\theta_3+\theta_4}.
                \]
                When $\theta_3+\theta_4>1$, we use the inequality $1-\exp\{-x\}\leq x$ to obtain
                \[
                \sum_{t=0}^{T-1}\frac{\exp\left\{-2\lambda_\star\eta_\star t T^{-\theta_4-\theta_3}\right\}}{1+\left(t\bar\eta\right)^{v}} \leq \exp\{2\lambda_\star\eta_\star\}T.
                \]
                \textbf{Case 2: }$v>0$ \\
                We bound the summation as  
            \begin{equation} \label{temp75}
                \begin{aligned}
                    \sum_{t=0}^{T-1}\frac{\exp\left\{-2\lambda_\star\eta_\star t T^{-\theta_4-\theta_3}\right\}}{1+\left(t\bar\eta\right)^{v}}
                    \leq& 1+\int_{0}^{T-1}\frac{1}{1+\left(t\bar\eta\right)^{v}}dt
                    \leq 1+\frac{1}{\bar\eta}\int_{0}^{\bar\eta(T-1)}\frac{1}{1+t^{v}}dt
                    \\ \leq& 1+\frac{1}{\bar\eta}\left(1+\int_{1}^{\bar\eta(T-1)}t^{-v}dt\right)
                    \\ \leq& 1+\frac{1}{\bar\eta}+\frac{1}{\bar\eta}
                    \begin{cases}
                        \frac{\left(\bar\eta T\right)^{1-v}}{1-v}, & \text{ when } 0<v<1, \\
                        \log(\bar\eta T), & \text{ when } v=1, \\
                        \frac{1}{v-1}, & \text{ when } v>1,
                    \end{cases}
                    \\ \leq& \delta_3
                    \begin{cases}
                        T^{1-v+\theta_3v}, & \text{ when } 0<v<1, \\
                        T^{\theta_3}\log T, & \text{ when } v=1, \\
                        T^{\theta_3}, & \text{ when } v>1,
                    \end{cases}
                \end{aligned}
            \end{equation}
            where $\delta_3$ is a constant independent of $T$. \\
            \textbf{Case 3: }$0<v<1$ with $\theta_3+\theta_4<1$ \\
            In this case, a more refined estimation can be achieved compared to Case 2.
            We split the summation into three parts,
            \begin{equation} \label{temp76}
                \begin{aligned}
                    &\sum_{t=0}^{T-1}\frac{\exp\left\{-2\lambda_\star\eta_\star t T^{-\theta_4-\theta_3}\right\}}{1+\left(t\bar\eta\right)^{v}}
                    \\ \leq&1+\sum_{t=1}^{T^{\theta_3+\theta_4}}\frac{1}{1+\left(t\bar\eta\right)^{v}}
                    +\sum_{t=T^{\theta_3+\theta_4}}^{T}\frac{\exp\left\{-2\lambda_\star\eta_\star t T^{-\theta_4-\theta_3}\right\}}{1+\left(t\bar\eta\right)^{v}}
                    \\=:& 1+\A_1+\A_2.
                \end{aligned}
            \end{equation}
            We estimate $\A_1$ in the same manner as in \eqref{temp75}. Noting that $\bar\eta T^{\theta_3+\theta_4}=\eta_\star T^{\theta_4}$, we obtain 
            \begin{equation*}
                \begin{aligned}
                    \A_1&\leq \int_{0}^{T^{\theta_3+\theta_4}}\frac{1}{1+\left(t\bar\eta\right)^{v}}dt
                    \leq\frac{1}{\bar\eta}+\frac{1}{\bar\eta}\int_{1}^{\bar\eta T^{\theta_3+\theta_4}}
                    \frac{1}{t^{v}}dt
                    \\&\leq \frac{1}{\eta_\star} T^{\theta_3}\left(1+\frac{\left(\eta_\star T^{\theta_4}\right)^{1-v}}{1-v}\right)
                    \leq \frac{1}{\eta_\star}\left(1+\frac{\eta_\star^{1-v}}{1-v}\right)T^{\theta_4(1-v)+\theta_3}.
                \end{aligned}
            \end{equation*}
            Now, we estimate $\A_2$. Since $T^{\theta_3+\theta_4}-1\geq kT^{\theta_3+\theta_4}$ for $T\geq2$, where $k=1-2^{-\theta_4-\theta_3}$, it follows that 
            \begin{equation*}
                \begin{aligned}
                    \A_2&\leq
                    \int_{kT^{\theta_3+\theta_4}}^{T}\frac{1}{1+\left(t\bar\eta\right)^{v}}\exp\left\{-2\lambda_\star\eta_\star t T^{-\theta_4-\theta_3}\right\}dt.
                \end{aligned}
            \end{equation*}
            Letting $x=tT^{-\theta_4-\theta_3}$, we rewrite the above as
            \begin{equation*}
                \begin{aligned}
                    \A_2\leq\eta_\star^{-v}T^{(1-v)\theta_4+\theta_3}\int_{k}^{+\infty}x^{-v}\exp\{-2\lambda_\star\eta_\star x\}dx.
                \end{aligned}
            \end{equation*}
            Since the integral is finite and satisfies
            \[
            \int_{k}^{+\infty}x^{-v}\exp\{-2\lambda_\star\eta_\star x\}dx\leq k^{-v}
            \frac{\exp\{-2\lambda_\star\eta_\star k\}}{2\lambda_\star\eta_\star}<\infty,
            \]
            we combine the bounds for $\A_1$ and $\A_2$, and use \eqref{temp76}, to conclude that
            \begin{equation*}
                \sum_{t=0}^{T-1}\frac{\exp\left\{-2\lambda_\star\eta_\star t T^{-\theta_4-\theta_3}\right\}}{1+\left(t\bar\eta\right)^{v}}
                \leq\l(1+\frac{1}{\eta_\star}\left(1+\frac{\eta_\star^{1-v}}{1-v}\right)+\eta_\star^{-v}k^{-v}
            \frac{\exp\{-2\lambda_\star\eta_\star k\}}{2\lambda_\star\eta_\star}\r)T^{(1-v)\theta_4+\theta_3}.
            \end{equation*}                

            We then finish the proof.
        \end{proof}

        The following proposition concerns the finite-horizon setting. Before presenting the result, we highlight the following distinctions: in this setting, we have $t_0 = \theta_1 = \theta_2 = 0$, the step size is fixed as $\eta_t \equiv \bar\eta = \eta_\star T^{-\theta_3}$, the regularization parameter is set to be $\lambda_t \equiv \bar\lambda = \lambda_\star T^{-\theta_4}$, where $\eta_\star$ and $\lambda_\star$ are constants independent of $T$.
        By contrast, under the online setting in the unified form \eqref{setting}, $\bar\eta$ and $\bar\lambda$ are fixed constants, while $\eta_t$ and $\lambda_t$ depend directly on the current iteration $t$, rather than on a prescribed total number of iterations.

        \begin{proposition} \label{prop9}
            Suppose that Assumption \ref{a3} holds with $0<s\leq1$. Let $\alpha\in[0,\frac12]$, and set the parameters $t_0=\theta_1=\theta_2=0$, $\eta_t\equiv\bar\eta=\eta_\star T^{-\theta_3}$ with $0<\theta_3<1$ and $\lambda_t\equiv\bar\lambda=\lambda_\star T^{-\theta_4}$ with $\theta_4>0$.
            Additionally, assume that $T\geq2$, $\eta_\star(\kappa^2+\lambda_\star)\leq1$ and there exists a constant $\widetilde M$ independent of $T$, such that for all $t\in\mathbb{N}_T$,
            \begin{equation} \label{temp20}
                \be_{z^{t-1}}\be_{x_t}\left\|(H_t-H^\dagger)\phi(x_t)\right\|_{\Y}^2\left(=\be_{z^{t-1}}\left[
		\left\|\left(H_{t}-H^{\dagger}\right)C^{\frac12}\right\|_{\mathrm{HS}}^2\right]\right)\leq \widetilde M.
            \end{equation}
   Recall that
            \begin{align*}
			\T_4=6\sqrt{c}\sum_{t=1}^{T}\eta_t^{2}\left(\sqrt{c}\be_{z^{t-1}}\be_{x_t}\left\|(H_t-H^\dagger)\phi(x_t)\right\|_{\Y}^2+\sigma^2\right)\mathrm{Tr}\left(C^{1+2\alpha}\prod_{j=t+1}^{T}(I-\eta_j(C+\lambda_j I))^2\right).
		\end{align*}
            Then, the following bound holds for $\T_4$:
            \begin{equation*} 
                \T_4\leq
                \widetilde{c_4}\begin{cases}
                    T^{-(1-2\alpha+s)\theta_3+(s-2\alpha)\min\{1,\theta_3+\theta_4\}}, & \text{ when } 2\alpha<s\leq 1+2\alpha, \\
                    T^{-\theta_3}\log T, & \text{ when } 2\alpha=s, \\
                    T^{-\theta_3}, & \text{ when } 2\alpha>s,
                \end{cases}
            \end{equation*}
            where $\widetilde{c_4}=\widetilde{c_4}(\lambda_\star,\eta
            _\star,\theta_3,\theta_4,\alpha,s,\widetilde M,c,\sigma^2,
\kappa,\mathrm{Tr}(C^s))$ is a constant independent of \/$T$.
        \end{proposition}

        \begin{proof}
            Applying the assumed condition \eqref{temp20} and Assumption \ref{a3}, we get
            \begin{equation*} 
                \begin{aligned}
                    \T_4\leq&
                    6\sqrt{c}\left(\sqrt{c}\widetilde M+\sigma^2\right)\bar\eta^2\sum_{t=1}^{T}\mathrm{Tr}\left(C^{1+2\alpha}\prod_{j=t+1}^{T}(I-\eta_j(C+\lambda_j I))^2\right)
                    \\ \leq& 6\sqrt{c}\left(\sqrt{c}\widetilde M+\sigma^2\right)\mathrm{Tr}(C^s)\bar\eta^2
                    \sum_{t=1}^{T}\left\|C^{1+2\alpha-s}\prod_{j=t+1}^{T}(I-\eta_j(C+\lambda_j I))^2\right\|.
                \end{aligned}
            \end{equation*}
            If $1+2\alpha-s>0$, applying Lemma \ref{lemma1} (3) for $1\leq t\leq T-1$, we obtain the following estimate, which also holds for $t=T$,
            \[
            \left\|C^{1+2\alpha-s}\prod_{j=t+1}^{T}(I-\eta_j(C+\lambda_j I))^2\right\| \leq
            \exp\left\{-2(T-t)\bar\lambda\bar\eta\right\}
            \frac{2(\kappa^{2(1+2\alpha-s)}+(\frac{1+2\alpha-s}{2e})^{1+2\alpha-s})}{1+\left((T-t)\bar\eta\right)^{1+2\alpha-s}}.
            \]
            If $1+2\alpha-s=0$, we derive the bound
            \[
            \left\|C^{1+2\alpha-s}\prod_{j=t+1}^{T}(I-\eta_j(C+\lambda_j I))^2\right\| \leq \kappa^{2(1+2\alpha-s)}\exp\left\{-2(T-t)\bar\lambda\bar\eta\right\}
            \]
            for any $1\leq t\leq T$.
            Thus, based on the above estimates, we obtain the bound for $\T_4$:
            \begin{equation*}
                \begin{aligned}
                    \T_4\lesssim \bar\eta^2\sum_{t=0}^{T-1}\frac{\exp\left\{-2t\bar\lambda\bar\eta\right\}}{1+\left(t\bar\eta\right)^{1+2\alpha-s}},
                \end{aligned}
            \end{equation*}
            where we use the notation $\lesssim$ to omit constants independent of $T$ and $t$ for simplicity, indicating an inequality up to a multiplicative constant.

            Since $\bar\eta=\eta_\star T^{-\theta_3}$ and  $\bar{\lambda} = \lambda_\star T^{-\theta_4}$, and applying Proposition \ref{prop12} with $v=1+2\alpha-s$, we obtain
            \begin{equation*}
                \begin{aligned}
                    \T_4\leq\widetilde{c_4}
                    \begin{cases}
                    T^{-(1-2\alpha+s)\theta_3+(s-2\alpha)\min\{1,\theta_3+\theta_4\}}, & \text{ when } 2\alpha<s\leq 1+2\alpha, \\
                    T^{-\theta_3}\log T, & \text{ when } 2\alpha=s, \\
                    T^{-\theta_3}, & \text{ when } 2\alpha>s,
                \end{cases}
                \end{aligned}
            \end{equation*}
            where $\widetilde{c_4}:=12(\kappa^{2(1+2\alpha-s)}+(\frac{1+2\alpha-s}{2e})^{1+2\alpha-s})\sqrt{c}\left(\sqrt{c}\widetilde M+\sigma^2\right)\mathrm{Tr}(C^s)\eta_\star^2\delta_3$ is a constant independent of $T$.

            The proof is then complete.
\end{proof}

\subsection{Key Bounds for Estimating Prediction Error}
     In this subsection, we establish the key bounds for estimating the prediction error, specifically \eqref{temp-2} in Proposition \ref{prop5} and \eqref{temp20} in Proposition \ref{prop9}. The following proposition pertains to the online setting.
	\begin{proposition} \label{prop6}
		Under Assumption \ref{a2}, Assumption \ref{a3} and Assumption \ref{a4}, suppose that $\theta_1+\theta_2=1$, $t_0\geq1$, $\bar\eta\bar\lambda>
\max\left\{
\theta_2\min\{r,1\},
\theta_1-\frac12
\right\}$, $t_0\geq\exp\{\frac{1}{\theta_1}\}$, $(t_0+1)^{\theta_1}\geq\bar{\eta}(\kappa^2+\bar{\lambda})$, and  $c_4\sqrt{c}\,t_0^{-\theta_1}\log t_0<1$, where $c_4$ is the constant in Proposition \ref{prop5} with
$\alpha=1/2$. Then there exists a constant $M=M(\bar\eta,\bar\lambda,t_0,\theta_1,r,\sigma^2,\|S^\dagger\|_{\mathrm{HS}},s,c,\kappa)$ independent of $t$, such that
		\begin{equation} \label{temp15}
		    \be_{z^{t-1}}\left[
				\left\|\left(H_{t}-H^{\dagger}\right)C^{1/2}\right\|_{\mathrm{HS}}^2\right]\leq M, \quad \forall t \geq 1.
		\end{equation}
	\end{proposition}

	\begin{proof}
            Throughout this proof, we set $\alpha=1/2$. The proposition is proved by induction. We have already bounded $\T_1$, $\T_2$, $\T_3$, and $\T_4$ through four propositions, where $M$ in Proposition \ref{prop5} will share the same value during the induction process. An important fact we need to be aware of is that the bounds of $\T_2$, $\T_3$ and $\T_4$ require that $t\geq t_0+1$ when we bound $\be_{z^{t-1}}\left[
		\left\|\left(H_{t}-H^{\dagger}\right)C^{1/2}\right\|_{\mathrm{HS}}^2\right]$. 
        Hence, we first bound $\be_{z^{t-1}}\left[
		\left\|\left(H_{t}-H^{\dagger}\right)C^{1/2}\right\|_{\mathrm{HS}}^2\right]$ when $t\leq \lfloor t_0\rfloor+1$.
        When $t=1$,
            \begin{equation*}
                \be_{z^0}\left[\left\|(H_1-H^\dagger)C^{\frac12}\right\|_{\mathrm{HS}}^2\right]
                \leq\left\|H^\dagger C^{\frac12}\right\|_{\mathrm{HS}}^2
                \leq\kappa^2\left\|H^\dagger\right\|_{\mathrm{HS}}^2.
            \end{equation*}
        Note that $\T_1$, $\T_2$, and $\T_3$ are deterministic and can be regarded as functions of $t$.
        Define a function $f:\{1,2,\cdots\}\rightarrow \mathbb{R}$ iteratively, as $f(1):=\kappa^2\left\|H^\dagger\right\|_{\mathrm{HS}}^2$
        and,
                \begin{equation*}
                    \begin{aligned}
                    f(t+1):=&\T_1(t+1)+\T_2(t+1)+\T_3(t+1)
                    \\&+6\sqrt{c}\sum_{k=1}^{t}\eta_k^{2}\left(\sqrt{c}f(k)+\sigma^2\right)\mathrm{Tr}\left(C^{2}\prod_{j=k+1}^{t}(I-\eta_j(C+\lambda_j I))^2\right)
                    \end{aligned}
                \end{equation*}
                when $t>1$.                
        Then, by the error decomposition in Proposition \ref{Proposition error 1},
        \begin{equation} \label{eq M}
            \be_{z^{t-1}}\left[
		\left\|\left(H_{t}-H^{\dagger}\right)C^{1/2}\right\|_{\mathrm{HS}}^2\right]\leq f(t)
        \end{equation}
        for any $t\geq1$. Set $M_0:=
\max_{1\leq k\leq \lfloor t_0\rfloor+1} f(k)$.
		Choose 
		\begin{equation*}
			M=M_0+ \frac{c_1\bar{\lambda}^{\min\{2r+1,2\}}+c_2\bar{\eta}^{-2(r+\alpha)}+c_3
				+c_4\sigma^2t_0^{-\theta_1}\log t_0}{1-c_4\sqrt{c}t_0^{-\theta_1}\log t_0}.
		\end{equation*}
            Then, \eqref{temp15} holds for any $t\leq\lfloor t_0\rfloor+1$. Suppose \eqref{temp15} holds until some $t\geq\lfloor t_0\rfloor+1$. For $t+1$, corresponding to the prediction error $\be_{z^{t}}\left[
				\left\|\left(H_{t+1}-H^{\dagger}\right)C^{1/2}\right\|_{\mathrm{HS}}^2\right]$. Using Proposition \ref{Proposition error 1}, Proposition \ref{prop2}, Proposition \ref{prop3}, Proposition \ref{prop4} and Proposition \ref{prop5} with $\alpha=1/2$, we obtain that
        \begin{equation*}
\begin{aligned}
&\be_{z^{t}}\!\bigl[
\bigl\|(H_{t+1}-H^\dagger)C^{1/2}\bigr\|_{\mathrm{HS}}^2
\bigr]
\\
\le{}&
c_1\bigl(\bar{\lambda}(t+t_0)^{-\theta_2}\bigr)^{\min\{2r+1,2\}}
+c_2\bar{\eta}^{-2(r+\alpha)}(t_0+1)^{2\bar{\eta}\bar{\lambda}}
(t+t_0)^{-2(r+\alpha)(1-\theta_1)-2\bar{\eta}\bar{\lambda}}
\\
&
+c_3(t+t_0)^{-2\theta_2\min\{r,1\}+\theta_1-1}
+c_4(\sqrt{c}M+\sigma^2)(t+t_0)^{-\theta_1}\log(t+t_0)
\\
\le{}&
c_1\bar{\lambda}^{\min\{2r+1,2\}}
+c_2\bar{\eta}^{-2(r+\alpha)}
+c_3
+c_4(\sqrt{c}M+\sigma^2)(t+t_0)^{-\theta_1}\log(t+t_0)
\\
\le{}&
c_1\bar{\lambda}^{\min\{2r+1,2\}}
+c_2\bar{\eta}^{-2(r+\alpha)}
+c_3
+c_4(\sqrt{c}M+\sigma^2)t_0^{-\theta_1}\log t_0.
\end{aligned}
\end{equation*}
            where the last inequality holds when $t_0\geq\exp\{\frac{1}{\theta_1}\}$.
            Since $c_4$ is independent of $t_0$, for sufficiently large $t_0$, we have
            \[
                c_4\sqrt{c}t_0^{-\theta_1}\log t_0<1.
            \]
		By \eqref{eq M}, it follows that
		\[
			\be_{z^{t}}\left[
				\left\|\left(H_{t+1}-H^{\dagger}\right)C^{1/2}\right\|_{\mathrm{HS}}^2\right]\leq M,
		\]
        which advances the induction.
        
	    The proof is then complete.
	\end{proof}

        Next, we establish a similar bound for the finite-horizon setting.
        \begin{proposition} \label{prop8}
            Under Assumption \ref{a2} and Assumption \ref{a4}, if $t_0=\theta_1=\theta_2=0$, suppose $\eta_\star(\kappa^2+\lambda_\star)\leq1$ and
            \[\eta_\star<\frac{1}{6c\kappa^2\left(1+\frac{1}{2e\theta_3}\right)}.\] 
            Then, for any $T\geq2$, there exists a constant $\widetilde{M}=\widetilde{M}(\lambda_\star,\eta_\star,\theta_3,\|S^\dagger\|_{\mathrm{HS}},r,\sigma^2,c,\kappa)$ independent of $T$, such that
            \begin{equation} \label{temp22}
                \be_{z^{t-1}}\left[
				\left\|\left(H_{t}-H^{\dagger}\right)C^{1/2}\right\|_{\mathrm{HS}}^2\right]\leq \widetilde M,
            \end{equation}
            for any $t\in\bn_T$.
        \end{proposition}

        \begin{proof}
            We prove this proposition by induction.
            Set
            \[
            \widetilde M= \kappa^2\|H^\dagger\|_{\mathrm{HS}}^2 + \frac{c_1\lambda_\star^{\min\{2r+1,2\}} + 6\kappa^2\|H^\dagger\|_{\mathrm{HS}}^2+6\sqrt{c}\sigma^2\kappa^2\left(1+\frac{1}{2e\theta_3}\right)\eta_\star}{1-6c\kappa^2\left(1+\frac{1}{2e\theta_3}\right)\eta_\star}.
            \]
            For $t=1$, it is clear that 
            \[
            \be_{z^0}\left[\left\|(H_1-H^\dagger)C^{\frac12}\right\|_{\mathrm{HS}}^2\right]\leq\kappa^2\|H^\dagger\|_{\mathrm{HS}}^2\leq\widetilde M.
            \]
            Assume that \eqref{temp22} holds from $1$ to $t$. We now prove that it also holds for $t+1$.      
            Since
$\bar\eta(\kappa^2+\bar\lambda)
\le\eta_\star(\kappa^2+\lambda_\star)\le1$
and
$H_{\bar\lambda}=H^\dagger C(C+\bar\lambda I)^{-1}$,
we have, for every $t\le T$,
\[
\T_2
=6\left\|H_{\bar\lambda}C^{1/2}
(I-\bar\eta(C+\bar\lambda I))^t\right\|_{\mathrm{HS}}^2
\le6\kappa^2\|H^\dagger\|_{\mathrm{HS}}^2.
\]  
            Using Proposition \ref{Proposition error 1} and Proposition \ref{prop2}, 
            we have
            \begin{equation} \label{temp23}
                \begin{aligned}
                    \be_{z^{t}}\left[
				\left\|\left(H_{t+1}-H^{\dagger}\right)C^{1/2}\right\|_{\mathrm{HS}}^2\right]
                \leq c_1\lambda_\star^{\min\{2r+1,2\}} + 6\kappa^2\|H^\dagger\|_{\mathrm{HS}}^2+\T_4.
                \end{aligned}
            \end{equation}
            Note that Proposition \ref{prop9} cannot be used in the induction process, because the current step size $\eta_t$ relies on the total number of iterations $T$.
            Therefore, we re-estimate $\T_4$. By the definition of $\T_4$ in Proposition \ref{Proposition error 1} and the induction hypothesis,
            \begin{equation*}
                \begin{aligned}
                    \T_4&\leq6\sqrt{c}(\sqrt{c}\widetilde M+\sigma^2)\bar\eta^2\sum_{i=1}^t\mathrm{Tr}\left(C^2\l(I-\bar\eta(C+\bar\lambda I)\r)^{2(t-i)}\right)
                    \\ &\leq6\sqrt{c}(\sqrt{c}\widetilde M+\sigma^2)\mathrm{Tr}(C)\bar\eta^2\left(\kappa^2+\sum_{i=1}^{t-1}\left\|C\l(I-\bar\eta(C+\bar\lambda I)\r)^{2i}\right\|\right)
                    \\ &\leq6\sqrt{c}(\sqrt{c}\widetilde M+\sigma^2)\mathrm{Tr}(C)\bar\eta^2\left(\kappa^2+\frac{1}{2e}\sum_{i=1}^{t-1} \left(\bar\eta i\right)^{-1}\exp\{-2\bar\eta\bar\lambda i\}\right),
                \end{aligned}
            \end{equation*}
            where we have used Lemma \ref{lemma1} (2) with $\beta=1$.
            Using $\eta_\star\kappa^2\leq1$, $\bar\eta=\eta_\star T^{-\theta_3}$ and $\sum_{i=1}^{t-1}i^{-1}\leq 1+\log T$, it follows that
            \begin{equation*}
                \begin{aligned}
                    \T_4&\leq6\sqrt{c}(\sqrt{c}\widetilde M+\sigma^2)\mathrm{Tr}(C)\eta_\star\left(1+\frac{1}{2e}(1+\log T)\right)T^{-\theta_3}
                    \\ &\leq6\sqrt{c}(\sqrt{c}\widetilde M+\sigma^2)\mathrm{Tr}(C)\left(1+\frac{1}{2e\theta_3}\right)
                    \eta_\star,
                \end{aligned}
            \end{equation*}
            where we have used the fact that $\sup_{x>0}(1+\log x)x^{-\theta_3}=\frac{1}{\theta_3}\exp\{\theta_3-1\}$. Substituting this into \eqref{temp23} yields that
            \begin{equation*}
                \begin{aligned}
                    \be_{z^{t}}\left[
				\left\|\left(H_{t+1}-H^{\dagger}\right)C^{1/2}\right\|_{\mathrm{HS}}^2\right]&\leq c_1\lambda_\star^{\min\{2r+1,2\}} + 6\kappa^2\|H^\dagger\|_{\mathrm{HS}}^2+
                6\sqrt{c}(\sqrt{c}\widetilde M+\sigma^2)\kappa^2\left(1+\frac{1}{2e\theta_3}\right)
                    \eta_\star
                \\&\leq \widetilde M,
                \end{aligned}
            \end{equation*}
            which advances the induction.
            
            The proof is thus complete.
        \end{proof}

	\section{Convergence Analysis in Expectation} \label{expected}
    In this section, we prove the error bounds in expectation provided by Subsection \ref{sub2}.
    
	\begin{proof}[Proof of Theorem \ref{thm1}]
		Let $\theta_1+\theta_2=1$ and $\alpha=\frac12$. If $T\geq t_0+1$,
		from Proposition \ref{Proposition error 1}, Proposition \ref{prop2}, Proposition \ref{prop3}, Proposition \ref{prop4}, Proposition \ref{prop5} and Proposition \ref{prop6} with $\alpha=\frac12$ and $0<s\leq1$, there holds
		\begin{equation*}
			\begin{aligned}
				\be_{z^{T}}[\mathcal{E}(H_{T+1})-\mathcal{E}(H^\dagger)]
				\leq& c_1\left(\bar{\lambda}(T+t_0)^{-\theta_2}\right)^{\min\{2r+1,2\}}
				\\&+c_2\bar{\eta}^{-(2r+1)}(t_0+1)^{2\bar{\eta}\bar{\lambda}}(T+t_0)^{-(2r+1)(1-\theta_1)-2\bar{\eta}\bar{\lambda}}
				\\&+c_3(T+t_0)^{-2\theta_2\min\{r,1\}+\theta_1-1}
				\\&+c_4\left(\sqrt{c}M+\sigma^2\right)
				\begin{cases}
					(T+t_0)^{-\theta_1}, &  \text{ when } s<1, \\
					(T+t_0)^{-\theta_1}\log(T+t_0), & \text{ when } s=1.
				\end{cases}
			\end{aligned}
		\end{equation*}
	We choose $\theta_1=\frac{2\min\{r+1/2,1\}}{1+2\min\{r+1/2,1\}}$ and $\theta_2=\frac{1}{1+2\min\{r+1/2,1\}}$, then 
	\begin{equation} \label{temp59}
		\be_{z^{T}}[\mathcal{E}(H_{T+1})-\mathcal{E}(H^\dagger)]
            \leq c_{1,1}
		\begin{cases}
			(T+t_0)^{-\frac{2\min\{r+1/2,1\}}{1+2\min\{r+1/2,1\}}}\log(T+t_0), & \text{ when } s=1, \\
			(T+t_0)^{-\frac{2\min\{r+1/2,1\}}{1+2\min\{r+1/2,1\}}}, & \text{ when } s<1,
		\end{cases}
	\end{equation}
	for any $T\geq t_0+1$, where the constant $c_{1,1}=c_1\bar\lambda^{\min\{2r+1,2\}}+c_2\bar{\eta}^{-(2r+1)}(t_0+1)^{2\bar{\eta}\bar{\lambda}}+c_3+c_4\left(\sqrt{c}M+\sigma^2\right)$ is independent of $T$. Let $c_{1,1}$ be sufficiently large such that \eqref{temp59} holds true for $1\leq T< t_0+1$.
	
	We then finish the proof.
	\end{proof}

	\begin{proof}[Proof of Theorem \ref{thm2}]
            Let $\theta_1+\theta_2=1$ and $\alpha=0$. If $T\geq t_0+1$,
		from Proposition \ref{Proposition error 1}, Proposition \ref{prop2}, Proposition \ref{prop3}, Proposition \ref{prop7}, Proposition \ref{prop5}, and Proposition \ref{prop6} with $0<s\leq1$, there holds
		\begin{equation*}
			\begin{aligned}
				\be_{z^{T}}[\left\|H_{T+1}-H^\dagger\right\|_{\mathrm{HS}}^2]
				\leq& c_1\left(\bar{\lambda}(T+t_0)^{-\theta_2}\right)^{2\min\{r,1\}}
				\\&+c_2\bar{\eta}^{-2r}(t_0+1)^{2\bar{\eta}\bar{\lambda}}(T+t_0)^{-2r(1-\theta_1)-2\bar{\eta}\bar{\lambda}}
				\\&+c_3(T+t_0)^{-2\theta_2\min\{r,1\}}
				\\&+c_4\left(\sqrt{c}M+\sigma^2\right)(T+t_0)^{-(1+s)\theta_1+s}.
			\end{aligned}
		\end{equation*}
		We choose $\theta_1=\frac{s+2\min\{r,1\}}{1+s+2\min\{r,1\}}$ and $\theta_2=\frac{1}{1+s+2\min\{r,1\}}$, then 
		\begin{equation} \label{temp60}
			\be_{z^{T}}\Big[\left\|H_{T+1}-H^\dagger\right\|_{\mathrm{HS}}^2\Big] \leq
			c_{1,2}(T+t_0)^{-\frac{2\min\{r,1\}}{1+s+2\min\{r,1\}}},
		\end{equation}
            for any $T\geq t_0+1$, where $c_{1,2}=c_1\bar\lambda^{2\min\{r,1\}}+c_2\bar{\eta}^{-2r}(t_0+1)^{2\bar{\eta}\bar{\lambda}}+c_3+c_4\left(\sqrt{c}M+\sigma^2\right)$ is a constant independent of $T$. 
            Let $c_{1,2}$ be sufficiently large such that \eqref{temp60} holds true for $1\leq T< t_0+1$.
	
	The proof is complete.
	\end{proof}

        \begin{proof}[Proof of Theorem \ref{thm3}] 
            If the conditions $\eta_\star(\kappa^2+\lambda_\star)\leq1$ and
            \[\eta_\star<\frac{1}{6c\kappa^2\left(1+\frac{1}{2e{\theta_3}}\right)}
            \] hold, then, by Proposition \ref{Proposition error 1}, Proposition \ref{prop2}, Proposition \ref{prop3}, Proposition \ref{prop9}, and  Proposition \ref{prop8} with $\alpha=1/2$ and $0<s\leq1$, we obtain that
            \begin{equation*}
                \begin{aligned}
                    \be_{z^{T}}[\mathcal{E}(H_{T+1})-\mathcal{E}(H^\dagger)]
				\leq& c_1\left(\lambda_\star T^{-\theta_4}\right)^{\min\{2r+1,2\}}
				\\&+c_2\eta_\star^{-(2r+1)}T^{-(2r+1)(1-\theta_3)}\exp\{-\tau \eta_\star\lambda_\star T^{1-\theta_4-\theta_3}\}
                \\&+\widetilde{c_4}\begin{cases}
                    T^{-\theta_3}\log T, & \text{ when } s=1, \\
                    T^{-\theta_3}, & \text{ when } s<1.
                \end{cases}
                \end{aligned}
            \end{equation*}
            We choose $\theta_3=\frac{2r+1}{2r+2}$ and $\theta_4\geq\frac{2r+1}{(2r+2)\min\{2r+1,2\}}$, then
            \begin{equation*}
                \be_{z^{T}}[\mathcal{E}(H_{T+1})-\mathcal{E}(H^\dagger)]\leq
                c_{1,3}\begin{cases}
                    T^{-\frac{2r+1}{2r+2}} , & \text{ when } s<1, \\
                    T^{-\frac{2r+1}{2r+2}}\log T , & \text{ when } s=1,
                \end{cases}
            \end{equation*}
            where $c_{1,3}=c_1\lambda_\star^{\min\{2r+1,2\}}+c_2\eta_\star^{-(2r+1)}+\widetilde{c_4}$ is a constant independent of $T$.

            The proof is complete.
        \end{proof}

        \begin{proof}[Proof of Theorem \ref{thm4}]
            If the conditions $\eta_\star(\kappa^2+\lambda_\star)\leq1$ and
            \[\eta_\star<\frac{1}{6c\kappa^2\left(1+\frac{1}{2e{\theta_3}}\right)}
            \] hold, using Proposition \ref{Proposition error 1}, Proposition \ref{prop2}, Proposition \ref{prop3}, Proposition \ref{prop9}, and  Proposition \ref{prop8} with $\alpha=0$ and $0<s\leq1$, there holds
            \begin{equation} \label{temp61}
                \begin{aligned}
                    \be_{z^{T}}\Big[\left\|H_{T+1}-H^\dagger\right\|_{\mathrm{HS}}^2\Big]
				\leq& c_1\left(\lambda_\star T^{-\theta_4}\right)^{\min\{2r,2\}}
				\\&+c_2\eta_\star^{-2r}T^{-2r(1-\theta_3)}\exp\{-\tau \eta_\star\lambda_\star T^{1-\theta_4-\theta_3}\}
                \\&+\widetilde{c_4}
                    T^{-(1+s)\theta_3+s\min\{1,\theta_3+\theta_4\}}.
                \end{aligned}
            \end{equation}
            Choosing $\theta_3=\frac{2r+s}{1+2r+s}$ and $\theta_4\geq\frac{2r}{(1+2r+s)\min\{2r,2\}}$, then
            \[
            \be_{z^{T}}\Big[\left\|H_{T+1}-H^\dagger\right\|_{\mathrm{HS}}^2\Big]
		\leq c_{1,4}T^{-\frac{2r}{1+2r+s}},
            \]
            where $c_{1,4}=c_1\lambda_\star^{\min\{2r,2\}}+c_2\eta_\star^{-2r}+\widetilde{c_4}$ is a constant independent of $T$.

            The proof is complete.
        \end{proof}

	\begin{remark}
		If we choose $\theta_3+\theta_4<1$ in the above two proofs under the constant step size, then $\exp\{-\tau T^{1-\theta_4-\theta_3}\}=o(T^{-k})$ for any $k>0$. Although the second term on the right-hand side of \eqref{temp61} decays faster than any polynomial, the overall learning rate would be slower than that achieved by our results.
	\end{remark}
	
	\section{Convergence Analysis in High Probability }
    \label{prob}
    In this section, we derive the high-probability error bounds presented in Subsection \ref{sub3}. Our proofs are mainly based on the following proposition. This proposition is from \citet[Proposition A.3]{tarres2014online}, and is an extension of \citet[Theorem 3.4]{pinelis1994optimum}.

    \begin{proposition} \label{prop10}
        Let $(\mathcal F_i)_{i\ge 0}$ be a filtration, and denote by
\[
\be_{i-1}[\cdot]:=\mathbb E[\cdot\mid\mathcal F_{i-1}]
\]
the conditional expectation with respect to $\mathcal F_{i-1}$. 

Let $(\xi_i)_{i\ge1}$ be a martingale difference sequence in a Hilbert space with respect to $(\mathcal F_i)$, i.e., each $\xi_i$ is $\mathcal F_i$-measurable and satisfies $\be_{i-1}[\xi_i]=0$.
Suppose that $\|\xi_i\|\leq M_\xi$ and $\sum_{i=1}^t\be_{i-1}\|\xi_i\|^2\leq\tau^2$ almost surely for some constant $M_\xi>0$ and $\tau>0$. Then, for any $\delta\in(0,2/e)$, with probability at least $1-\delta$, the following inequality holds:
        \begin{equation*}
            \sup_{1\leq k\leq t}\l\|\sum_{i=1}^k\xi_i\r\|\leq2\l(\frac {M_\xi}3+\tau\r)\log\frac2\delta.
        \end{equation*}
        Additionally,
        \[
        \sup_{1\leq k\leq t}\l\|\sum_{i=1}^k\xi_i\r\|^2\leq8\l(M_\xi^2+\tau^2\r)\log^2\frac2\delta.
        \]
    \end{proposition}

    By Proposition \ref{proposition error2}, since $\T_1$, $\T_2$ and $\T_3$ have already been bounded in Section \ref{section:basic}, our goal is to bound the remaining term $6\left\|\sum_{t=1}^T\chi_t\right\|^2_{\mathrm{HS}}$ with high probability. According to Proposition \ref{prop10}, this requires the uniform bound on $\l\|\chi_t\r\|_\HS$ for $1\leq t\leq T$. Using \eqref{temp63}, it is sufficient to bound $\|H_t\|_\HS$. However, $\l\|H_t\r\|_\HS$ may grow rapidly with increasing $t$. Therefore, we first establish a high-probability bound on $\|H_t\|_\HS$, which motivates the decomposition of $H_t - H^\dagger$ into $L_t + R_t$ as follows.

    Let us denote $\phi(x_t)\otimes\phi(x_t)$ by $C_t$.
    We define two random processes $(L_t)_{t\geq1}$ and $(R_t)_{t\geq1}$ recursively by
    \[
    L_1=-H^\dagger,\quad R_1=\mathbf{0},
    \]
    and for any $t\geq1$, 
    \begin{equation} \label{temp66}
        \begin{split}
            L_{t+1}&:=L_t\left(I-\eta_t(C+\lambda_t I)\right)-\eta_t\lambda_tH^\dagger, \\
            R_{t+1}&:=R_t\left(I-\eta_t(C_t+\lambda_t I)\right)+\eta_t(y_t-H^\dagger\phi(x_t))\otimes\phi(x_t)+\eta_tL_t(C-C_t).
        \end{split}
    \end{equation}
     Note that for any $t \geq 1$, $L_t$ is deterministic, while $R_t$ depends on $z^{t-1}$ and is independent of $z_t$.  Moreover, by induction, one can verify that
     \begin{equation}\label{LtRt}
         L_t + R_t = H_t - H^\dagger
     \end{equation}
     for all $t \geq 1$.

We then provide a bound on $L_t$ in Lemma \ref{lemma6}.
     
\begin{lemma} \label{lemma6}
        Suppose that $(t_0+1)^{\theta_1}\geq\bar{\eta}(\kappa^2+\bar{\lambda})$ holds. Then, for any $t\geq1$,
        \[
        \|L_t\|_{\mathrm{HS}}\leq\|H^\dagger\|_{\mathrm{HS}}.
        \]
    \end{lemma}

    \begin{proof}
         We prove it by induction. For $t=1$, $\|L_1\|_{\mathrm{HS}}=\|H^\dagger\|_{\mathrm{HS}}$.
            Since $(t_0+1)^{\theta_1}\geq\bar{\eta}(\kappa^2+\bar{\lambda})$, there holds
            $1-\eta_t(\kappa^2+\lambda_t)\geq0$ for any $t\geq1$. Thus,           
            \begin{equation*}
                \|L_{t+1}\|_{\mathrm{HS}}\leq\|L_t\|_{\mathrm{HS}}(1-\eta_t\lambda_t)+\eta_t\lambda_t\|H^\dagger\|_{\mathrm{HS}},
            \end{equation*}
            which implies that $\|L_t\|_{\mathrm{HS}}\leq \|H^\dagger\|_{\mathrm{HS}}$ for any $t\geq1$. We then finish the proof.
    \end{proof}

    By the following proposition, we can, with high probability, control the increasing rate of $R_t$ in the online setting.

    \begin{proposition} \label{corollary}
        Under Assumption \ref{a5}, suppose that $\theta_1+\theta_2=1$,  $(t_0+1)^{\theta_1}\geq\bar{\eta}(\kappa^2+\bar{\lambda})$, and $\bar\eta\bar\lambda\geq\theta_1$. Then, with  probability at least $1-\delta$, there holds
        \begin{equation*}
            \begin{aligned}
                \|R_t\|_\HS\leq d_2(t+t_0)^{\frac12-\theta_1}\log(t+t_0)\log\frac2\delta, \quad 1\leq t\leq T,
            \end{aligned}
        \end{equation*}
        where $d_2=d_2(\bar\eta,\bar\lambda,\theta_1,t_0,M_\rho,\|S^\dagger\|_{\mathrm{HS}},r,\kappa)$ is a constant independent of $t$, $T$, and $\delta$.
    \end{proposition}

    \begin{proof}
        Denote $(y_t-H^\dagger\phi(x_t))\otimes\phi(x_t)+L_t(C-C_t)$ by $K_t$. Then, by applying induction to \eqref{temp66}, $R_{t+1}$ can be expressed as
        \begin{equation} \label{temp67}
            R_{t+1}=\sum_{i=1}^t\eta_iK_i\prod_{j=i+1}^t\l(I-\eta_j(C_j+\lambda_j I)\r).
        \end{equation}
        Since $\be_{z_i}[K_i]=0$, for each fixed $t$, consider the reversed sequence
\[
\left(\eta_iK_i\prod_{j=i+1}^t\l(I-\eta_j(C_j+\lambda_j I)\r)\right)_{i=t,t-1,\ldots,1}.
\]
Define the filtration $(\mathcal F_{t,\ell})_{\ell=0}^t$ by
\[
\mathcal F_{t,0}:=\{\emptyset,\emptyset^c\},\qquad
\mathcal F_{t,\ell}:=\sigma(z_j:\, j=t-\ell+1,\ldots,t),\quad \ell=1,\ldots,t.
\]
Then the above reversed sequence forms a martingale difference sequence with respect to $(\mathcal F_{t,\ell})_{\ell=0}^t$. Indeed, for each $1\leq i\leq t$, the product $\prod_{j=i+1}^t\l(I-\eta_j(C_j+\lambda_j I)\r)$ is $\sigma(z_{i+1},\ldots,z_t)$-measurable, while $K_i$ depends only on $z_i$. Hence, by the independence of the samples and the fact that $\be_{z_i}[K_i]=0$,
\[
\be\!\left[\eta_iK_i\prod_{j=i+1}^t\l(I-\eta_j(C_j+\lambda_j I)\r)\,\middle|\,\sigma(z_{i+1},\ldots,z_t)\right]=0.
\]
Moreover, the conditional quadratic variation required in Proposition \ref{prop10} is exactly controlled by
\[
\sum_{i=1}^{t}\be_{z_i}\l[\l\|\eta_iK_i\prod_{j=i+1}^t\l(I-\eta_j(C_j+\lambda_j I)\r)\r\|_\HS^2\r],
\]
which will be estimated below. We apply Proposition \ref{prop10} to bound $R_{t+1}$ for each fixed $t$.


        Using Lemma \ref{lemma6} and Assumption \ref{a5}, we have
        \begin{equation} \label{temp68}
            \|K_t\|_\HS\leq\kappa M_\rho+3\kappa^2\|H^\dagger\|_\HS.
        \end{equation}
        By Lemma \ref{lemma2},
        \begin{equation} \label{temp69}
            \begin{aligned}
                \l\|\prod_{j=i+1}^t\l(I-\eta_j(C_j+\lambda_j I)\r)\r\| &\leq\exp\l\{-\sum_{j=i+1}^t\eta_j\lambda_j\r\} \\
                &\leq \l(\frac{t+t_0+1}{i+t_0+1}\r)^{-\bar\eta\bar\lambda}
                \leq 2^{\bar\eta\bar\lambda}\l(\frac{t+t_0}{i+t_0}\r)^{-\bar\eta\bar\lambda}.
            \end{aligned}
        \end{equation}
        Combining \eqref{temp67}, \eqref{temp68} and \eqref{temp69}, and using the assumption that $\bar\eta\bar\lambda\geq\theta_1$, we obtain
        \begin{equation*} 
            \begin{aligned}
                \l\|\eta_iK_i\prod_{j=i+1}^t\l(I-\eta_j(C_j+\lambda_j I)\r)\r\|_\HS
                \leq\l(\kappa M_\rho+3\kappa^2\|H^\dagger\|_\HS\r)2^{\bar\eta\bar\lambda}\bar\eta(t+t_0)^{-\theta_1}, \quad 1\leq i\leq t.
            \end{aligned}
        \end{equation*} Moreover, 
        \begin{equation*}
            \begin{aligned}
                &\sum_{i=1}^{t}\be_{z_{i}}\l[\l\|\eta_iK_i\prod_{j=i+1}^t\l(I-\eta_j(C_j+\lambda_j I)\r)\r\|_\HS^2\r]
                \\ \leq&\l(\kappa M_\rho+3\kappa^2\|H^\dagger\|_\HS\r)^22^{2\bar\eta\bar\lambda}\bar\eta^2\sum_{i=1}^t(i+t_0)^{-2\theta_1}\l(\frac{t+t_0}{i+t_0}\r)^{-2\bar\eta\bar\lambda}
                \\ \leq&\l(\kappa M_\rho+3\kappa^2\|H^\dagger\|_\HS\r)^22^{4\bar\eta\bar\lambda}\frac{\bar\eta^2}{2\bar\eta\bar\lambda-2\theta_1+1}(t+t_0+1)^{1-2\theta_1}.
            \end{aligned}
        \end{equation*}
        Then, by Proposition \ref{prop10}, with probability at least $1-\delta_{t+1}$,
        \begin{equation*}
            \|R_{t+1}\|_{\HS}\leq d_1(t+t_0+1)^{\frac12-\theta_1}\log\frac{2}{\delta_{t+1}},
        \end{equation*}
        for some constant $d_1$ that is independent of $t$, $\delta_{t+1}$ and $T$. 
        
        Now, for any $\delta\in(0,2/e)$,  choose $\delta_t=\delta(t+t_0)^{-2}t_0$ for any $1\leq t\leq T$. Then
        $\sum_{t=1}^T\delta_t\leq\delta$ and
        \begin{equation*}
            \begin{aligned}
                \|R_t\|_\HS\leq d_2(t+t_0)^{\frac12-\theta_1}\log(t+t_0)\log\frac2\delta, \quad 1\leq t\leq T,
            \end{aligned}
        \end{equation*} where $d_2$ is a constant independent of $t$, $T$, and $\delta$.

        The proof is complete.
    \end{proof}

    Recall that $\chi_t=\eta_{t}\B_{t}C^\alpha\prod_{j=t+1}^{T}(I-\eta_j(C+\lambda_j I))$ in Proposition \ref{proposition error2}. Define 
    \begin{equation} \label{chi_t}
        \widetilde{\chi}_t:=\chi_t\mathbbm{1}_{A_t},
    \end{equation}
    where 
    \begin{equation} \label{At}
        A_t:=\l\{\|R_t\|_\HS\leq d_2(t+t_0)^{\frac12-\theta_1}\log(t+t_0)\log\frac2\delta\r\}.
    \end{equation}
    Then, $A_t$ is independent of $z_t$, and $\widetilde{\chi}_t$ depends on $z^t=\{z_1,z_2,\cdots,z_t\}$. Moreover, for any $t \geq 1$, we have $\be_{z_t}\l[\widetilde{\chi}_t\r]=\mathbbm{1}_{A_t}\be_{z_t}\l[\chi_t\r]=0$. By Proposition \ref{corollary},
    \[
    \mathbb{P}\l(\widetilde{\chi}_t=\chi_t \text{ for any } 1\leq t\leq T\r)\geq 1-\delta.
    \]

    In the next proposition,  we provide bounds for $\sup_{1\leq t\leq T}\l\|\widetilde{\chi}_t\r\|^2_\HS$ and $\sum_{t=1}^T\be_{z_t}\l\|\widetilde{\chi}_t\r\|_\HS^2$ in preparation for applying Proposition \ref{prop10}.

    \begin{proposition} \label{prop11}
        Under Assumptions \ref{a3} and \ref{a5}, suppose that $\theta_1+\theta_2=1$, $\alpha\in[0,\frac12]$, $(t_0+1)^{\theta_1}\geq\bar\eta(\kappa^2+\bar\lambda)$,  $\bar\eta\bar\lambda\geq\theta_1$ and $\bar\eta\bar\lambda\geq2\theta_1-\frac12$. 
        Then,
        \begin{itemize}
            \item[(1)] 
            the Hilbert-Schmidt norm of $\widetilde{\chi}_t$ is uniformly bounded as follows:
            \begin{equation*}
                \sup_{1\leq t\leq T}\l\|\widetilde{\chi}_t\r\|^2_\HS\leq
                M_1^2,
            \end{equation*}   
            where 
            \[M_1^2:=2d_3^2(T+t_0)^{-2\theta_1}+2d_3^2(T+t_0)^{1-4\theta_1}\log^2(T+t_0)\log^2\frac2\delta;\]
            \item[(2)] the total squared Hilbert-Schmidt norm in expectation is bounded by
            \[\sum_{t=1}^T\be_{z_t}\l\|\widetilde{\chi}_t\r\|_\HS^2\leq \tau_1^2,\]
            where $\tau_1$ is defined as
            \begin{equation*}
\tau_1^2:=d_5
\begin{cases}
(T+t_0)^{-\theta_1}
+(T+t_0)^{1-3\theta_1}\log^2(T+t_0)\log^2\frac{2}{\delta},
\\
\qquad\quad \text{if } 1+2\alpha-s>1,\\[0.6ex]
(T+t_0)^{-\theta_1}\log(T+t_0)
+(T+t_0)^{1-3\theta_1}\log^3(T+t_0)\log^2\frac{2}{\delta},
\\\qquad\quad \text{if } 1+2\alpha-s=1,\\[0.6ex]
(T+t_0)^{s-2\alpha-\theta_1(1+s-2\alpha)}
+(T+t_0)^{1+s-2\alpha-\theta_1(3+s-2\alpha)}
\log^2(T+t_0)\log^2\frac{2}{\delta},
\\
\qquad\quad \text{if } 0\le 1+2\alpha-s<1.
\end{cases}
\end{equation*}
        \end{itemize}
        Here,
$d_3=d_3(\bar\eta,\bar\lambda,\theta_1,t_0,M_\rho,
\|S^\dagger\|_{\mathrm{HS}},r,\alpha,\kappa)$
and
$d_5=d_5(\bar\eta,\bar\lambda,\theta_1,t_0,M_\rho,
\|S^\dagger\|_{\mathrm{HS}},r,\alpha,s,
\kappa,\mathrm{Tr}(C^s))$
are constants independent of $T$ and $\delta$.
    \end{proposition}

    \begin{proof}
        \begin{itemize} 
            \item[(1)] Since $H_t$ depends only on $z^{t-1}$, the centered
representation of $\mathcal B_t$ in the proof of Proposition~\ref{proposition error2}
and Assumption~\ref{a5} give
\[
 \|\B_t\|_{\mathrm{HS}}
 \leq 2\kappa\bigl(M_\rho+\kappa\|H_t\|_{\mathrm{HS}}\bigr)
 \quad\text{almost surely}.
\]
Using this bound, \eqref{LtRt}, Lemma~\ref{lemma6}, and \eqref{At}, we obtain
            \begin{equation} \label{temp46}
                \begin{aligned}
                    \l\|\widetilde{\chi}_t\r\|_\HS\leq&2\eta_t\kappa\l(M_\rho+\kappa\l\|L_t+R_t+H^\dagger\r\|_{\mathrm{HS}}\r)\left\|C^\alpha\prod_{j=t+1}^{T}(I-\eta_j(C+\lambda_j I))\right\|\mathbbm{1}_{A_t} \\
                    \leq& 2\eta_t\kappa\l(M_\rho+2\kappa\|H^\dagger\|_\HS+\kappa d_2(t+t_0)^{\frac12-\theta_1}\log(t+t_0)\log\frac2\delta\r) \\
                    &\times\left\|C^\alpha\prod_{j=t+1}^{T}(I-\eta_j(C+\lambda_j I))\right\|.
                \end{aligned}
            \end{equation}
            If $\alpha>0$, by Lemma \ref{lemma1} (1) and Lemma \ref{lemma2}, we obtain
            \begin{equation} \label{temp47}
                \begin{aligned}
                    \left\|C^\alpha\prod_{j=t+1}^{T}(I-\eta_j(C+\lambda_j I))\right\|
                    &\leq\exp\left\{-\sum_{j=t+1}^{T}\eta_{j}\lambda_j\right\}\frac{2(\kappa^{2\alpha}+(\alpha/e)^\alpha)}{1+\left(\sum_{j=t+1}^{T}\eta_{j}\right)^\alpha} \\
                    &\leq2(\kappa^{2\alpha}+(\alpha/e)^\alpha)\l(\frac{T+t_0+1}{t+t_0+1}\r)^{-\bar\eta\bar\lambda}.                    
                \end{aligned}
            \end{equation}
            If $\alpha=0$, then
            \begin{equation} \label{temp48}
                \begin{aligned}
                    \left\|\prod_{j=t+1}^{T}(I-\eta_j(C+\lambda_j I))\right\|
                    \leq\exp\left\{-\sum_{j=t+1}^{T}\eta_{j}\lambda_j\right\}
                    \leq\l(\frac{T+t_0+1}{t+t_0+1}\r)^{-\bar\eta\bar\lambda}.
                \end{aligned}
            \end{equation}
            Substituting \eqref{temp47} or \eqref{temp48} into \eqref{temp46} yields that
            \begin{equation*}
                \l\|\widetilde{\chi}_t\r\|_\HS\leq d_3
                (T+t_0)^{-\bar\eta\bar\lambda}(t+t_0)^{\bar\eta\bar\lambda-\theta_1}\l(1+(t+t_0)^{\frac12-\theta_1}\log(t+t_0)\log\frac2\delta\r),
            \end{equation*}
            where $d_3$ is a constant independent of $\delta$, $T$, and $t$. If $\bar\eta\bar\lambda\geq\theta_1$ and $\bar\eta\bar\lambda\geq2\theta_1-\frac12$, the right-hand side of the above inequality achieves its maximum within $1\leq t\leq T$ at $t=T$. Therefore, we obtain
            \begin{equation*}
                \sup_{1\leq t\leq T}\l\|\widetilde{\chi}_t\r\|_\HS\leq d_3(T+t_0)^{-\theta_1}+d_3(T+t_0)^{\frac12-2\theta_1}\log(T+t_0)\log\frac2\delta.
            \end{equation*}
            Thus,
            \[
            \sup_{1\leq t\leq T}\l\|\widetilde{\chi}_t\r\|^2_\HS\leq 2d_3^2(T+t_0)^{-2\theta_1}+2d_3^2(T+t_0)^{1-4\theta_1}\log^2(T+t_0)\log^2\frac2\delta.
            \]
        \item[(2)] 
        By \eqref{chi_t}, we see that
        \begin{equation*}
            \begin{aligned}
                \sum_{t=1}^T\be_{z_t}\l\|\widetilde{\chi}_t\r\|_\HS^2&=
                \sum_{t=1}^T\eta_{t}^2\be_{z_t}\l[\l\|\B_{t}C^\alpha\prod_{j=t+1}^{T}(I-\eta_j(C+\lambda_j I))\r\|_\HS^2\mathbbm{1}_{A_t}\r] \\
                &\leq\sum_{t=1}^T\eta_{t}^2\be_{z_t}\l[\l\|\l(y_t-H_t\phi(x_t)\r)\otimes\phi(x_t)C^\alpha\prod_{j=t+1}^{T}(I-\eta_j(C+\lambda_j I))\r\|_\HS^2\mathbbm{1}_{A_t}\r] \\
                &=\sum_{t=1}^T\eta_{t}^2\be_{z_t}\l[\|y_t-H_t\phi(x_t)\|_\Y^2\l\|C^\alpha\prod_{j=t+1}^{T}(I-\eta_j(C+\lambda_j I))\phi(x_t)\r\|_{\H_\K}^2\mathbbm{1}_{A_t}\r],
            \end{aligned}
        \end{equation*}
        where the last inequality uses the definition of the Hilbert-Schmidt norm. 
        By \eqref{LtRt}, Lemma \ref{lemma6}, and \eqref{At}, it follows that
        \begin{equation*}
            \|y_t-H_t\phi(x_t)\|_\Y^2\mathbbm{1}_{A_t}
            \leq 2M_\rho^2+2\kappa^2\l(8\|H^\dagger\|_\HS^2+2d_2^2(t+t_0)^{1-2\theta_1}\log^2(t+t_0)\log^2\frac2\delta\r).
        \end{equation*}
        Then, 
        \begin{equation} \label{temp71}
            \begin{aligned}
\sum_{t=1}^T\be_{z_t}\l\|\widetilde{\chi}_t\r\|_\HS^2\leq&\sum_{t=1}^T\eta_{t}^2\l(2M_\rho^2+2\kappa^2\l(8\|H^\dagger\|_\HS^2+2d_2^2(t+t_0)^{1-2\theta_1}\log^2(t+t_0)\log^2\frac2\delta\r)\r)\\ 
&\times\be_{z_t}\l\|C^\alpha\prod_{j=t+1}^{T}(I-\eta_j(C+\lambda_j I))\phi(x_t)\r\|_{\H_\K}^2.
            \end{aligned}
        \end{equation}
        
        By the definition of the trace of operators and using Assumption \ref{a3} and (3) in Lemma \ref{lemma1}, it follows that
        if $2\alpha+1-s>0$, 
        \begin{equation} \label{temp51}
            \begin{aligned}
                &\quad\be_{z_t}\l\|C^\alpha\prod_{j=t+1}^{T}(I-\eta_j(C+\lambda_j I))\phi(x_t)\r\|_{\H_\K}^2\\
                &\quad =\mathrm{Tr}\l(C^{2\alpha+1}\prod_{j=t+1}^{T}\l(I-\eta_j(C+\lambda_j I)\r)^2\r) \\
                &\quad \leq \mathrm{Tr}\l(C^s\r)\l\|C^{2\alpha+1-s}\prod_{j=t+1}^{T}\l(I-\eta_j(C+\lambda_j I)\r)^2\r\| \\
                &\quad \leq\mathrm{Tr}\l(C^s\r)\frac{2(\kappa^{2(2\alpha+1-s))}+((2\alpha+1-s)/(2e))^{2\alpha+1-s})}{1+\left(\sum_{j=t+1}^{T}\eta_{j}\right)^{2\alpha+1-s}} \\
                &\quad \quad \times\exp\left\{-2\sum_{j=t+1}^{T}\eta_{j}\lambda_j\right\},
            \end{aligned}
        \end{equation}
        else if $2\alpha+1-s=0$,
        \begin{equation} \label{temp64}
            \be_{z_t}\l\|C^\alpha\prod_{j=t+1}^{T}(I-\eta_j(C+\lambda_j I))\phi(x_t)\r\|_{\H_\K}^2\leq\mathrm{Tr}\l(C^s\r)\exp\left\{-2\sum_{j=t+1}^{T}\eta_{j}\lambda_j\right\}.
        \end{equation}

    Substituting \eqref{temp51} (or \eqref{temp64}) into \eqref{temp71}, we deduce that there exists a constant $d_4$ independent of $\delta$, $t$, and $T$, such that
        \begin{equation} \label{temp52}
            \begin{aligned}
                \sum_{t=1}^T\be_{z_t}\l\|\widetilde{\chi}_t\r\|_\HS^2\leq&d_4\sum_{t=1}^T(t+t_0)^{-2\theta_1}\l(1+(t+t_0)^{1-2\theta_1}\log^2(T+t_0)\log^2\frac2\delta\r) \\
                &\times\frac{\exp\left\{-2\sum_{j=t+1}^{T}\eta_{j}\lambda_j\right\}}{1+\left(\sum_{j=t+1}^{T}\eta_{j}\right)^{2\alpha+1-s}}.
            \end{aligned}
        \end{equation}
        Now, we apply Proposition \ref{prop1} and Remark \ref{remark1} with $v=2\alpha+1-s$ to derive the bound.  Let $d_{5}$ denote a constant independent of $T$ and $\delta$.\\
        \textbf{Case 1:}
        If $2\alpha+1-s>1$, then
        \begin{equation*}
            \sum_{t=1}^T\be_{z_t}\l\|\widetilde{\chi}_t\r\|_\HS^2\leq
            d_{5}(T+t_0)^{-\theta_1}+d_{5}(T+t_0)^{1-3\theta_1}\log^2(T+t_0)\log^2\frac2\delta.
        \end{equation*}
        \textbf{Case 2:}
        If $2\alpha+1-s=1$, then
        \begin{equation*}
            \sum_{t=1}^T\be_{z_t}\l\|\widetilde{\chi}_t\r\|_\HS^2\leq
            d_{5}(T+t_0)^{-\theta_1}\log(T+t_0)+d_{5}(T+t_0)^{1-3\theta_1}\log^3(T+t_0)\log^2\frac2\delta.
        \end{equation*}
        \textbf{Case 3:}
        If $0<2\alpha+1-s<1$, then
        \begin{equation*}
            \sum_{t=1}^T\be_{z_t}\l\|\widetilde{\chi}_t\r\|_\HS^2\leq
            d_{5}(T+t_0)^{s-2\alpha-\theta_1(1+s-2\alpha)}+d_{5}(T+t_0)^{1+s-2\alpha-\theta_1(3+s-2\alpha)}\log^2(T+t_0)\log^2\frac2\delta.
        \end{equation*}
        \textbf{Case 4:}
        If $2\alpha+1-s=0$, since
        \begin{equation*}
            \begin{aligned}
                \exp\left\{-2\sum_{j=t+1}^{T}\eta_{j}\lambda_j\right\}&=\exp\{-2\bar\eta\bar\lambda\sum_{j=t+1}^T(j+t_0)^{-1}\} \\
                &\leq\l(\frac{T+t_0+1}{t+t_0+1}\r)^{-2\bar\eta\bar\lambda},
            \end{aligned}
        \end{equation*}
        substituting this into \eqref{temp52} gives
        \begin{equation*}
            \begin{aligned}
                \sum_{t=1}^T\be_{z_t}\l\|\widetilde{\chi}_t\r\|_\HS^2&\leq d_{5}(T+t_0)^{-2\bar\eta\bar\lambda}\l((T+t_0)^{2\bar\eta\bar\lambda-2\theta_1+1}+(T+t_0)^{2\bar\eta\bar\lambda+2-4\theta_1}\log^2(T+t_0)\log^2\frac2\delta\r) \\
                &\leq d_{5}\l((T+t_0)^{1-2\theta_1}+(T+t_0)^{2-4\theta_1}\log^2(T+t_0)\log^2\frac2\delta\r),
            \end{aligned}
        \end{equation*}
       which is consistent with the bound in Case 3.     
        \end{itemize}       
        The proof is complete. 
    \end{proof}

    Since $\widetilde{\chi}_t$ is $\sigma\l(z_1,z_2,\cdots, z_t\r)$ measurable and $\be_{z_t}\l[\widetilde{\chi}_t\r]=0$, $\l(\widetilde{\chi}_t\r)_{1\leq t\leq T}$ is a martingale difference sequence. Based on Proposition \ref{prop10}, we derive the high-probability error bounds.

\begin{proof}[Proof of Theorem \ref{thm5}]
        By Proposition \ref{proposition error2} with $\alpha=\frac12$, 
        \begin{equation} \label{temp53} 
		\left\|\left(H_{T+1}-H^{\dagger}\right)C^\frac12\right\|_{\mathrm{HS}}^2
            \leq
            \T_1+\T_2+\T_3+6\left\|\sum_{t=1}^T\chi_t\right\|^2_{\mathrm{HS}}.
        \end{equation}
        Using Proposition \ref{corollary}, there holds
        \begin{equation} \label{temp54}
            \mathbb{P}\l(\widetilde{\chi}_t=\chi_t \text{ for any } 1\leq t\leq T\r)\geq 1-\delta.
        \end{equation}
        Choose $\theta_1=\frac{\min\{2r+1,2\}}{1+\min\{2r+1,2\}}$ and $\theta_2=\frac{1}{1+\min\{2r+1,2\}}$.
        Applying Proposition \ref{prop2}, Proposition \ref{prop3} and Proposition \ref{prop4} with $\alpha=\frac12$ and $T\geq  t_0+1$,
        we have 
        \begin{equation} \label{temp55}
            \T_1\leq c_1\lambda_T^{\min\{2r+1,2\}}=c_1\bar\lambda^{\min\{2r+1,2\}}(T+t_0)^{-\frac{\min\{2r+1,2\}}{1+\min\{2r+1,2\}}},
        \end{equation}
        \begin{equation} \label{temp56}
            \begin{aligned}
                \T_2&\leq c_2\bar{\eta}^{-(2r+1)}(t_0+1)^{2\bar{\eta}\bar{\lambda}}(T+t_0)^{-(2r+1)(1-\theta_1)-2\bar{\eta}\bar{\lambda}} \\
                &\leq c_2\bar{\eta}^{-(2r+1)}(t_0+1)^{2\bar{\eta}\bar{\lambda}}(T+t_0)^{-\frac{\min\{2r+1,2\}}{1+\min\{2r+1,2\}}},
            \end{aligned}
        \end{equation}
        and 
        \begin{equation} \label{temp57}
            \T_3\leq c_3(T+t_0)^{-2\theta_2\min\{r,1\}+\theta_1-1}
            \leq c_3(T+t_0)^{-\frac{\min\{2r+1,2\}}{1+\min\{2r+1,2\}}}.
        \end{equation}
        Using Proposition \ref{prop10} and Proposition \ref{prop11}, we deduce that with probability at least $1-\delta$,
        \begin{equation*}
        \l\|\sum_{i=1}^T\widetilde{\chi}_t\r\|_\HS^2\leq8\l(M_1^2+\tau_1^2\r)\log^2\frac2\delta
        \end{equation*}
        with
        \begin{equation*}
            \begin{aligned}
                M_1^2&=2d_3^2(T+t_0)^{-2\theta_1}+2d_3^2(T+t_0)^{1-4\theta_1}\log^2(T+t_0)\log^2\frac2\delta \\
                &\leq 2d_3^2(T+t_0)^{-\theta_1}+2d_3^2(T+t_0)^{1-3\theta_1}\log^2(T+t_0)\log^2\frac2\delta,
            \end{aligned}
        \end{equation*}
        and 
        \begin{equation*}
            \begin{aligned}
                \tau_1^2&=d_{5}
                \begin{cases}
                    (T+t_0)^{-\theta_1}+(T+t_0)^{1-3\theta_1}\log^2(T+t_0)\log^2\frac2\delta, &\text{ when }s<1, \\
                    (T+t_0)^{-\theta_1}\log(T+t_0)+(T+t_0)^{1-3\theta_1}\log^3(T+t_0)\log^2\frac2\delta, &\text{ when } s=1. \\ 
                \end{cases} \\
            \end{aligned}
        \end{equation*}
        Therefore, 
        \begin{equation} \label{temp58}
            \begin{aligned}
                \l\|\sum_{i=1}^T\widetilde{\chi}_t\r\|_\HS^2\leq&8(2d_3^2+d_{5})\l((T+t_0)^{-\theta_1}+(T+t_0)^{1-3\theta_1}\log^2(T+t_0)\log^2\frac2\delta\r) \\
                &\times\log^2\frac2\delta
                \begin{cases}
                    1, &\text{ when }s<1, \\
                    \log(T+t_0), &\text{ when }s=1.
                \end{cases}
            \end{aligned}
        \end{equation}
        If $T\geq t_0+1$, combining \eqref{temp53}, \eqref{temp54}, \eqref{temp55}, \eqref{temp56}, \eqref{temp57}, and \eqref{temp58}, we obtain that there exists some constant $c_{2,1}$ independent of $T$ and $\delta$, such that
        \begin{equation*}
            \begin{aligned}
                &\left\|\left(H_{T+1}-H^{\dagger}\right)C^\frac12\right\|_{\mathrm{HS}}^2
            \\&\leq c_{2,1}
            \begin{cases}
                (T+t_0)^{-\theta_1}\log^2\frac2\delta+(T+t_0)^{1-3\theta_1}\log^2(T+t_0)\log^4\frac2\delta, & s<1, \\
                (T+t_0)^{-\theta_1}\log(T+t_0)\log^2\frac2\delta+(T+t_0)^{1-3\theta_1}\log^3(T+t_0)\log^4\frac2\delta, & s=1. \\
            \end{cases}
            \end{aligned}
        \end{equation*}
        holds with probability at least $1-2\delta$. 
        
        Since it is easy to verify that $\left\|\left(H_{T+1}-H^{\dagger}\right)C^\frac12\right\|_{\mathrm{HS}}^2$
         for $1\leq T<t_0+1$ can be bounded uniformly by a constant, 
         we can choose $c_{2,1}$ to be sufficiently large such that the bound also holds true for $1\leq T< t_0+1$.
        
        The proof is complete.
    \end{proof} 

    \begin{proof}[Proof of Corollary \ref{coro1}]
    For any $t\geq1$ and any $\delta\in(0,2/e)$, apply Theorem \ref{thm5} with $\delta_t=(t+t_0)^{-2}t_0\delta$, so that  $\sum_{t\geq1}\delta_t\leq\delta$. When $s<1$, we have
        \begin{equation*}
            \begin{aligned}
                \E(h_{t+1})-\E(h^\dagger)\lesssim(t+t_0)^{-\theta_1}\log^4\frac{2}{\delta_t}\lesssim(t+t_0)^{-\theta_1}\log^4(t+t_0)\log^4\frac{2}{\delta}.
            \end{aligned}
        \end{equation*}
        When $s=1$, similarly, we have
        \[
        \E(h_{t+1})-\E(h^\dagger)\lesssim(t+t_0)^{-\theta_1}\log^5(t+t_0)\log^4\frac{2}{\delta}.
        \]
        
        The proof is complete.
    \end{proof}

    \begin{proof}[Proof of Theorem \ref{thm8}]
        By Proposition \ref{proposition error2} with $\alpha=0$, 
        \begin{equation} \label{temp88}
		\left\|H_{T+1}-H^{\dagger}\right\|_{\mathrm{HS}}^2
            \leq
            \T_1+\T_2+\T_3+6\left\|\sum_{t=1}^T\chi_t\right\|^2_{\mathrm{HS}}.
        \end{equation}
        According to Proposition \ref{corollary}, we have
        \begin{equation} \label{temp89}
            \mathbb{P}\l(\widetilde{\chi}_t=\chi_t \text{ for any } 1\leq t\leq T\r)\geq 1-\delta.
        \end{equation}
        Applying Proposition \ref{prop2}, Proposition \ref{prop3}, and Proposition \ref{prop7} with $\alpha=0$ and $T\geq  t_0+1$,
        we obtain
        \begin{equation} \label{temp90}
            \T_1\leq c_1\lambda_T^{\min\{2r,2\}}=c_1\bar\lambda^{\min\{2r,2\}}(T+t_0)^{-\min\{2r,2\}\theta_2},
        \end{equation}
        \begin{equation} \label{temp91}
            \begin{aligned}
                \T_2&\leq c_2\bar{\eta}^{-2r}(t_0+1)^{2\bar{\eta}\bar{\lambda}}(T+t_0)^{-2r(1-\theta_1)-2\bar{\eta}\bar{\lambda}}, 
            \end{aligned}
        \end{equation}
        and 
        \begin{equation} \label{temp92}
            \T_3\leq c_3(T+t_0)^{-2\theta_2\min\{r,1\}}.
        \end{equation}
        Using Proposition \ref{prop10} and Proposition \ref{prop11}, we deduce that with probability at least $1-\delta$,
        \begin{equation*}
        \l\|\sum_{i=1}^T\widetilde{\chi}_t\r\|_\HS^2\leq8\l(M_1^2+\tau_1^2\r)\log^2\frac2\delta.
        \end{equation*}
        Additionally, we have
        \begin{equation*}
            \begin{aligned}
                M_1^2&=2d_3^2(T+t_0)^{-2\theta_1}+2d_3^2(T+t_0)^{1-4\theta_1}\log^2(T+t_0)\log^2\frac2\delta,
            \end{aligned}
        \end{equation*}
        and 
        \begin{equation*}
            \begin{aligned}
                \tau_1^2=d_{5}(T+t_0)^{s-\theta_1(1+s)}+d_5(T+t_0)^{1+s-\theta_1(3+s)}\log^2(T+t_0)\log^2\frac2\delta.
            \end{aligned}
        \end{equation*}
        Thus,
        \begin{equation} \label{temp93}
            \begin{aligned}
                \l\|\sum_{i=1}^T\widetilde{\chi}_t\r\|_\HS^2\leq& 8\l(2d_3^2+d_5\r)\left((T+t_0)^{s-\theta_1(1+s)}+(T+t_0)^{1+s-\theta_1(3+s)} \right.
                \\&\left.\times\log^2(T+t_0)\log^2\frac2\delta\right)\log^2\frac2\delta.
            \end{aligned}
        \end{equation}
        Combining \eqref{temp88}, \eqref{temp89}, \eqref{temp90}, \eqref{temp91}, \eqref{temp92}, and \eqref{temp93}, let $c_{2,2}$ denote a constant independent of $T$ and $\delta$,
        \begin{itemize}
            \item[(1)] 
            if $2r+s\leq1$, choose 
            $\theta_1=\frac{1+2\min\{r,1\}+s}{3+2\min\{r,1\}+s}=\frac{1+2r+s}{3+2r+s}$ and 
            \[
            \theta_2=\frac{2}{3+2\min\{r,1\}+s}=\frac{2}{3+2r+s},
            \]
            then,
            \begin{equation*}
                \begin{aligned}
                    \left\|H_{T+1}-H^{\dagger}\right\|_{\mathrm{HS}}^2
            &\leq c_{2,2}(T+t_0)^{-\frac{4\min\{r,1\}}{3+2\min\{r,1\}+s}}\log^2(T+t_0)\log^4\frac2\delta
            \\&= c_{2,2}(T+t_0)^{-\frac{4r}{3+2r+s}}\log^2(T+t_0)\log^4\frac2\delta
                \end{aligned}    
            \end{equation*}
            holds with probability at least $1-2\delta$;
            \item[(2)] 
            if $2\min\{r,1\}+s>1$, i.e., $2r+s>1$, choose $\theta_1=\frac{2\min\{r,1\}+s}{1+2\min\{r,1\}+s}$  and $\theta_2=\frac{1}{1+2\min\{r,1\}+s}$,  then
            \begin{equation*}
                \begin{aligned}
                    \left\|H_{T+1}-H^{\dagger}\right\|_{\mathrm{HS}}^2
            &\leq c_{2,2}\l((T+t_0)^{-\frac{2\min\{r,1\}}{1+2\min\{r,1\}+s}}+(T+t_0)^{-\frac{4\min\{r,1\}+s-1}{1+2\min\{r,1\}+s}}\log^2(T+t_0)\log^2\frac2\delta\r)\log^2\frac2\delta \\
            &\lesssim (T+t_0)^{-\frac{2\min\{r,1\}}{1+2\min\{r,1\}+s}}\log^4\frac2\delta
                \end{aligned}    
            \end{equation*}
            holds with probability at least $1-2\delta$.
        \end{itemize}
        
        Since $\left\|H_{T+1}-H^{\dagger}\right\|_{\mathrm{HS}}^2$
         for $1\leq T<t_0+1$ can be uniformly bounded by a constant, we can choose $c_{2,2}$ to be sufficiently large so that the bounds also hold in this case.
         
        The proof is then complete.
    \end{proof} 

    \begin{proof}[Proof of Corollary \ref{coro2}]
        For any $\delta\in(0,2/e)$, using Theorem \ref{thm8} with $\delta_t=(t+t_0)^{-2}t_0\delta$, we derive the desired bounds.
    \end{proof}

    Next, we focus on the finite-horizon setting and derive the corresponding high-probability error bounds.
    \begin{proposition} \label{corollary1}
        Under Assumption \ref{a5}, let $T\geq2$, $t_0=\theta_1=\theta_2=0$, $\eta_t=\bar\eta=\eta_\star T^{-\theta_3}$ and $\lambda_t=\bar\lambda=\lambda_\star T^{-\theta_4}$. Suppose that $\eta_\star(\kappa^2+\lambda_\star)\leq1$. Then, for any $\delta\in(0,2/e)$, with probability at least $1-\delta$, there holds
        \begin{equation*}
            \l\|R_{t}\r\|_\HS\leq \frac{2d_6}{\log 2}T^{\frac12-\theta_3}\log T\log\frac{2}{\delta},\quad 1\leq t\leq T,
        \end{equation*}
        where $d_6=d_6(\eta_\star,\theta_3,M_\rho,\|S^\dagger\|_{\mathrm{HS}},r,\kappa)$ is a constant independent of $t$, $T$, and $\delta$.
    \end{proposition}
    
    \begin{proof}
        The proof follows the same strategy as Proposition \ref{corollary}. Define  $K_t$ as in the proof of Proposition \ref{corollary}, and represent $R_{t+1}$ in the same summation form.
        By \eqref{temp68}, there holds
        \begin{equation*}
            \begin{aligned}
                \l\|\eta_iK_i\prod_{j=i+1}^t\l(I-\eta_j(C_j+\lambda_j I)\r)\r\|_\HS
                \leq\eta_\star\l(\kappa M_\rho+3\kappa^2\|H^\dagger\|_\HS\r)T^{-\theta_3},
            \end{aligned}
        \end{equation*}
        and
        \begin{equation*}
            \sum_{i=1}^t\be_{z_i}\l[\l\|\eta_iK_i\prod_{j=i+1}^t\l(I-\eta_j(C_j+\lambda_j I)\r)\r\|_\HS^2\r]
            \leq\eta_\star^2\l(\kappa M_\rho+3\kappa^2\|H^\dagger\|_\HS\r)^2T^{1-2\theta_3}.
        \end{equation*}
        Then, by Proposition \ref{prop10}, for some constant $d_6$ independent of $t$, $T$, and $\delta_{t+1}$, with probability at least $1-\delta_{t+1}$, it holds that
        \begin{equation*}
            \l\|R_{t+1}\r\|_\HS\leq d_6T^{\frac12-\theta_3}\log\frac{2}{\delta_{t+1}}.
        \end{equation*}
        For $\delta\in(0,2/e)$, choose $\delta_t=\frac{\delta}{T}$ for any $1\leq t\leq T$. Then we obtain
        \begin{equation*}
            \l\|R_{t}\r\|_\HS\leq \frac{2d_6}{\log 2}T^{\frac12-\theta_3}\log T\log\frac{2}{\delta}, \quad 1\leq t\leq T.
        \end{equation*}

        The proof is complete.
    \end{proof}

    We now define 
    \begin{equation}\label{Xtbar}
        \overline{\chi}_t:=\chi_t\mathbbm{1}_{\overline{A}_t},
    \end{equation}
    where 
    \begin{equation} \label{Atbar}
        \overline{A}_t:=\l\{\|R_t\|_\HS\leq \frac{2d_6}{\log 2}T^{\frac12-\theta_3}\log T\log\frac{2}{\delta}\r\}.
    \end{equation}
    Note that $\overline{A}_t$ is independent of $z_t$, and for any $t \in \bn_T$, we have $\be_{z_t} \l[\overline{\chi}_t \r] = 0$. Moreover, by Proposition \ref{corollary1},      
    \[
    \mathbb{P}\l(\overline{\chi}_t=\chi_t \text{ for any } 1\leq t\leq T\r)\geq 1-\delta.
    \]

    \begin{proposition} \label{prop13}
        Suppose Assumptions \ref{a3} and \ref{a5} hold. Let $\alpha\in[0,\frac12]$, $t_0=\theta_1=\theta_2=0$, $\eta_t\equiv\bar\eta=\eta_\star T^{-\theta_3}$ with $0<\theta_3<1$, and $\lambda_t\equiv\bar\lambda=\lambda_\star T^{-\theta_4}$ with $\theta_4>0$.
        Assume $T\geq2$ and $\eta_\star(\kappa^2+\lambda_\star)\leq1$. Then,
        \begin{itemize}
            \item[(1)] 
            The Hilbert-Schmidt norm of\/ $\overline{\chi}_t$ is uniformly bounded as follows:
            \begin{equation*}
                \sup_{1\leq t\leq T}\l\|\overline{\chi}_t\r\|^2_\HS\leq
                M_2^2,
            \end{equation*}   
            where 
            \[M_2^2:=d_7T^{-2\theta_3}\l(1+T^{1-2\theta_3}\log^2 T\log^2\frac{2}{\delta}\r).\]            
            \item[(2)] The total squared Hilbert-Schmidt norm in expectation is bounded by
            \[\sum_{t=1}^T\be_{z_t}\l\|\overline{\chi}_t\r\|_\HS^2\leq \tau_2^2,\]
            where $\tau_2$ is defined as
            \begin{equation*}
            \begin{aligned}
                &\ \tau_2^2:=
                \\&d_8\delta_3\l(1+T^{1-2\theta_3}\log^2 T\log^2\frac{2}{\delta}\r)
                    \begin{cases}
                    T^{-(1-2\alpha+s)\theta_3+(s-2\alpha)\min\{1,\theta_3+\theta_4\}}, & \text{ when } 2\alpha<s\leq1+2\alpha, \\
                    T^{-\theta_3}\log T, & \text{ when } 2\alpha=s, \\
                    T^{-\theta_3}, & \text{ when } 2\alpha>s.
                    \end{cases}
            \end{aligned}
            \end{equation*}
        \end{itemize}
        Here $d_7$ depends only on
$\eta_\star$, $\theta_3$, $M_\rho$,
$\|S^\dagger\|_{\mathrm{HS}}$, $r$, $\alpha$, and $\kappa$,
while $d_8$ depends only on
$\eta_\star$, $\lambda_\star$, $\theta_3$, $\theta_4$,
$M_\rho$, $\|S^\dagger\|_{\mathrm{HS}}$, $r$, $\alpha$, $s$,
$\kappa$, and $\mathrm{Tr}(C^s)$.
Both constants are independent of $T$ and $\delta$.
    \end{proposition}

    \begin{proof}
        \begin{itemize}
            \item[(1)] 
            Using the same almost sure bound for $\mathcal B_t$ as in
the proof of Proposition~\ref{prop11}, together with \eqref{LtRt}, Lemma~\ref{lemma6},
and \eqref{Atbar}, we obtain, for $1\leq t\leq T$,
            \begin{equation} 
                \begin{aligned}
                    \l\|\overline{\chi}_t\r\|_\HS\leq&2\eta_t\kappa\l(M_\rho+\kappa\l\|L_t+R_t+H^\dagger\r\|_{\mathrm{HS}}\r)\left\|C^\alpha\prod_{j=t+1}^{T}(I-\eta_j(C+\lambda_j I))\right\|\mathbbm{1}_{\overline{A}_t} \\
                    \leq& 2\eta_\star\kappa^{1+2\alpha}\l(M_\rho+2\kappa\|H^\dagger\|_\HS+\kappa \frac{2d_6}{\log 2}T^{\frac12-\theta_3}\log T\log\frac{2}{\delta}\r) 
                    T^{-\theta_3},
                \end{aligned}
            \end{equation}
            where the last inequality uses $\left\|C^\alpha\prod_{j=t+1}^{T}(I-\eta_j(C+\lambda_j I))\right\|\leq\kappa^{2\alpha}$. Hence, there exists a constant $d_7$ independent of $t$, $T$, and $\delta$, such that
            \[
            \sup_{1\leq t\leq T}\l\|\overline{\chi}_t\r\|_\HS^2\leq d_7T^{-2\theta_3}\l(1+T^{1-2\theta_3}\log^2 T\log^2\frac{2}{\delta}\r).
            \]

            \item[(2)] The argument is parallel to that for 
$\widetilde{\chi}_t$. We therefore only indicate the changes arising from the constant step size and regularization parameter.
            By \eqref{Xtbar},
        \begin{equation*}
            \begin{aligned}
                \sum_{t=1}^T\be_{z_t}\l\|\overline{\chi}_t\r\|_\HS^2&=
                \sum_{t=1}^T\eta_{t}^2\be_{z_t}\l[\l\|\B_{t}C^\alpha\prod_{j=t+1}^{T}(I-\eta_j(C+\lambda_j I))\r\|_\HS^2\mathbbm{1}_{\overline{A}_t}\r] \\
                &\leq\sum_{t=1}^T\eta_{t}^2\be_{z_t}\l[\|y_t-H_t\phi(x_t)\|_\Y^2\l\|C^\alpha\prod_{j=t+1}^{T}(I-\eta_j(C+\lambda_j I))\phi(x_t)\r\|_{\H_\K}^2\mathbbm{1}_{\overline{A}_t}\r].
            \end{aligned}
        \end{equation*}
        By Assumption \ref{a5}, \eqref{LtRt}, Lemma \ref{lemma6}, and \eqref{Atbar}, we obtain
        \begin{equation*}
            \|y_t-H_t\phi(x_t)\|_\Y^2\mathbbm{1}_{\overline{A}_t}
            \leq 2M_\rho^2+2\kappa^2\l(8\|H^\dagger\|_\HS^2+2\l(\frac{2d_6}{\log 2}\r)^2T^{1-2\theta_3}\log^2 T\log^2\frac{2}{\delta}\r).
        \end{equation*}
        Substituting into the earlier bound yields
        \begin{equation} \label{temp72}
            \begin{aligned}
\sum_{t=1}^T\be_{z_t}\l\|\overline{\chi}_t\r\|_\HS^2\leq&\sum_{t=1}^T\eta_{t}^2\l(2M_\rho^2+2\kappa^2\l(8\|H^\dagger\|_\HS^2+2\l(\frac{2d_6}{\log 2}\r)^2T^{1-2\theta_3}\log^2 T\log^2\frac{2}{\delta}\r)\r)
                \\ &\times\be_{z_t}\l\|C^\alpha\prod_{j=t+1}^{T}(I-\eta_j(C+\lambda_j I))\phi(x_t)\r\|_{\H_\K}^2.
            \end{aligned}
        \end{equation}

        Then, if $2\alpha + 1 - s > 0$,
        \begin{equation} \label{temp73}
            \begin{aligned}
                &\ \be_{z_t}\l\|C^\alpha\prod_{j=t+1}^{T}(I-\eta_j(C+\lambda_j I))\phi(x_t)\r\|_{\H_\K}^2
                \\=&\mathrm{Tr}\l(C^{2\alpha+1}\l(I-\bar\eta(C+\bar\lambda I)\r)^{2(T-t)}\r) \\
                \leq& \mathrm{Tr}\l(C^s\r)\l\|C^{2\alpha+1-s}\l(I-\bar\eta(C+\bar\lambda I)\r)^{2(T-t)}\r\| \\
                \leq&\mathrm{Tr}\l(C^s\r)\frac{2(\kappa^{2(2\alpha+1-s))}+((2\alpha+1-s)/(2e))^{2\alpha+1-s})}{1+\left((T-t)\bar\eta\right)^{2\alpha+1-s}} \\
                &\times\exp\left\{-2(T-t)\bar\eta\bar\lambda\right\},
            \end{aligned}
        \end{equation}
        and if $2\alpha+1-s=0$, then
        \begin{equation} \label{temp74}
            \be_{z_t}\l\|C^\alpha\prod_{j=t+1}^{T}(I-\eta_j(C+\lambda_j I))\phi(x_t)\r\|_{\H_\K}^2\leq\mathrm{Tr}\l(C^s\r)\exp\left\{-2(T-t)\bar\eta\bar\lambda\right\}.
        \end{equation}
            Substituting \eqref{temp73} or \eqref{temp74} into \eqref{temp72} yields
        \begin{equation*}
            \begin{aligned}
                \sum_{t=1}^T\be_{z_t}\l\|\overline{\chi}_t\r\|_\HS^2\leq&d_8T^{-2\theta_3}\l(1+T^{1-2\theta_3}\log^2 T\log^2\frac{2}{\delta}\r)\sum_{t=0}^{T-1}\frac{\exp\left\{-2\eta_\star\lambda_\star tT^{-\theta_4-\theta_3}\right\}}{1+\left(t\bar\eta\right)^{2\alpha+1-s}},
            \end{aligned}
        \end{equation*}
        where $d_8$ is a constant independent of $t$, $T$, and $\delta$. Using Proposition \ref{prop12} with $v = 1 + 2\alpha - s$, we obtain
        \begin{equation*}
            \begin{aligned}
                \sum_{t=1}^T\be_{z_t}\l\|\overline{\chi}_t\r\|_\HS^2\leq&
                d_8\delta_3T^{-2\theta_3}\l(1+T^{1-2\theta_3}\log^2 T\log^2\frac{2}{\delta}\r)
                \\ &\times
                    \begin{cases}
                    T^{(1+2\alpha-s)\theta_3+(s-2\alpha)\min\{1,\theta_3+\theta_4\}}, & \text{ when } 2\alpha<s\leq1+2\alpha, \\
                    T^{\theta_3}\log T, & \text{ when } 2\alpha=s, \\
                    T^{\theta_3}, & \text{ when } 2\alpha>s,
                    \end{cases}
                \\ =&d_8\delta_3\l(1+T^{1-2\theta_3}\log^2 T\log^2\frac{2}{\delta}\r)
                \\ &\times
                    \begin{cases}
                    T^{-(1-2\alpha+s)\theta_3+(s-2\alpha)\min\{1,\theta_3+\theta_4\}}, & \text{ when } 2\alpha<s\leq1+2\alpha, \\
                    T^{-\theta_3}\log T, & \text{ when } 2\alpha=s, \\
                    T^{-\theta_3}, & \text{ when } 2\alpha>s.
                    \end{cases}
            \end{aligned}
        \end{equation*}
        \end{itemize}
        The proof is complete.
    \end{proof}

    Next, we prove the high-probability bounds for prediction and estimation errors in the finite-horizon setting.

    \begin{proof}[Proof of Theorem \ref{thm6}]
        By Proposition \ref{proposition error2} with $\alpha=\frac12$, we have
        \begin{equation} \label{temp77} 
		\left\|\left(H_{T+1}-H^{\dagger}\right)C^\frac12\right\|_{\mathrm{HS}}^2
            \leq
            \T_1+\T_2+\T_3+6\left\|\sum_{t=1}^T\chi_t\right\|^2_{\mathrm{HS}},
        \end{equation}
        where $\T_3=0$.
        Using Proposition \ref{corollary1}, it holds that
        \begin{equation} \label{temp78}
            \mathbb{P}\l(\overline{\chi}_t=\chi_t \text{ for any } 1\leq t\leq T\r)\geq 1-\delta.
        \end{equation}
        Let $\theta_3=\frac{2r+1}{2r+2}$ and choose $\theta_4\geq\frac{2r+1}{(2r+2)\min\{2r+1,2\}}$.
        Applying Proposition \ref{prop2} and Proposition \ref{prop3} with $\alpha=\frac12$, $t_0=\theta_1=\theta_2=0$, $\bar\eta=\eta_\star T^{-\theta_3}$ and $\bar\lambda=\lambda_\star T^{-\theta_4}$,
        we obtain
        \begin{equation} \label{temp79}
            \T_1\leq c_1\lambda_T^{\min\{2r+1,2\}}=c_1\l(\lambda_\star T^{-\theta_4}\r)^{\min\{2r+1,2\}}
            \leq c_1\lambda_\star^{\min\{2r+1,2\}}T^{-\frac{2r+1}{2r+2}},
        \end{equation}
        and
        \begin{equation} \label{temp80}
            \begin{aligned}
                \T_2&\leq c_2\bar{\eta}^{-(2r+1)}T^{-(2r+1)}\exp\{-\tau\bar\eta\bar\lambda T\} 
                \\ &\leq c_2\eta_\star^{-(2r+1)}T^{-\frac{2r+1}{2r+2}}.
            \end{aligned}
        \end{equation}
        Using Proposition \ref{prop10} and Proposition \ref{prop13}, we conclude that with probability at least $1-\delta$,
        \begin{equation*}
\l\|\sum_{i=1}^T\overline{\chi}_t\r\|_\HS^2\leq8\l(M_2^2+\tau_2^2\r)\log^2\frac2\delta.
        \end{equation*}
        Additionally, for $\alpha=\frac12$, we have
        \begin{equation*}
            \begin{aligned}
                M_2^2=&d_7T^{-2\theta_3}\l(1+T^{1-2\theta_3}\log^2 T\log^2\frac{2}{\delta}\r) \\
                \leq&d_7\l(T^{-\theta_3}+T^{1-3\theta_3}\log^2 T\log^2\frac{2}{\delta}\r),
            \end{aligned}
        \end{equation*}
        and 
        \begin{equation*}
            \begin{aligned}
                \tau_2^2=d_8\delta_3\l(T^{-\theta_3}+T^{1-3\theta_3}\log^2 T\log^2\frac{2}{\delta}\r)
                    \begin{cases}
                    \log T, & \text{ when } s=1, \\
                    1, & \text{ when } s<1.
                    \end{cases}
            \end{aligned}
        \end{equation*}
        Therefore, 
        \begin{equation} \label{temp81}
            \begin{aligned}
                \l\|\sum_{i=1}^T\overline{\chi}_t\r\|_\HS^2\leq&8(d_7+d_8\delta_3)\l(T^{-\theta_3}+T^{1-3\theta_3}\log^2 T\log^2\frac{2}{\delta}\r)\log^2\frac{2}{\delta}
                    \begin{cases}
                    \log T, & \text{ when } s=1, \\
                    1, & \text{ when } s<1.
                    \end{cases}
            \end{aligned}
        \end{equation}
        Combining \eqref{temp77}, \eqref{temp78}, \eqref{temp79}, \eqref{temp80}, and \eqref{temp81}, we conclude that there exists a constant $c_{2,3}$ independent of $T$, such that
        \begin{equation*}
            \begin{aligned}
                \left\|\left(H_{T+1}-H^{\dagger}\right)C^\frac12\right\|_{\mathrm{HS}}^2
            \leq c_{2,3}
            \begin{cases}
                T^{-\theta_3}\log^2\frac2\delta+T^{1-3\theta_3}\log^2T\log^4\frac2\delta, & \text{ when } s<1, \\
                T^{-\theta_3}\log T\log^2\frac2\delta+T^{1-3\theta_3}\log^3T\log^4\frac2\delta, & \text{ when } s=1, \\
            \end{cases}
            \end{aligned}
        \end{equation*}
        holds with probability at least $1-2\delta$. 
        
        The proof is then complete.
    \end{proof}

    \begin{proof}[Proof of Theorem \ref{thm7}]
        By Proposition \ref{proposition error2} with $\alpha=0$, we have
        \begin{equation} \label{temp82} 
		\left\|H_{T+1}-H^{\dagger}\right\|_{\mathrm{HS}}^2
            \leq
            \T_1+\T_2+\T_3+6\left\|\sum_{t=1}^T\chi_t\right\|^2_{\mathrm{HS}},
        \end{equation}
        where $\T_3=0$.
        According to Proposition \ref{corollary1}, it holds that 
        \begin{equation} \label{temp83}
            \mathbb{P}\l(\overline{\chi}_t=\chi_t \text{ for any } 1\leq t\leq T\r)\geq 1-\delta.
        \end{equation}
        Applying Proposition \ref{prop2} and Proposition \ref{prop3} with $\alpha=0$, $t_0=\theta_1=\theta_2=0$, $\bar\eta=\eta_\star T^{-\theta_3}$ and $\bar\lambda=\lambda_\star T^{-\theta_4}$,
        we obtain 
        \begin{equation} \label{temp84}
            \T_1\leq c_1\lambda_T^{\min\{2r,2\}}=c_1\l(\lambda_\star T^{-\theta_4}\r)^{\min\{2r,2\}},
        \end{equation}
        and
        \begin{equation} \label{temp85}
            \begin{aligned}
                \T_2&\leq c_2\bar{\eta}^{-2r}T^{-2r}\exp\{-\tau\bar\eta\bar\lambda T\} 
                \\ &\leq c_2\eta_\star^{-2r}T^{-2r(1-\theta_3)}.
            \end{aligned}
        \end{equation}
        Using Proposition \ref{prop10} and Proposition \ref{prop13}, we conclude that, with probability at least $1-\delta$,
        \begin{equation} \label{temp86}
        \l\|\sum_{i=1}^T\overline{\chi}_t\r\|_\HS^2\leq8\l(M_2^2+\tau_2^2\r)\log^2\frac2\delta.
        \end{equation}
        Additionally, setting $\alpha=0$, we have
        \begin{equation*} 
            \begin{aligned}
                M_2^2=&d_7T^{-2\theta_3}\l(1+T^{1-2\theta_3}\log^2 T\log^2\frac{2}{\delta}\r),                 
            \end{aligned}
        \end{equation*}
        and 
        \begin{equation*} 
            \begin{aligned}
                \tau_2^2&=d_8\delta_3\l(1+T^{1-2\theta_3}\log^2 T\log^2\frac{2}{\delta}\r)T^{-(1+s)\theta_3+s\min\{1,\theta_3+\theta_4\}} \\
                &=d_8\delta_3\l(T^{-(1+s)\theta_3+s}+T^{-(3+s)\theta_3+1+s}\log^2 T\log^2\frac{2}{\delta}\r)
            \end{aligned}
        \end{equation*}
        when $\theta_3+\theta_4\geq1$. Therefore, if $\theta_3+\theta_4\geq1$, then
        \begin{equation} \label{temp87}
            \l\|\sum_{i=1}^T\overline{\chi}_t\r\|_\HS^2\leq 8(d_7+d_8\delta_3)\l(T^{-(1+s)\theta_3+s}+T^{-(3+s)\theta_3+1+s}\log^2 T\log^2\frac{2}{\delta}\r)\log^2\frac{2}{\delta}.
        \end{equation}

        Let $c_{2,4}$ be a constant independent of $T$ and $\delta$. Combining \eqref{temp82}, \eqref{temp83}, \eqref{temp84}, \eqref{temp85}, \eqref{temp86}, and \eqref{temp87}, we obtain the following estimates:
        
        \begin{itemize}
            \item[(1)] 
            If $2r+s>1$, choose $\theta_3=\frac{2r+s}{1+2r+s}$ and $\theta_4\geq\frac{r}{(1+2r+s)\min\{r,1\}}$, then 
            \begin{equation*}
                \begin{aligned}
                    \left\|H_{T+1}-H^{\dagger}\right\|_{\mathrm{HS}}^2
                    &\leq c_{2,4}\l(T^{-\frac{2r}{1+2r+s}}+T^{-\frac{4r+s-1}{1+2r+s}}\log^2 T\log^2\frac{2}{\delta}\r)\log^2\frac{2}{\delta} \\
                    &\lesssim T^{-\frac{2r}{1+2r+s}}\log^4\frac{2}{\delta}
                \end{aligned}
            \end{equation*}
            holds with probability at least $1-2\delta$. 

            \item[(2)] 
            If $2r+s\leq1$, choose $\theta_3=\frac{1+2r+s}{3+2r+s}$ and $\theta_4\geq\frac{2r}{(3+2r+s)\min\{r,1\}}=\frac{2r}{(3+2r+s)r}$, then
            \begin{equation*}
                \begin{aligned}
                    \left\|H_{T+1}-H^{\dagger}\right\|_{\mathrm{HS}}^2
                    &\leq c_{2,4}\l(T^{-\frac{1+2r-s}{3+2r+s}}+T^{-\frac{4r}{3+2r+s}}\log^2 T\log^2\frac{2}{\delta}\r)\log^2\frac{2}{\delta} \\
                    &\lesssim T^{-\frac{4r}{3+2r+s}}\log^2 T\log^4\frac{2}{\delta}
                \end{aligned}
            \end{equation*}
            holds with probability at least $1-2\delta$. 
        \end{itemize}
        The proof is then complete.
    \end{proof}

    
 

\begin{appendix}

    \section{Proof of Proposition \ref{transform}} \label{tran}

    \begin{proof}
        Since $W$ is self-adjoint and positive, $W^{1/2}$ is also self-adjoint and positive. Let $\H_0=\mathrm{span}\{K(x,\cdot)y=\K(x,\cdot)W y:x\in\mathcal{X},y\in\Y\}$ and $\B_0=\mathrm{span}\{W^{1/2}y\otimes\phi(x):x\in\mathcal{X},y\in\Y\}$. Then, it is clear that $\H_K=\overline{\H_0}$ and $\B_{\mathrm{HS}}(\H_{\K},\overline{W^{1/2}\Y})=\overline{\B_0}$.
		
		We define the mapping $W_0:\H_0\rightarrow\B_0$ by $\  \sum_{i=1}^{n}\alpha_iK(x_i,\cdot)y_i\mapsto\sum_{i=1}^{n}\alpha_iW^{1/2}y_i\otimes\phi(x_i)$ for any $n\in\bn$, $x_1,\cdots x_n\in\mathcal{X}$ and $y_1,\cdots y_n\in\Y$. We see that $W_0$ is well-defined and linear. Moreover, for any $\sum_{i=1}^{n}\alpha_iK(x_i,\cdot)y_i=\sum_{i=1}^{n}\alpha_i\K(x_i,\cdot)W y_i\in\H_0$, we have
        \begin{equation*}
            \begin{aligned}
                \l\|W_0\l(\sum_{i=1}^{n}\alpha_iK(x_i,\cdot)y_i\r)\r\|_\HS^2&=\sum_{i,j=1}^n\alpha_i\alpha_j\langle W^{1/2}y_i\otimes\phi(x_i), W^{1/2}y_j\otimes\phi(x_j)\rangle_\HS
                \\
                &=\sum_{i,j=1}^n\alpha_i\alpha_j\K(x_i,x_j)\langle W y_i,y_j\rangle_\Y
                \\ 
                &=\sum_{i,j=1}^n\alpha_i\alpha_j\langle K(x_i,x_j)y_i,y_j\rangle_\Y\\
                &=\l\|\sum_{i=1}^n\alpha_iK(x_i,\cdot)y_i\r\|_{\H}^2.
            \end{aligned}
        \end{equation*}
		By extending $W_0$ to $W:\H\rightarrow\B_{\mathrm{HS}}(\H_\K,\overline{W^{1/2}\Y})$, we conclude that $\H$ is isometric to 
        \[
        \B_{\mathrm{HS}}(\H_{\K},\overline{W^{1/2}\Y})\subset\B_{\mathrm{HS}}(\H_{\K},\Y).
        \]
		
		Next, we show that $h(x)=(W h)(\phi(x))$ for all $h\in\H$. For any $y\in\Y$,
		\begin{align*}
			\langle y,h(x)\rangle_{\Y}&=\langle K(x,\cdot)y,h\rangle_{\H}=\left\langle W(K(x,\cdot)y),Wh\right\rangle_{\mathrm{HS}}=\langle W^{1/2}y\otimes\phi(x),Wh\rangle_{\mathrm{HS}} \\
			&=\mathrm{Tr}\left(\l(W^{1/2}y\otimes\phi(x)\r)^*(Wh)\right)=\l\langle y, W^{1/2}(Wh)\phi(x)\r\rangle_{\Y},
		\end{align*}
		where the property $\langle y,h(x)\rangle_{\Y}=\langle K(x,\cdot)y,h\rangle_{\H}$ is used in the first equality. Thus, we  conclude that $h(x)= W^{1/2}(Wh)\phi(x)$. The uniqueness of this representation is obvious. This completes the proof.
	\end{proof}

    \section{Proof of Proposition \ref{prop14}} \label{Proof of prop14}
    
\begin{lemma}\label{Auxiliary lemma}
    Suppose that $\overline\phi=\be\l[\phi(x)\r]$. The moment condition \eqref{a4'} holds if 
    \begin{equation} \label{temp99}
        \be\l[\l\langle\phi(x)-\overline\phi,f\r\rangle_{\H_\K}^4\r]\leq c\l(\be\l[\l\langle\phi(x)-\overline\phi,f\r\rangle_{\H_\K}^2\r]\r)^2 
    \end{equation} holds for some constant $c>0$.
\end{lemma}
\begin{proof}
    First, using that $\be\l[\phi(x)-\overline\phi\r]=0$, we see that
    \begin{equation} \label{temp95}
        \begin{aligned}
            \l(\be\l[\l\langle\phi(x),f\r\rangle_{\H_\K}^2\r]\r)^2
            &=\l(\be\l[\l\langle\phi(x)-\overline\phi+\overline\phi,f\r\rangle_{\H_\K}^2\r]\r)^2
            \\ &=\l(\be\l[\l\langle\phi(x)-\overline\phi,f\r\rangle_{\H_\K}^2\r]\r)^2+\l\langle \overline\phi, f \r\rangle_{\H_\K}^4+2\l\langle \overline\phi, f \r\rangle_{\H_\K}^2\be\l[\l\langle\phi(x)-\overline\phi,f\r\rangle^2_{\H_\K}\r].
        \end{aligned}
    \end{equation}
    Using $\be\l[\phi(x)-\overline\phi\r]=0$ again, it holds that
    \begin{equation} \label{temp96}
        \begin{aligned}
            \be\l[\l\langle\phi(x),f\r\rangle_{\H_\K}^4\r]
            =&\be\l[\l(\l\langle \phi(x)-\overline\phi,f\r\rangle_{\H_\K}^2+\l\langle \overline\phi,f \r\rangle_{\H_\K}^2+2\l\langle \overline\phi,f \r\rangle_{\H_\K}\l\langle \phi(x)-\overline\phi,f\r\rangle_{\H_\K}\r)^2\r]
            \\ =&\be\l[\l\langle \phi(x)-\overline\phi,f\r\rangle_{\H_\K}^4\r]+\l\langle \overline\phi,f \r\rangle_{\H_\K}^4+6\l\langle \overline\phi, f \r\rangle_{\H_\K}^2\be\l[\l\langle\phi(x)-\overline\phi,f\r\rangle^2_{\H_\K}\r]
            \\&+4\l\langle \overline\phi, f \r\rangle_{\H_\K}
            \be\l[\l\langle\phi(x)-\overline\phi\r\rangle^3_{\H_\K}\r].
        \end{aligned}
    \end{equation}
    By H{\"o}lder's inequality and \eqref{temp99}, we obtain that
    \begin{equation} \label{temp97}
        \begin{aligned}
            &4\l\langle \overline\phi, f \r\rangle_{\H_\K} \be\l[\l\langle\phi(x)-\overline\phi\r\rangle^3_{\H_\K}\r]
            \\ \leq&4\l\langle \overline\phi, f \r\rangle_{\H_\K}\l(\be\l[\l\langle\phi(x)-\overline\phi\r\rangle^4_{\H_\K}\r]\r)^{3/4}
            \\ \leq&4c^{3/4}\l\langle \overline\phi, f \r\rangle_{\H_\K}\l(\be\l[\l\langle\phi(x)-\overline\phi\r\rangle^2_{\H_\K}\r]\r)^{3/2}
            \\ \leq&2c^{3/4}\l\langle \overline\phi, f \r\rangle_{\H_\K}^2\be\l[\l\langle\phi(x)-\overline\phi\r\rangle^2_{\H_\K}\r]
            +2c^{3/4}\l(\be\l[\l\langle\phi(x)-\overline\phi\r\rangle^2_{\H_\K}\r]\r)^2.
        \end{aligned}
    \end{equation}
    Then, taking \eqref{temp97} back into \eqref{temp96} yields that
    \begin{equation} \label{temp98}
        \begin{aligned}
            \be\l[\l\langle\phi(x),f\r\rangle_{\H_\K}^4\r]
        \leq& \l(c+2c^{3/4}\r)\l(\be\l[\l\langle\phi(x)-\overline\phi\r\rangle^2_{\H_\K}\r]\r)^2+\l\langle \overline\phi,f \r\rangle_{\H_\K}^4
        \\&+\l(6+2c^{3/4}\r)\l\langle \overline\phi, f \r\rangle_{\H_\K}^2\be\l[\l\langle\phi(x)-\overline\phi,f\r\rangle^2_{\H_\K}\r].
        \end{aligned}
    \end{equation}
    Combining \eqref{temp98} with \eqref{temp95}, we conclude that
    \begin{equation*}
        \be\l[\l\langle\phi(x),f\r\rangle_{\H_\K}^4\r]
        \leq \max\l\{c+2c^{3/4},3+c^{3/4}\r\}\l(\be\l[\l\langle\phi(x),f\r\rangle_{\H_\K}^2\r]\r)^2,
    \end{equation*}
    which completes this proof.
\end{proof}

\begin{proof}[Proof of Proposition \ref{prop14}]
    Recall that $\Sigma:=\be\l[\l(\phi(x)-\overline\phi\r)\otimes\l(\phi(x)-\overline\phi\r)\r]$. Since $\Sigma$ is compact and self-adjoint, it admits the spectral decomposition:
    \[
    \Sigma=\sum_{k\geq1}\lambda_k\phi_k\otimes\phi_k.
    \]
    We claim that $\phi(x)-\overline\phi\in\overline{\mathrm{ran}\l(\Sigma\r)}$ almost surely. To prove this, for any $f\in ker\l(\Sigma\r)$, it holds that
    \[
    \be\l[\l\langle\phi(x)-\overline\phi,f\r\rangle_{\H_\K}^2\r]=\l\langle \Sigma f,f\r\rangle_{\H_\K}=0.
    \]
    This implies that $\phi(x) - \overline\phi \in \ker(\Sigma)^\perp = \overline{\mathrm{ran}(\Sigma)}$ almost surely. Therefore, we have
    \begin{equation*}
        \phi(x)-\overline\phi=\sum_{k\geq1}\sqrt{\lambda_k}\xi_k\phi_k,
    \end{equation*}
    where $\xi_k:=\frac{\l\langle\phi(x)-\overline\phi,\phi_k\r\rangle}{\sqrt{\lambda_k}}$, $\be[\xi_k]=0$ and $\be\l[\xi_k^2\r]=1$ for all $k\geq1$. 
    
    Now, Assume that $\{\xi_k\}_{k\geq1}$ consists of independent random variables. We will show that \eqref{a4'} holds if $\l\{\be\l[\xi_k^4\r]\r\}_{k\geq1}$ is uniformly bounded, i.e., 
    there exists a constant $C > 0$ such that $\be\left[\xi_k^4\right] \leq C$ for all $k \geq 1$. 
    Since the $\xi_k$ are mean-zero and independent, it follows that
\[
\begin{aligned}
\mathbb E\!\left[
\left\langle \phi(x)-\overline\phi,f\right\rangle_{\H_\K}^4
\right]
=
\sum_{k\ge1}\lambda_k^2
\langle f,\phi_k\rangle_{\H_\K}^4
\mathbb E[\xi_k^4]+
6\sum_{i<j}\lambda_i\lambda_j
\langle f,\phi_i\rangle_{\H_\K}^2
\langle f,\phi_j\rangle_{\H_\K}^2,
\end{aligned}
\]
and
\[
\begin{aligned}
\left(
\mathbb E\!\left[
\left\langle \phi(x)-\overline\phi,f\right\rangle_{\H_\K}^2
\right]
\right)^2
=
\sum_{k\ge1}\lambda_k^2
\langle f,\phi_k\rangle_{\H_\K}^4+
2\sum_{i<j}\lambda_i\lambda_j
\langle f,\phi_i\rangle_{\H_\K}^2
\langle f,\phi_j\rangle_{\H_\K}^2.
\end{aligned}
\]
    Using $\be\l[\xi_k^4\r]\leq C$, we obtain
    \begin{equation*}
        \be\l[\l\langle\phi(x)-\overline\phi,f\r\rangle_{\H_\K}^4\r]\leq \max\{C,3\}\l(\be\l[\l\langle\phi(x)-\overline\phi,f\r\rangle_{\H_\K}^2\r]\r)^2. 
    \end{equation*}
    By Lemma \ref{Auxiliary lemma}, there exists a constant $c>0$, such that
    \[
        \be\l[\l\langle\phi(x),f\r\rangle_{\H_\K}^4\r]\leq c\l(\be\l[\l\langle\phi(x),f\r\rangle_{\H_\K}^2\r]\r)^2 .
    \]
    The proof is then complete.
\end{proof}

\section{An Auxiliary Equivalence Related to Assumption \ref{a4}}\label{Appendix lemma}
The following proposition provides an equivalent formulation related to Assumption \ref{a4}. Its proof has essentially been established in our previous work \citep{shi2024learning}. We include it here for completeness and to keep the present paper reasonably self-contained.

\begin{proposition}\label{lemma in Appendix}
The following three statements are equivalent.
\begin{enumerate}
    \item[(1)] There exists a constant $c>0$ such that for any $f\in\H_{\K}$,
    \begin{equation*}
        \mathbb{E}\left[\left\langle v,f\right\rangle_{\H_{\K}}^{4}\right]
        \leq
        c\left(\mathbb{E}\left[\left\langle v,f\right\rangle_{\H_{\K}}^{2}\right]\right)^{2}.
    \end{equation*}

    \item[(2)] For any separable Hilbert space $\Y$, there exists a constant $c>0$ such that for any bounded linear operator $A:\H_{\K}\to\Y$, there holds
    \begin{equation*}
        \mathbb{E}\left[\left\|Av\right\|_{\Y}^{4}\right]
        \leq
        c\left(\mathbb{E}\left[\left\|Av\right\|_{\Y}^{2}\right]\right)^{2}.
    \end{equation*}

    \item[(3)] There exists a separable Hilbert space $\Y$ and a constant $c>0$ such that for any bounded linear operator $A:\H_{\K}\to\Y$, there holds
    \begin{equation*}
        \mathbb{E}\left[\left\|Av\right\|_{\Y}^{4}\right]
        \leq
        c\left(\mathbb{E}\left[\left\|Av\right\|_{\Y}^{2}\right]\right)^{2}.
    \end{equation*}
\end{enumerate}
\end{proposition}

\begin{proof}
\textbf{(1) implies (2):}
Let $A:\H_{\K}\to\Y$ be any bounded linear operator.
Let $\{e_j\}_{j\geq1}$ be an orthonormal basis of $\Y$.
By Parseval's identity,
\[
\|Av\|_{\Y}^{2}
=
\sum_{j\geq1}
\left|\langle Av,e_j\rangle_{\Y}\right|^{2}
=
\sum_{j\geq1}
\left|\langle v,A^{*}e_j\rangle_{\H_{\K}}\right|^{2}.
\]
and thus
\begin{equation*}
\begin{aligned}
\mathbb{E}\left[\|Av\|_{\Y}^{4}\right]
&=
\mathbb{E}\left[
\left(
\sum_{j\geq1}
\l|\langle v,A^{*}e_j\rangle_{\H_{\K}}\r|^{2}
\right)^{2}
\right] \\
&=
\sum_{i,j\geq1}
\mathbb{E}\left[
\l|\langle v,A^{*}e_i\rangle_{\H_{\K}}\r|^{2}
\l|\langle v,A^{*}e_j\rangle_{\H_{\K}}\r|^{2}
\right] \\
&\leq
\sum_{i,j\geq1}
\sqrt{
\mathbb{E}\left[
\l|\langle v,A^{*}e_i\rangle_{\H_{\K}}\r|^{4}
\right]}
\sqrt{
\mathbb{E}\left[
\l|\langle v,A^{*}e_j\rangle_{\H_{\K}}\r|^{4}
\right]},
\end{aligned}
\end{equation*}
where we used the Cauchy--Schwarz inequality. By (1), it follows that
\[
\mathbb{E}\left[\|Av\|_{\Y}^{4}\right]
\leq
c\left(
\sum_{j\geq1}
\mathbb{E}\left[
\l|\langle v,A^{*}e_j\rangle_{\H_{\K}}\r|^{2}
\right]
\right)^{2}.
\]
Since, by Parseval's identity,
\[
\sum_{j\geq1}
\mathbb{E}\left[
\l|\langle v,A^{*}e_j\rangle_{\H_{\K}}\r|^{2}
\right]
=
\mathbb{E}\left[\|Av\|_{\Y}^{2}\right],
\]
we conclude that
\[
\mathbb{E}\left[\|Av\|_{\Y}^{4}\right]
\leq
c\left(\mathbb{E}\left[\|Av\|_{\Y}^{2}\right]\right)^{2}.
\]

\medskip
\noindent\textbf{(2) implies (3):}
This is immediate.

\medskip
\noindent\textbf{(3) implies (1):}
Let $\Y$ and $c>0$ be as in (3). Choose any $g\in\Y$ with $\|g\|_{\Y}=1$. For a fixed $f\in\H_{\K}$, define the rank-one operator $A:\H_{\K}\to\Y$ by
\[
Ah:=\langle h,f\rangle_{\H_{\K}}\,g,\qquad \forall\, h\in\H_{\K}.
\]
Then $A$ is bounded. Moreover,
\[
\|Av\|_{\Y}^{2}
=
|\langle v,f\rangle_{\H_{\K}}|^{2}\|g\|_{\Y}^{2}
=
|\langle v,f\rangle_{\H_{\K}}|^{2}.
\]
Applying (3) to this operator $A$, we obtain
\[
\mathbb{E}\left[\langle v,f\rangle_{\H_{\K}}^{4}\right]
=
\mathbb{E}\left[\|Av\|_{\Y}^{4}\right]
\leq
c\left(\mathbb{E}\left[\|Av\|_{\Y}^{2}\right]\right)^{2}
=
c\left(\mathbb{E}\left[\langle v,f\rangle_{\H_{\K}}^{2}\right]\right)^{2}.
\]
This proves (1).
\end{proof}
    
    \section{Proof of Proposition \ref{prop1}} \label{Appendix 1}
	In this subsection, our goal is to bound 
	\begin{align*}
		\sum_{t=1}^{T}\exp\left\{-\sum_{j=t+1}^{T}\eta_{j}\lambda_j\right\}\frac{(t+t_0)^{-\theta}}{1+\left(\sum_{j=t+1}^{T}\eta_{j}\right)^{v}}.
	\end{align*}

	\begin{lemma} \label{lemma4}
		Let $v>0$, $p\in\mathbb{R}$, $T\geq t_0+1$ and $t_0\geq1$. The step size $\eta_t$ is set as \eqref{setting}. Then, there holds 
		\begin{equation*}
			\sum_{t=1}^{T/2}\frac{(t+t_0)^{p}}{1+\left(\sum_{j=t+1}^{T}\eta_{j}\right)^{v}}
            \leq\delta^\prime
			\begin{cases}
				(T+t_0)^{-(1-\theta_1)v+p+1}, &\text{ when }p>-1, \\
				(T+t_0)^{-(1-\theta_1)v}\log(T+t_0), &\text{ when }p=-1, \\
				(T+t_0)^{-(1-\theta_1)v}, &\text{ when }p<-1,
			\end{cases}
		\end{equation*}
            where $\delta^\prime$ is a constant independent of $T$ and $t_0$.
	\end{lemma}
	\begin{proof}
            Since $\sum_{j=t+1}^{T}\eta_{j}\geq\bar\eta\left[(T+t_0+1)^{1-\theta_1}-(t+t_0+1)^{1-\theta_1}\right]$, it follows that 
		\begin{align*}
			\sum_{t=1}^{T/2}\frac{(t+t_0)^{p}}{1+\left(\sum_{j=t+1}^{T}\eta_{j}\right)^{v}}
			\leq\frac{1}{\min\{1,{\bar\eta}^v\}}\sum_{t=1}^{T/2}\frac{(t+t_0)^{p}}{1+\left[(T+t_0+1)^{1-\theta_1}-(t+t_0+1)^{1-\theta_1}\right]^v}.
		\end{align*}
            As $t+t_0+1\leq\frac34(T+t_0+1)$ when $t\leq T/2$ and $T\geq t_0+1$, we obtain
            \begin{equation*}
                \begin{aligned}
                    \sum_{t=1}^{T/2}\frac{(t+t_0)^{p}}{1+\left(\sum_{j=t+1}^{T}\eta_{j}\right)^{v}}
                    &\leq
                    \frac{\left(1-(3/4)^{1-\theta_1}\right)^{-v}}{\min\{1,{\bar\eta}^v\}}
                    \sum_{t=1}^{T/2}\frac{(t+t_0)^{p}}{(T+t_0)^{(1-\theta_1)v}}
			\\& \leq\frac{\left(1-(3/4)^{1-\theta_1}\right)^{-v}}{\min\{1,{\bar\eta}^v\}}(T+t_0)^{-(1-\theta_1)v}\int_{0}^{\frac{T}{2}+1}(x+t_0)^{p}dx
			\\& \leq\frac{\left(1-(3/4)^{1-\theta_1}\right)^{-v}}{\min\{1,{\bar\eta}^v\}}(T+t_0)^{-(1-\theta_1)v}
			\begin{cases}
				\frac{1}{p+1}(T+t_0)^{p+1}, &\text{ when }p>-1, \\
				\log(T+t_0), &\text{ when }p=-1, \\
				\frac{t_0^{p+1}}{-1-p}\leq\frac{1}{-1-p}, &\text{ when }p<-1,
			\end{cases}
			\\& \leq \delta^\prime
			\begin{cases}
				(T+t_0)^{-(1-\theta_1)v+p+1}, &\text{ when }p>-1, \\
				(T+t_0)^{-(1-\theta_1)v}\log(T+t_0), &\text{ when }p=-1, \\
				(T+t_0)^{-(1-\theta_1)v}, &\text{ when }p<-1,
			\end{cases}
                \end{aligned}
            \end{equation*}
		where 
		\begin{equation*}
			\delta^\prime=
			\frac{\left(1-(3/4)^{1-\theta_1}\right)^{-v}}{\min\{1,{\bar\eta}^v\}}
			\begin{cases}
				\frac{1}{p+1}, &\text{ when }p>-1, \\
				1, &\text{ when }p=-1, \\
				\frac{1}{-1-p}, &\text{ when }p<-1,
			\end{cases}
		\end{equation*}
		which is independent of $T$ and $t_0$. 
        
        The proof is then finished.
	\end{proof}
	
	\begin{lemma} \label{lemma5}
            Let $v>0$, $p\in\mathbb{R}$, $T\geq t_0+1$ and $t_0\geq1$. The step size $\eta_t$ is set as \eqref{setting}. Then, there holds 
		\begin{equation*}
			\sum_{t=T/2}^{T}\frac{(t+t_0)^{p}}{1+\left(\sum_{j=t+1}^{T}\eta_{j}\right)^{v}}
            \leq
            \delta^{\prime\prime}
			\begin{cases}
				(T+t_0)^{p+\theta_1}, &\text{ when }v>1, \\
				(T+t_0)^{p+\theta_1}\log(T+t_0), &\text{ when }v=1, \\
				(T+t_0)^{p+1-v(1-\theta_1)}, &\text{ when }v<1,
			\end{cases}
		\end{equation*}
		where $\delta^{\prime\prime}$ is a constant independent of $T$ and $t_0$.

	\end{lemma}
	\begin{proof}
            It is obvious that
		\begin{equation} \label{temp14}
             \begin{split}
			&\sum_{t=T/2}^{T}\frac{(t+t_0)^{p}}{1+\left(\sum_{j=t+1}^{T}\eta_{j}\right)^{v}}\\
			&\quad \leq\frac{1}{\min\{1,{\bar\eta}^v\}}\sum_{t=T/2}^{T-1}\frac{(t+t_0)^{p}}{1+\left[(T+t_0+1)^{1-\theta_1}-(t+t_0+1)^{1-\theta_1}\right]^v}	+(T+t_0)^p.
            \end{split}
		\end{equation}
		Next, we bound $\sum_{t=T/2}^{T-1}\frac{(t+t_0)^{p}}{1+\left[(T+t_0+1)^{1-\theta_1}-(t+t_0+1)^{1-\theta_1}\right]^v}$.
		If $p>0$, it holds that
		\begin{equation*}
			\frac{(t+t_0)^{p}}{1+\left[(T+t_0+1)^{1-\theta_1}-(t+t_0+1)^{1-\theta_1}\right]^v}
			\leq\int_{t+1}^{t+2}\frac{(u+t_0)^{p}}{1+\left[(T+t_0+1)^{1-\theta_1}-(u+t_0)^{1-\theta_1}\right]^v}du.
		\end{equation*}
		If $p\leq0$, there holds $(t+t_0)^p\leq3^{-p}(t+t_0+2)^p$. Thus we have
		\begin{equation*}
			\frac{(t+t_0)^{p}}{1+\left[(T+t_0+1)^{1-\theta_1}-(t+t_0+1)^{1-\theta_1}\right]^v}
			\leq3^{-p}\int_{t+1}^{t+2}\frac{(u+t_0)^{p}}{1+\left[(T+t_0+1)^{1-\theta_1}-(u+t_0)^{1-\theta_1}\right]^v}du.
		\end{equation*}
		Therefore,
		\begin{equation*}
			\begin{aligned}
				&\sum_{t=T/2}^{T-1}\frac{(t+t_0)^{p}}{1+\left[(T+t_0+1)^{1-\theta_1}-(t+t_0+1)^{1-\theta_1}\right]^v}
				\\ &\leq\max\{3^{-p},1\}\int_{T/2+1}^{T+1}\frac{(u+t_0)^{p}}{1+\left[(T+t_0+1)^{1-\theta_1}-(u+t_0)^{1-\theta_1}\right]^v}du.
			\end{aligned}
		\end{equation*}
		Let $\xi=(T+t_0+1)^{1-\theta_1}-(u+t_0)^{1-\theta_1}$, then $d\xi=-(1-\theta_1)(u+t_0)^{-\theta_1}du$, then
		\begin{equation} \label{temp11}
			\begin{aligned}
				\sum_{t=T/2}^{T-1}&\frac{(t+t_0)^{p}}{1+\left[(T+t_0+1)^{1-\theta_1}-(t+t_0+1)^{1-\theta_1}\right]^v}
				\\ &\leq \frac{\max\{3^{-p},1\}}{1-\theta_1}\int_{0}^{(T+t_0+1)^{1-\theta_1}-(T/2+t_0+1)^{1-\theta_1}}\frac{(u+t_0)^{p+\theta_1}}{1+\xi^v}d\xi
				\\ &\leq
				\frac{\max\{3^{-p},1\}}{1-\theta_1}\int_{0}^{(T+t_0+1)^{1-\theta_1}(1-2^{\theta_1-1})}\frac{(u+t_0)^{p+\theta_1}}{1+\xi^v}d\xi.
			\end{aligned}
		\end{equation}
		Since $u+t_0\in\left[T/2+t_0+1,T+t_0+1\right]$, whenever $p+\theta_1>0$ or not, we have
		\begin{equation}\label{temp12}
			(u+t_0)^{p+\theta_1}\leq2^{\lvert p+\theta_1\rvert}(T+t_0)^{p+\theta_1}.
		\end{equation}
		Hence, substituting \eqref{temp12} to \eqref{temp11} yields that
		\begin{equation} \label{temp13}
			\begin{aligned}
				\sum_{t=T/2}^{T-1}&\frac{(t+t_0)^{p}}{1+\left[(T+t_0+1)^{1-\theta_1}-(t+t_0+1)^{1-\theta_1}\right]^v}
				\\ &\leq
				\frac{\max\{3^{-p},1\}}{1-\theta_1}2^{\lvert p+\theta_1\rvert}(T+t_0)^{p+\theta_1}
				\left(1+\int_{1}^{(T+t_0+1)^{1-\theta_1}(1-2^{\theta_1-1})}\frac{1}{\xi^v}d\xi\right)
				\\ &\leq
				\frac{\max\{3^{-p},1\}}{1-\theta_1}2^{\lvert p+\theta_1\rvert}(T+t_0)^{p+\theta_1}
				\begin{cases}
					\frac{v}{v-1}, &\text{ when }v>1, \\
					\left(2-\theta_1\right)\log(T+t_0+1), &\text{ when }v=1, \\
					\frac{\left(1-2^{\theta_1-1}\right)^{1-v}}{1-v}(T+t_0+1)^{(1-\theta_1)(1-v)}, &\text{ when }v<1,
				\end{cases}
				\\ &=
				\widetilde{\delta^{\prime\prime}}
				\begin{cases}
					(T+t_0)^{p+\theta_1}, &\text{ when }v>1, \\
					(T+t_0)^{p+\theta_1}\log(T+t_0), &\text{ when } v=1, \\
					(T+t_0)^{p+1-v(1-\theta_1)}, &\text{ when } v<1,
				\end{cases}
			\end{aligned}
		\end{equation}
		with some constant $\widetilde{\delta^{\prime\prime}}$ independent of $T$ and $t_0$. Combining \eqref{temp13} with \eqref{temp14} yields that there exists a constant $\delta^{\prime\prime}$ independent of $T$ and $t_0$, such that 
		\begin{equation}
			\sum_{t=T/2}^{T}\frac{(t+t_0)^{p}}{1+\left(\sum_{j=t+1}^{T}\eta_{j}\right)^{v}}
			\leq\delta^{\prime\prime}
			\begin{cases}
				(T+t_0)^{p+\theta_1}, &\text{ when } v>1, \\
				(T+t_0)^{p+\theta_1}\log(T+t_0), &\text{ when } v=1, \\
				(T+t_0)^{p+1-v(1-\theta_1)}, &\text{ when } v<1,
			\end{cases}
		\end{equation}
        which completes the proof.
	\end{proof}

	\begin{proposition}
		Let $v>0$, $\theta\in\mathbb{R}$, $t_0\geq1$ and $T\geq t_0+1$. The step size $\eta_t$ is set as \eqref{setting}. Suppose that $\bar{\eta}\bar{\lambda}>\theta-1$ and  $\theta_1+\theta_2=1$. Then, there holds
		\begin{align*}
			\sum_{t=1}^{T}\exp\left\{-\sum_{j=t+1}^{T}\eta_{j}\lambda_j\right\}\frac{(t+t_0)^{-\theta}}{1+\left(\sum_{j=t+1}^{T}\eta_{j}\right)^{v}}
			\leq\delta_1
			\begin{cases}
				(T+t_0)^{-\theta+\theta_1}, &\text{ when } v>1, \\
				(T+t_0)^{-\theta+\theta_1}\log(T+t_0), &\text{ when } v=1, \\
				(T+t_0)^{-\theta+1-v(1-\theta_1)}, &\text{ when } v<1,
			\end{cases}
		\end{align*}
		where $\delta_1$ is a constant independent of $T$ and $t_0$.
	\end{proposition}
	\begin{proof}
            From Lemma \ref{lemma2} (2), we have
		\begin{align*}
			\sum_{t=1}^{T}&\exp\left\{-\sum_{j=t+1}^{T}\eta_{j}\lambda_j\right\}\frac{(t+t_0)^{-\theta}}{1+\left(\sum_{j=t+1}^{T}\eta_{j}\right)^{v}}
			\leq\sum_{t=1}^{T}\left(\frac{t+t_0+1}{T+t_0+1}\right)^{\bar{\eta}\bar{\lambda}}\frac{(t+t_0)^{-\theta}}{1+\left(\sum_{j=t+1}^{T}\eta_{j}\right)^{v}}
			\\ &\leq2^{\bar{\eta}\bar{\lambda}}(T+t_0)^{-\bar{\eta}\bar{\lambda}}\sum_{t=1}^{T}\frac{(t+t_0)^{\bar{\eta}\bar{\lambda}-\theta}}{1+\left(\sum_{j=t+1}^{T}\eta_{j}\right)^{v}}.
		\end{align*}
            Using Lemma \ref{lemma4} and Lemma \ref{lemma5}, we obtain
            \begin{equation*}
                \begin{aligned}
                    &\sum_{t=1}^{T}\exp\left\{-\sum_{j=t+1}^{T}\eta_{j}\lambda_j\right\}\frac{(t+t_0)^{-\theta}}{1+\left(\sum_{j=t+1}^{T}\eta_{j}\right)^{v}}
                    \\
                    \leq&2^{\bar{\eta}\bar{\lambda}}(T+t_0)^{-\bar{\eta}\bar{\lambda}}
			\left(\delta^{\prime}(T+t_0)^{-(1-\theta_1)v+p+1}+\delta^{\prime\prime}
			\begin{cases}
				(T+t_0)^{p+\theta_1}, &\text{ when } v>1, \\
				(T+t_0)^{p+\theta_1}\log(T+t_0), &\text{ when } v=1, \\
				(T+t_0)^{p+1-v(1-\theta_1)}, &\text{ when } v<1,
			\end{cases}
			\right)
                \end{aligned}
            \end{equation*}
		where $p=\bar{\eta}\bar{\lambda}-\theta>-1$. Therefore,
		\begin{equation}
			\begin{aligned}
					&\sum_{t=1}^{T}\exp\left\{-\sum_{j=t+1}^{T}\eta_{j}\lambda_j\right\}\frac{(t+t_0)^{-\theta}}{1+\left(\sum_{j=t+1}^{T}\eta_{j}\right)^{v}}
					\\ \leq&
					2^{\bar{\eta}\bar{\lambda}}(\delta^{\prime}+\delta^{\prime\prime})
					\begin{cases}
						(T+t_0)^{-\theta+\theta_1}, &\text{ when } v>1, \\
						(T+t_0)^{-\theta+\theta_1}\log(T+t_0), &\text{ when } v=1, \\
						(T+t_0)^{-\theta+1-v(1-\theta_1)}, &\text{ when } v<1.
					\end{cases}
			\end{aligned}
		\end{equation}
            We finish the proof by setting $\delta_1=2^{\bar{\eta}\bar{\lambda}}(\delta^{\prime}+\delta^{\prime\prime})$, which is independent of $T$ and $t_0$.	
	\end{proof}

\section{Proofs in Section \ref{sec:pde-application}}
\label{proofs in sec:pde-application}
\subsection{Proofs for arcsine sampling}
\label{Proofs for arcsine sampling}
\begin{proof}[Proof of \eqref{pde:kernel-evaluation}]
For $q\in(0,1)$ and $t\in\mathbb R$, the geometric series gives
\[
 \begin{aligned}
  1+2\sum_{n=1}^\infty q^n\cos(nt)
  &=1+2\operatorname{Re}
       \frac{qe^{\mathrm{i}t}}{1-qe^{\mathrm{i}t}}\\
  &=\frac{1-q^2}{1-2q\cos t+q^2}
   =P_q(t).
 \end{aligned}
\]
For $\vartheta,\vartheta'\in[0,\pi]$, the definitions of $p_n$
and $T_n$ therefore imply
\[
 \begin{aligned}
  \sum_{n=0}^\infty
  q^n p_n(\cos\vartheta)p_n(\cos\vartheta')
  &=1+2\sum_{n=1}^\infty
       q^n\cos(n\vartheta)\cos(n\vartheta')\\
  &=1+\sum_{n=1}^\infty q^n
       \bigl(
       \cos(n(\vartheta-\vartheta'))
       +\cos(n(\vartheta+\vartheta'))
       \bigr)\\
  &=\frac{
       P_q(\vartheta-\vartheta')
       +P_q(\vartheta+\vartheta')}{2}.
 \end{aligned}
\]
All these series converge absolutely and uniformly in the angular
variables.

Now let $q_j=\rho_j^{-1/2}$,
$\vartheta_j=\arccos\zeta_j$, and
$\vartheta_j'=\arccos\zeta_j'$.
Since $p_0=1$ and $|p_n|\le\sqrt2$ on $[-1,1]$ for $n\ge1$,
\[
 \sum_{\nu\in\mathbb N_0^N}
 \lambda_\nu
 |p_\nu(\zeta)p_\nu(\zeta')|
 \le
 \prod_{j=1}^N
 \left(1+2\sum_{n=1}^\infty q_j^n\right)
 <\infty.
\]
Absolute convergence allows us to factor the kernel series as
\[
 \begin{aligned}
  \K_N(\zeta,\zeta')
  &=\prod_{j=1}^N
    \left(
      \sum_{n=0}^\infty
      q_j^n p_n(\zeta_j)p_n(\zeta_j')
    \right)\\
  &=\prod_{j=1}^N
    \frac{
      P_{q_j}(\vartheta_j-\vartheta_j')
      +P_{q_j}(\vartheta_j+\vartheta_j')}{2}.
 \end{aligned}
\]
This proves \eqref{pde:kernel-evaluation}.
\end{proof}

Hereafter, for every $z\in\mathbb C$, we write $\Re z$ to denote $\operatorname{Re} z$.
\begin{proof}[Proof of Proposition \ref{pde:verification}]
Let $E_{\rho_j}$ be the closed region bounded by the Bernstein ellipse
with real and imaginary semiaxes $(\rho_j+\rho_j^{-1})/2$ and
$(\rho_j-\rho_j^{-1})/2$, respectively
\citep[equations~(2.110)--(2.111)]{cohen2015approximation}.
For $z\in\prod_j E_{\rho_j}$, define
\[
\zeta_j^*=\max\{-1,\min\{1,\Re z_j\}\}.
\]
Then $\zeta^*\in\Omega_\zeta$. Since $z_j\in E_{\rho_j}$,
\[
\left|\Re z_j-\zeta_j^*\right|
\leq
\frac{\rho_j+\rho_j^{-1}}{2}-1.
\]
Since the coefficient functions $a_j$ are real-valued,
\[
\begin{aligned}
\Re a(x,z)
&=
a(x,\zeta^*)
+\sum_{j=1}^N(\Re z_j-\zeta_j^*)a_j(x)\\
&\geq
a(x,\zeta^*)
-\sum_{j=1}^N
\left(
\frac{\rho_j+\rho_j^{-1}}{2}-1
\right)B_j
\geq \frac{a_{\min}}2,
\end{aligned}
\]
where the last inequality follows from
\eqref{pde:radius-budget}.
Let $\Y_{\mathbb C}:=H_0^1(\Omega_x;\mathbb C)$ denote the
complexification of $\Y$. For each
$z\in\prod_j E_{\rho_j}$, define the sesquilinear form
\[
\mathfrak a_z(u,v)
:=
\int_{\Omega_x}
a(x,z)\,\nabla u(x)\cdot\overline{\nabla v(x)}\,dx,
\qquad u,v\in\Y_{\mathbb C}.
\]
Since $a(\cdot,z)\in L^\infty(\Omega_x)$, the form
$\mathfrak a_z$ is bounded. Moreover,
\[
\begin{aligned}
|\mathfrak a_z(v,v)|
&\ge \Re \mathfrak a_z(v,v)\\
&=
\int_{\Omega_x}\Re a(x,z)\,|\nabla v(x)|^2\,dx\\
&\ge
\frac{a_{\min}}2\|v\|_{\Y_{\mathbb C}}^2 .
\end{aligned}
\]
The complex
Lax--Milgram theorem
\citep[Theorem~2.1]{cohen2015approximation} gives,
for every $z\in\prod_jE_{\rho_j}$, a unique
$h^\dagger(z)\in\Y_{\mathbb C}$ satisfying
\begin{equation}\label{pde:complex-weak-form}
\mathfrak a_z(h^\dagger(z),v)
=
\langle f,v\rangle,
\qquad
v\in\Y_{\mathbb C},
\end{equation}
and
\[
\|h^\dagger(z)\|_{\Y_{\mathbb C}}
\le
\frac{2\|f\|_{H^{-1}(\Omega_x)}}{a_{\min}}
=M.
\]
Consequently,
\[
\sup_{z\in\prod_jE_{\rho_j}}
\|h^\dagger(z)\|_{\Y_{\mathbb C}}
\le M.
\]

Since the coercivity bound above is strict and
$\prod_jE_{\rho_j}$ is compact, it remains valid, with a
possibly smaller positive constant, on some open neighborhood
of $\prod_jE_{\rho_j}$. 
Since $z\mapsto\mathfrak a_z$ depends affinely, and hence
holomorphically, on $z$, while the linear functional on the
right-hand side of \eqref{pde:complex-weak-form} is independent
of $z$, \citet[Corollary~2.4]{cohen2015approximation} show that
$h^\dagger$ is holomorphic on this neighborhood.

For the Bochner coefficients
$u_\nu=\int h^\dagger(\zeta)p_\nu(\zeta)\,d\mu_N(\zeta)$,
set
\[
 \Phi(w)=h^\dagger\!\left(
 \left(\frac{w_j+w_j^{-1}}2\right)_{j=1}^N\right).
\]
This function is holomorphic on a neighborhood of the product of
annuli $\rho_j^{-1}\le |w_j|\le\rho_j$.
Let $b_\nu$ be its Laurent coefficient with index $\nu$.
Using the product structure
$p_\nu(\zeta)=\prod_{j=1}^N p_{\nu_j}(\zeta_j)$ together with
\[
p_n\left(\frac{w+w^{-1}}{2}\right)
=
\begin{cases}
1, & n=0,\\[1mm]
2^{-1/2}(w^n+w^{-n}), & n\geq1,
\end{cases}
\]
the corresponding Laurent coefficient satisfies
\[
b_\nu
=
2^{-|\operatorname{supp}\nu|/2}u_\nu,
\qquad
\operatorname{supp}\nu:=\{j:\nu_j>0\}.
\]
By the multivariate Cauchy formula,
\[
b_\nu
=
\frac{1}{(2\pi i)^N}
\int_{|w_1|=\rho_1}\cdots\int_{|w_N|=\rho_N}
\frac{\Phi(w)}
{\prod_{j=1}^N w_j^{\nu_j+1}}
\,dw_1\cdots dw_N .
\]
Since $\|\Phi(w)\|_{\Y_{\mathbb C}}\le M$ on
$\prod_j\{|w_j|=\rho_j\}$, Cauchy's estimate gives
\[
\|b_\nu\|_{\Y_{\mathbb C}}
\le M\prod_{j=1}^N\rho_j^{-\nu_j}.
\]
Hence
\begin{equation}\label{pde:coefficient-bound}
 \|u_\nu\|_\Y
 \le M\,2^{|\operatorname{supp}\nu|/2}
       \prod_{j=1}^N\rho_j^{-\nu_j}
 =M\,2^{|\operatorname{supp}\nu|/2}\lambda_\nu^2.
\end{equation}

Since $|p_\nu|\le2^{|\operatorname{supp}\nu|/2}$,
the kernel series converges absolutely and uniformly and defines a
continuous positive definite kernel.
By the product structure and the absolute convergence established above,
for any $\zeta\in\Omega_\zeta$,
\[
\begin{aligned}
K_N(\zeta,\zeta)
&=
\prod_{j=1}^N
\left(
1+\sum_{n=1}^\infty
\rho_j^{-n/2}p_n(\zeta_j)^2
\right)\\
&\leq
\prod_{j=1}^N
\left(
1+2\sum_{n=1}^\infty\rho_j^{-n/2}
\right)\\
&=
\prod_{j=1}^N
\frac{1+\rho_j^{-1/2}}
     {1-\rho_j^{-1/2}},
\end{aligned}
\]
where the inequality uses $|p_n(t)|\leq\sqrt{2}$ for $n\geq1$.
Since $p_n(1)^2=2$ for every $n\geq1$, equality is attained at
$\zeta=(1,\ldots,1)$. Therefore,
\[
\sup_{\zeta\in\Omega_\zeta}K_N(\zeta,\zeta)
=
\prod_{j=1}^N
\frac{1+\rho_j^{-1/2}}
     {1-\rho_j^{-1/2}},
\]
which proves the first identity in
\eqref{pde:structural-bounds}.

The functions $\psi_\nu=\sqrt{\lambda_\nu}p_\nu$ form an orthonormal
basis of $\H_{\K_N}$ and satisfy
$C_N\psi_\nu=\lambda_\nu\psi_\nu$.
Summing $\lambda_\nu^s$ proves the trace identity.
For $0<r<3/2$, \eqref{pde:coefficient-bound} also gives
\begin{equation}\label{pde:source-bound}
 \sum_\nu\frac{\|u_\nu\|_\Y^2}{\lambda_\nu^{2r+1}}
 \le M^2\sum_\nu 2^{|\operatorname{supp}\nu|}
            \prod_{j=1}^N\rho_j^{-(3/2-r)\nu_j}
 =R_{N,r}^2<\infty.
\end{equation}
Thus $S_r^\dagger\psi_\nu=u_\nu/\lambda_\nu^{r+1/2}$ defines a
Hilbert--Schmidt operator.
Since $C_N^r$ is bounded, $H^\dagger=S_r^\dagger C_N^r$ is also
Hilbert--Schmidt and satisfies
$H^\dagger\psi_\nu=u_\nu/\sqrt{\lambda_\nu}$.
Moreover,
\[
\phi_N(\zeta)
=
\sum_\nu \sqrt{\lambda_\nu}\,
p_\nu(\zeta)\psi_\nu
\quad\text{in }\H_{K_N},
\]
and hence
$H^\dagger\phi_N(\zeta)
=
\sum_\nu u_\nu p_\nu(\zeta).$ Thus its associated RKHS function has the same
$L^2(\mu_N;\Y)$ coefficients as $h^\dagger$. Thus, the two functions agree $\mu_N$-almost everywhere. Since both are
continuous on $\Omega_\zeta$ and $\mu_N$ has full support, they must agree
everywhere on $\Omega_\zeta$. 
Finally, $\|y\|_\Y\le M+\sigma_b$ almost surely.
\end{proof}

\begin{proof}[Proof of Corollary \ref{pde:rates}]
Apply Theorems~\ref{thm5} and~\ref{thm8} with $r=1$ and
$\bar\eta=\bar\lambda=1$.
For both choices, $2/3\le\theta_1<3/4$ and
$\theta_2=1-\theta_1$.
Hence
$1>\max\{\theta_2,\theta_1,2\theta_1-1/2\}$, and
$(t_{0,N}+1)^{\theta_1}\ge\kappa_N^2+1$.
Independence and centering of the noise give
$\E(h)-\E(h^\dagger)=\|h-h^\dagger\|_{L^2(\mu_N;\Y)}^2$.
\end{proof}

\subsection{Uniform sampling on the parameter domain}
\label{pde:uniform}

We retain the PDE, coefficient assumptions, output space $\Y$,
and radii $\rho_j$ from Section~\ref{sec:pde-application}.
Let $m_N$ be the uniform probability measure on
$\Omega_\zeta=[-1,1]^N$, given by
\[
 dm_N(\zeta)=2^{-N}\,d\zeta.
\]
The training samples are independent copies of
$(\zeta,h^\dagger(\zeta)+\epsilon)$, where $\zeta\sim m_N$
and $\epsilon$ is centered, independent of $\zeta$, and satisfies
$\|\epsilon\|_\Y\le\sigma_b$ almost surely.

Let $L_n$ be the Legendre polynomial of degree $n$, normalized
by $L_n(1)=1$, and set
\[
 \ell_n(t)=\sqrt{2n+1}\,L_n(t),
 \qquad
 \ell_\nu(\zeta)=\prod_{j=1}^N\ell_{\nu_j}(\zeta_j).
\]
Then $\{\ell_\nu\}_{\nu\in\mathbb N_0^N}$ is an orthonormal
basis of $L^2(m_N)$, and
\[
 \sup_{t\in[-1,1]}|L_n(t)|=1,
 \qquad
 \sup_{t\in[-1,1]}|\ell_n(t)|=\sqrt{2n+1}.
\]
For these normalization relations and bounds, see
\citet[equations~(3.23)--(3.28)]{cohen2015approximation}.

Using the same eigenvalues as in \eqref{pde:kernel}, define
\begin{equation}\label{pde:uniform-kernel}
 \lambda_\nu=\prod_{j=1}^N\rho_j^{-\nu_j/2},
 \qquad
 \K_N^{\mathrm U}(\zeta,\zeta')
 =
 \sum_{\nu\in\mathbb N_0^N}
 \lambda_\nu\ell_\nu(\zeta)\ell_\nu(\zeta'),
 \qquad
 K_N^{\mathrm U}=\K_N^{\mathrm U}I.
\end{equation}
Let $\H_{\K_N^{\mathrm U}}$ and $\H_N^{\mathrm U}$ denote the
corresponding scalar and vector-valued RKHSs. Define
\[
 \phi_N^{\mathrm U}(\zeta)
 =\K_N^{\mathrm U}(\zeta,\cdot),
 \qquad
 C_N^{\mathrm U}
 =
 \int_{\Omega_\zeta}
 \phi_N^{\mathrm U}(\zeta)\otimes\phi_N^{\mathrm U}(\zeta)
 \,dm_N(\zeta).
\]

\begin{proposition}\label{pde:uniform-verification}
The model with input distribution $m_N$ and kernel
$K_N^{\mathrm U}$ satisfies Assumptions~\ref{a1}, \ref{a2},
\ref{a3}, and~\ref{a5} for every $0<r<3/2$ and $0<s\le1$.
More precisely,
\begin{equation}\label{pde:uniform-structural-bounds}
 \kappa_{N,\mathrm U}^2
 :=
 \sup_{\zeta\in\Omega_\zeta}
 \K_N^{\mathrm U}(\zeta,\zeta)
 =
 \prod_{j=1}^N
 \frac{1+\rho_j^{-1/2}}{(1-\rho_j^{-1/2})^2},
 \qquad
 \mathrm{Tr}\bigl((C_N^{\mathrm U})^s\bigr)
 =
 \prod_{j=1}^N(1-\rho_j^{-s/2})^{-1}.
\end{equation}
Let $M$ be as in Proposition~\ref{pde:verification}, and set
\[
 Q_N=\prod_{j=1}^N
 \frac{1+\rho_j^{-1}}{1-\rho_j^{-1}}.
\]
The operator $H^{\dagger,\mathrm U}$ representing $h^\dagger$
satisfies $h^\dagger=H^{\dagger,\mathrm U}\phi_N^{\mathrm U}$ and $H^{\dagger,\mathrm U}
 =S_r^{\dagger,\mathrm U}(C_N^{\mathrm U})^r$,
where
$S_r^{\dagger,\mathrm U}
\in\B_{\mathrm{HS}}(\H_{\K_N^{\mathrm U}},\Y)$ and
\begin{equation}\label{pde:uniform-source-radius}
 \|S_r^{\dagger,\mathrm U}\|_{\mathrm{HS}}^2
 \le (R_{N,r}^{\mathrm U})^2
 :=
 M^2Q_N^2
 \prod_{j=1}^N
 \frac{1+3\rho_j^{r-3/2}}{1-\rho_j^{r-3/2}}.
\end{equation}
Moreover, Assumption~\ref{a5} holds with $M_\rho=M+\sigma_b$.
\end{proposition}

\begin{proof}
Write $q_j=\rho_j^{-1/2}$.
The bounds on the Legendre polynomials imply
\[
 \begin{aligned}
 \sum_{\nu\in\mathbb N_0^N}
 \lambda_\nu
 |\ell_\nu(\zeta)\ell_\nu(\zeta')|
 &\le
 \prod_{j=1}^N
 \sum_{n=0}^\infty(2n+1)q_j^n\\
 &=
 \prod_{j=1}^N\frac{1+q_j}{(1-q_j)^2}
 <\infty.
 \end{aligned}
\]
Thus the kernel series converges absolutely and uniformly,
and defines a continuous positive definite kernel. Moreover,
for any $\zeta\in\Omega_\zeta$,
\[
\begin{aligned}
K_N^{\mathrm U}(\zeta,\zeta)
&\le
\prod_{j=1}^N
\sum_{n=0}^\infty (2n+1)q_j^n\\
&=
\prod_{j=1}^N
\frac{1+q_j}{(1-q_j)^2}.
\end{aligned}
\]
Since $\ell_n(1)^2=2n+1$, equality is attained at
$\zeta=(1,\ldots,1)$. Hence
\[
\sup_{\zeta\in\Omega_\zeta}
K_N^{\mathrm U}(\zeta,\zeta)
=
\prod_{j=1}^N
\frac{1+q_j}{(1-q_j)^2},
\]
which proves the formula for $\kappa_{N,\mathrm U}^2$.

The functions
$\psi_\nu^{\mathrm U}=\sqrt{\lambda_\nu}\ell_\nu$
form an orthonormal basis of $\H_{\K_N^{\mathrm U}}$.
Orthogonality in $L^2(m_N)$ gives
\[
 C_N^{\mathrm U}\psi_\nu^{\mathrm U}
 =\lambda_\nu\psi_\nu^{\mathrm U}.
\]
Therefore,
\[
\begin{aligned}
\operatorname{Tr}\big((C_N^{\mathrm U})^s\big)
&=
\sum_{\nu\in\mathbb N_0^N}\lambda_\nu^s\\
&=
\prod_{j=1}^N
\sum_{n=0}^\infty \rho_j^{-sn/2}
=
\prod_{j=1}^N
(1-\rho_j^{-s/2})^{-1},
\end{aligned}
\]
which proves the trace identity in
\eqref{pde:uniform-structural-bounds}.

We next verify Assumption~\ref{a2}. 
By the proof of Proposition~\ref{pde:verification},
\[
h^\dagger(\zeta)
=
\sum_{\gamma\in\mathbb N_0^N}
u_\gamma p_\gamma(\zeta),
\qquad
\zeta\in\Omega_\zeta.
\]
Although the training distribution is now $m_N$, this expansion
represents the same solution map and can therefore be used as an
auxiliary expansion. Moreover, by
\eqref{pde:coefficient-bound}, it converges absolutely and
uniformly on $\Omega_\zeta$, since
\[
 \sum_\gamma
 \|u_\gamma\|_\Y\|p_\gamma\|_\infty
 \le
 M\sum_\gamma
 2^{|\operatorname{supp}\gamma|}
 \prod_{j=1}^N\rho_j^{-\gamma_j}
 <\infty.
\]

Let $v_\nu
 =
 \int_{\Omega_\zeta}
 h^\dagger(\zeta)\ell_\nu(\zeta)\,dm_N(\zeta)$ be the Legendre coefficients of $h^\dagger$.
Uniform convergence permits integration of the Chebyshev
expansion term by term. Hence
\[
 v_\nu=\sum_\gamma u_\gamma d_{\gamma,\nu},
 \qquad
 d_{\gamma,\nu}
 =
 \int_{\Omega_\zeta}
 p_\gamma(\zeta)\ell_\nu(\zeta)\,dm_N(\zeta).
\]
By the product structure of $m_N$,
\[
d_{\gamma,\nu}
=
\prod_{j=1}^N
\int_{-1}^1
p_{\gamma_j}(t)\ell_{\nu_j}(t)\,\frac{dt}{2}.
\]
Since $\ell_{\nu_j}$ is orthogonal to every polynomial of
degree less than $\nu_j$, it follows that
$d_{\gamma,\nu}=0$ whenever $\gamma_j<\nu_j$ for at least one $j$. 
Moreover, the Cauchy--Schwarz inequality gives
\[
 |d_{\gamma,\nu}|
 \le
 \|p_\gamma\|_{L^2(m_N)}
 \|\ell_\nu\|_{L^2(m_N)}
 \le 2^{|\operatorname{supp}\gamma|/2}.
\]
Writing $\gamma\ge\nu$ when $\gamma_j\ge\nu_j$ for every $j$,
and separating the coordinates with $\nu_j=0$ from those with
$\nu_j>0$, we obtain
\begin{equation}\label{pde:uniform-coefficient-bound}
 \begin{aligned}
 \|v_\nu\|_\Y
 &\le
 M\sum_{\gamma\ge\nu}
 2^{|\operatorname{supp}\gamma|}
 \prod_{j=1}^N\rho_j^{-\gamma_j}\\
 &=
 M
 \prod_{\substack{1\le j\le N\\ \nu_j=0}}
 \frac{1+\rho_j^{-1}}{1-\rho_j^{-1}}
 \prod_{\substack{1\le j\le N\\ \nu_j>0}}
 \frac{2\rho_j^{-\nu_j}}{1-\rho_j^{-1}}\\
 &\le
 MQ_N\,2^{|\operatorname{supp}\nu|}
 \prod_{j=1}^N\rho_j^{-\nu_j}.
 \end{aligned}
\end{equation}
Consequently, for $0<r<3/2$,
\[
 \begin{aligned}
 \sum_\nu
 \frac{\|v_\nu\|_\Y^2}{\lambda_\nu^{2r+1}}
 &\le
 M^2Q_N^2
 \sum_\nu
 4^{|\operatorname{supp}\nu|}
 \prod_{j=1}^N\rho_j^{-(3/2-r)\nu_j}\\
 &=
 M^2Q_N^2
 \prod_{j=1}^N
 \left(
 1+4\sum_{n=1}^{\infty}
 \rho_j^{-(3/2-r)n}
 \right)\\
 &=
 M^2Q_N^2
 \prod_{j=1}^N
 \frac{1+3\rho_j^{r-3/2}}{1-\rho_j^{r-3/2}}
 =
 (R_{N,r}^{\mathrm U})^2
 <\infty.
 \end{aligned}
\]
Therefore,
\[
 S_r^{\dagger,\mathrm U}\psi_\nu^{\mathrm U}
 =
 \frac{v_\nu}{\lambda_\nu^{r+1/2}}
\]
defines a Hilbert--Schmidt operator.
Set
$H^{\dagger,\mathrm U}
=S_r^{\dagger,\mathrm U}(C_N^{\mathrm U})^r$.
This operator is also Hilbert--Schmidt and satisfies
\[
 H^{\dagger,\mathrm U}\psi_\nu^{\mathrm U}
 =
 \frac{v_\nu}{\sqrt{\lambda_\nu}}.
\]
Moreover,
\[
\phi_N^{\mathrm U}(\zeta)
=
\sum_\nu
\sqrt{\lambda_\nu}\,
\ell_\nu(\zeta)\psi_\nu^{\mathrm U}
\quad\text{in }\H_{\K_N^{\mathrm U}},
\]
and hence
\[
H^{\dagger,\mathrm U}\phi_N^{\mathrm U}(\zeta)
=
\sum_\nu v_\nu\ell_\nu(\zeta).
\]
Thus its associated RKHS function has the same
$L^2(m_N;\Y)$ coefficients as $h^\dagger$, so the two functions
agree $m_N$-almost everywhere. Since both are continuous on
$\Omega_\zeta$ and $m_N$ has full support, they agree everywhere
on $\Omega_\zeta$. 
Finally, the PDE bound and the noise assumption give
$\|y\|_\Y\le M+\sigma_b$ almost surely.
\end{proof}

\begin{corollary}\label{pde:uniform-rates}
Under the conditions of Proposition~\ref{pde:uniform-verification},
fix $s\in(0,1)$ and set
\[
 t_{0,N}^{\mathrm U}\geq(\kappa_{N,\mathrm U}^2+1)^{3/2}.
\]
For each parameter choice below, let $\{h_t\}_{t\ge1}$ be the
corresponding iteration~\eqref{iter2}, with $h_1=0$,
$K=K_N^{\mathrm U}$, and the samples specified above.
For every $T\ge1$ and $\delta\in(0,1/2)$, each bound holds with
probability at least $1-2\delta$:
\begin{enumerate}
\item[(1)] If
$\eta_t=(t+t_{0,N}^{\mathrm U})^{-2/3}$ and
$\lambda_t=(t+t_{0,N}^{\mathrm U})^{-1/3}$, then
\[
 \|h_{T+1}-h^\dagger\|_{L^2(m_N;\Y)}^2
 \le
 A_{N,s}^{\mathrm U}
 (T+t_{0,N}^{\mathrm U})^{-2/3}
 \log^4\frac2\delta.
\]
\item[(2)] If
$\eta_t=(t+t_{0,N}^{\mathrm U})^{-(2+s)/(3+s)}$ and
$\lambda_t=(t+t_{0,N}^{\mathrm U})^{-1/(3+s)}$, then
\[
 \|h_{T+1}-h^\dagger\|_{\H_N^{\mathrm U}}^2
 \le
 D_{N,s}^{\mathrm U}
 (T+t_{0,N}^{\mathrm U})^{-2/(3+s)}
 \log^4\frac2\delta.
\]
\end{enumerate}
The constants $A_{N,s}^{\mathrm U}$ and $D_{N,s}^{\mathrm U}$
can be chosen to depend only on
$t_{0,N}^{\mathrm U}$, $\kappa_{N,\mathrm U}$,
$\mathrm{Tr}((C_N^{\mathrm U})^s)$,
$R_{N,1}^{\mathrm U}$, $M$, $\sigma_b$, and $s$,
and are independent of $T$ and $\delta$.
\end{corollary}

\begin{proof}
Proposition~\ref{pde:uniform-verification} verifies the required
assumptions with $r=1$.
Apply Theorems~\ref{thm5} and~\ref{thm8} with
$\bar\eta=\bar\lambda=1$.
For both parameter choices,
$2/3\le\theta_1<3/4$ and $\theta_2=1-\theta_1$, so
\[
 1>\max\{\theta_2,\theta_1,2\theta_1-1/2\},
 \qquad
 (t_{0,N}^{\mathrm U}+1)^{\theta_1}
 \ge\kappa_{N,\mathrm U}^2+1.
\]
The prediction error identity follows from independence and
centering of the noise:
\[
 \E(h)-\E(h^\dagger)
 =
 \|h-h^\dagger\|_{L^2(m_N;\Y)}^2.
\]
 
This completes this proof.
\end{proof}

For numerical evaluation, the kernel in
\eqref{pde:uniform-kernel} can be written as
\[
 \K_N^{\mathrm U}(\zeta,\zeta')
 =
 \prod_{j=1}^N k_{q_j}(\zeta_j,\zeta_j'),
 \qquad
 k_q(t,t')
 =
 \sum_{n=0}^\infty
 (2n+1)q^nL_n(t)L_n(t'),
 \qquad
 q_j=\rho_j^{-1/2}.
\]
The Legendre polynomials can be evaluated using
$L_0(t)=1$, $L_1(t)=t$, and the recurrence
\[
 (n+1)L_{n+1}(t)
 =
 (2n+1)tL_n(t)-nL_{n-1}(t),
 \qquad n\ge1.
\]
Let $k_{q,J}$ be the sum defining $k_q$ truncated at degree
$J\ge0$. Since $|L_n|\le1$ on $[-1,1]$,
\[
 \sup_{t,t'\in[-1,1]}
 |k_q(t,t')-k_{q,J}(t,t')|
 \le
 \sum_{n=J+1}^\infty(2n+1)q^n
 =
 \frac{q^{J+1}\bigl(2J+3-(2J+1)q\bigr)}{(1-q)^2}
 =:e_J(q).
\]
Writing $b(q)=(1+q)/(1-q)^2$, both $k_q$ and $k_{q,J}$
have absolute value at most $b(q)$.
A telescoping expansion therefore gives
\[
 \sup_{\zeta,\zeta'\in\Omega_\zeta}
 \left|
 \K_N^{\mathrm U}(\zeta,\zeta')
 -
 \prod_{j=1}^N k_{q_j,J}(\zeta_j,\zeta_j')
 \right|
 \le
 \sum_{j=1}^N e_J(q_j)
 \prod_{\substack{i=1\\i\ne j}}^N b(q_i).
\]
These estimates provide explicit control of the error in kernel
evaluation. Corollary~\ref{pde:uniform-rates} is stated for the
full kernel $\K_N^{\mathrm U}$.

\begin{proposition}[Failure of the moment condition]
\label{pde:moment-failure}
Fix an integer $N\ge1$ and radii $\rho_j>1$ for
$j=1,\ldots,N$.
Let $\K_N$ and $\K_N^{\mathrm U}$ be the kernels defined in
\eqref{pde:kernel} and \eqref{pde:uniform-kernel}, respectively.
Then Assumption~\ref{a4} fails for both kernel--measure pairs
$(\K_N,\mu_N)$ and $(\K_N^{\mathrm U},m_N)$.
\end{proposition}

\begin{proof}
For each integer $m\ge1$, define
\[
 g_m(\zeta)=\left(\frac{1+\zeta_1}{2}\right)^m,
 \qquad \zeta\in\Omega_\zeta.
\]
Since $\lambda_\nu>0$ for every $\nu\in\mathbb N_0^N$,
and $\sqrt{\lambda_\nu}p_\nu$ and
$\sqrt{\lambda_\nu}\ell_\nu$ form orthonormal bases of
$\H_{\K_N}$ and $\H_{\K_N^{\mathrm U}}$, respectively,
both scalar RKHSs contain every polynomial on $\Omega_\zeta$.
In particular,
\[
 g_m\in\H_{\K_N}\cap\H_{\K_N^{\mathrm U}}
 \qquad\text{for every }m\ge1.
\]
Applying condition~\eqref{a4'} to $g_m$ would require
the ratios
\[
 \frac{\be_{\mu_N}[|g_m(\zeta)|^4]}
      {\bigl(\be_{\mu_N}[|g_m(\zeta)|^2]\bigr)^2}
 \quad\text{and}\quad
 \frac{\be_{m_N}[|g_m(\zeta)|^4]}
      {\bigl(\be_{m_N}[|g_m(\zeta)|^2]\bigr)^2}
\]
to be bounded uniformly in $m$ for the respective
kernel--measure pairs.

For arcsine sampling, write $\zeta_1=\cos\vartheta$,
where $\vartheta$ is uniform on $[0,\pi]$.
For $k\in\{2,4\}$, the product structure of $\mu_N$ gives
\[
 \begin{aligned}
 \be_{\mu_N}[|g_m(\zeta)|^k]
 &=
 \frac1\pi\int_0^\pi
 \left(\frac{1+\cos\vartheta}{2}\right)^{km}
 \,d\vartheta\\
 &=
 \frac2\pi\int_0^{\pi/2}\cos^{2km}t\,dt
 =
 4^{-km}\binom{2km}{km}.
 \end{aligned}
\]
Consequently, Stirling's formula yields
\[
 \frac{\be_{\mu_N}[|g_m(\zeta)|^4]}
      {\bigl(\be_{\mu_N}[|g_m(\zeta)|^2]\bigr)^2}
 =
 \frac{\binom{8m}{4m}}{\binom{4m}{2m}^{\,2}}
 \sim \sqrt{\pi m}
 \longrightarrow\infty
 \qquad\text{as }m\to\infty.
\]
Thus, \eqref{a4'} cannot hold with a finite constant $c$
for $(\K_N,\mu_N)$.

For uniform sampling, $\zeta_1$ is uniform on $[-1,1]$.
Hence, for $k\in\{2,4\}$,
\[
 \be_{m_N}[|g_m(\zeta)|^k]
 =
 \frac12\int_{-1}^1
 \left(\frac{1+t}{2}\right)^{km}\,dt
 =
 \frac{1}{km+1}.
\]
It follows that
\[
 \frac{\be_{m_N}[|g_m(\zeta)|^4]}
      {\bigl(\be_{m_N}[|g_m(\zeta)|^2]\bigr)^2}
 =
 \frac{(2m+1)^2}{4m+1}
 \longrightarrow\infty
 \qquad\text{as }m\to\infty.
\]
Therefore, \eqref{a4'} also fails for
$(\K_N^{\mathrm U},m_N)$.
The equivalence between \eqref{a4'} and
Assumption~\ref{a4} completes the proof.
\end{proof}

\subsection{Uniform bounds under coefficient decay}
\label{pde:decay}

\begin{proof}[Proof of the bounds under coefficient decay]
We use \eqref{pde:coefficient-decay} and
\eqref{pde:decay-radii}, with fixed $s\in(0,1)$,
$\beta>1+2/s$, and $2/s<\gamma<\beta-1$.
The condition on $\beta$ ensures that such a $\gamma$ exists.

First, the radius condition holds for every finite truncation.
Indeed,
\[
 \begin{aligned}
 \sum_{j=1}^N
 \left(\frac{\rho_j+\rho_j^{-1}}2-1\right)B_j
 &=
 \sum_{j=1}^N
 \frac{(\rho_j-1)^2}{2\rho_j}B_j\\
 &\le
 \frac{\tau B}{2}
 \sum_{j=1}^{\infty}j^{\gamma-\beta}
 \le \frac{a_{\min}}2.
 \end{aligned}
\]
Here the series converges because $\gamma<\beta-1$.

Set $q_j=\rho_j^{-1/2}$.
Since $\gamma s/2>1$,
\[
 \sum_{j=1}^{\infty}q_j^s
 \le
 \tau^{-s/2}
 \sum_{j=1}^{\infty}j^{-\gamma s/2}
 <\infty.
\]
Moreover, $0<q_j<1$ and $s<1$, so $q_j\le q_j^s$ and
$\sum_jq_j<\infty$.
The inequalities $\log(1+x)\le x$ for $x\ge0$ and
$-\log(1-x)\le2x$ for $0\le x\le1/2$ show that the
infinite products below are finite.
Only finitely many factors can have $q_j>1/2$ or $q_j^s>1/2$,
and each such factor is finite.

For the two kernels,
\[
 \kappa_N^2
 \le \kappa_{N,\mathrm U}^2
 \le
 \kappa_*^2
 :=
 \prod_{j=1}^{\infty}
 \frac{1+q_j}{(1-q_j)^2}
 <\infty.
\]
Their capacity quantities satisfy
\[
 \mathrm{Tr}(C_N^s)
 =
 \mathrm{Tr}((C_N^{\mathrm U})^s)
 \le
 \prod_{j=1}^{\infty}(1-q_j^s)^{-1}
 <\infty.
\]

We next establish uniform bounds for $\l\|S_1^\dagger\r\|_{\mathrm{HS}}$.
For the Chebyshev kernel,
\[
 R_{N,1}^2=M^2\kappa_N^2\le M^2\kappa_*^2.
\]
For the Legendre kernel, define
\[
 Q_*=
 \prod_{j=1}^{\infty}
 \frac{1+q_j^2}{1-q_j^2}
 <\infty.
\]
Since $Q_N\le Q_*$, the refined bound
\eqref{pde:uniform-source-radius} gives
\[
 (R_{N,1}^{\mathrm U})^2
 \le
 M^2Q_*^2
 \prod_{j=1}^{\infty}
 \frac{1+3q_j}{1-q_j}
 <\infty.
\]
Since $a_{\min}$ and $f$ are independent of $N$, the same is
true of $M$ and the output bound $M_\rho=M+\sigma_b$.

Choose
\[
 t_{0,N}=t_{0,N}^{\mathrm U}=t_0
 :=(\kappa_*^2+1)^{3/2}.
\]
This choice satisfies the conditions on $t_0$ in
Corollaries~\ref{pde:rates} and~\ref{pde:uniform-rates}.
Applying Theorems~\ref{thm5} and~\ref{thm8} with the bounds above
gives the stated estimates for each fixed $N$, with constants
independent of $N$, $T$, and $\delta$.
These constants may depend on $B$, $\beta$, $\gamma$, $\tau$,
$a_{\min}$, $\|f\|_{H^{-1}(\Omega_x)}$, $\sigma_b$, and $s$.
\end{proof}

\section{Extension to General Kernel Setting}
\label{Section 3.1}
We now discuss an extension of the regularized SGD analysis beyond the kernel class in Assumption~\ref{a1}. This choice of kernel has been employed in functional regression with structured output learning, as noted in Section \ref{introduction}. We now turn to extending the class of operator-valued kernels, further showing the generality and applicability of our analysis to a broader range of nonlinear operator learning problems.

We now consider an alternative setting for vector-valued RKHS and briefly list the conditions below.
\begin{enumerate}
    \item[(1)] Let $\H$ be separable, which is true if the spaces $\X$ and $\Y$ are separable and $K$ is a Mercer kernel \citep[Corollary 5.2]{carmeli2006vector}. A kernel $K$ is Mercer if and only if the RKHS induced by $K$ is a subspace of the space of continuous operators from $\X$ to $\Y$, which in turn holds if and only if $K$ is locally bounded and $K(x,\cdot)$ is strongly continuous for any $x\in\X$ \citep[Proposition 5.1]{carmeli2006vector}.

    \item[(2)] We assume that
the operator-valued kernel $K:\X\times\X\rightarrow\B(\Y)$ satisfies that $K(x,x)$ is compact for any $x\in\X$.

    \item[(3)] We further assume that $K$ is strongly measurable. Under assumptions (1) and (2), this is equivalent to requiring that each element in $\H$ is a measurable function \citep[Proposition 3.3]{carmeli2006vector}.
    
    \item[(4)] By Corollary 4.6 and Proposition 4.8 in \citet{carmeli2006vector}, if $K(x,x)$ is an operator of trace class (which implies compactness obviously) for almost all $x\in\X$ and 
\begin{equation}
   \be\l[\mathrm{Tr}\l(K(x,x)\r)\r] <\infty, 
\end{equation}
then the inclusion $\iota:\H\rightarrow L^2(\X,\Y)$ is well-defined and Hilbert-Schmidt. As a result, the operator $L_K:=\iota^*\iota$ is trace-class, which plays a role similar to that of $C$. And thus, we assume Assumption \ref{a3} holds for $L_K$ with $s\in (0,1]$, i.e., $\mathrm{Tr}(L_K^s) < +\infty$.
\end{enumerate}

The conditions listed above are required when conducting error analysis for kernels beyond the special case considered in Assumption \ref{a1}.  We remark that when the kernel is chosen as Eq. \eqref{kernel}, i.e., $K(x,x')=\K(x,x')W$ with $W$ being a self-adjoint and positive operator,
the above conditions are satisfied if $W$ is trace-class and the scalar-valued kernel $\K$ is a Mercer kernel such that $\sup_{x \in \X} \K(x,x) < \infty$.

We now outline the framework for generalizing the conclusions of this paper to the general scenario discussed above, while leaving the detailed proof to future work.  It is straightforward to observe that for any $h \in \H$, the operator $L_K$ satisfies  
\[L_Kh=\be\l[K(x,\cdot)h(x)\r].\]
Moreover, for any $h\in\H$, since the noise $\epsilon$ is centered and independent of $x$, there holds
\begin{align*}
    \E(h)-\E(h^\dagger)&=\be\l\|h(x)-h^\dagger(x)\r\|_{\Y}^2
    \\ &=\be\l[\l\langle K(x,\cdot)(h(x)-h^\dagger(x)),h-h^\dagger\r\rangle_{\H}\r]
    \\ &=\l\|L_K^{1/2}\l(h-h^\dagger\r)\r\|_{\H}^2.
\end{align*}
Our goal is to bound the prediction error $\l\|L_K^{1/2}\l(h-h^\dagger\r)\r\|_{\H}^2$ and estimation error $\l\|h-h^\dagger\r\|_{\H}^2$ for the regularized SGD estimator $h=h_{T+1}$. Similar to the approach in the proofs of our main results, we can derive analogs of equations \eqref{temp3}, \eqref{induction}, Proposition \ref{Proposition error 1}, and Proposition \ref{proposition error2}. Next, under assumptions similar to those in Section \ref{results}, we can carry out the error analysis, which we leave to future work.

We point out that, when $\Y$ is infinite-dimensional, the framework discussed in this subsection does not cover the case considered in Assumption \ref{a1}, as the kernel $\K(x,x)I$ is not a compact operator and thus fails to satisfy the assumptions required in the current setting. Hence, the framework developed here and the one based on Assumption \ref{a1} are complementary and non-overlapping in the infinite-dimensional-output setting. Nevertheless, combining the kernel choices from Assumption \ref{a1} and those introduced in this subsection can significantly broaden the applicability of our analysis developed in this paper.

\section{Connection with PCA Encoder-decoder Framework}
\label{Section 3.3}

In this appendix, we briefly discuss how the regularized SGD framework may be combined with a PCA-based encoder-decoder strategy. The purpose of this discussion is only to illustrate a possible implementation route in discretized settings; it is not part of the proved results of the paper. A full analysis would require additional estimates for empirical PCA and projection errors, which are beyond the scope of the present work.

Let $\l(\mathcal{X}, \langle\cdot\rangle_{\mathcal{X}}, \l\|\cdot\r\|_{\mathcal{X}}\r)$ and $\l(\mathcal{Y}, \langle\cdot\rangle_{\mathcal{Y}}, \l\|\cdot\r\|_{\mathcal{Y}}\r)$ be two real separable Hilbert spaces. Suppose that 
\begin{equation} \label{h}
    h^\dagger: \mathcal{X}\to\mathcal{Y}
\end{equation}
is a potentially nonlinear operator. Given i.i.d. samples $\l\{x_t, y_t\r\}_{t=1}^T\sim\rho$, where $y_t = h^\dagger(x_t) + \epsilon_t$, and $\epsilon_t$ denotes centered i.i.d. noise independent of $x_t$, our goal is to solve the prediction problem, i.e., to minimize the prediction error $\E(h) = \be_\rho\l[\|h(x)-y\|_{\Y}^2\r]$ for some estimator $h$. To this end, we apply the PCA technique to project the input and output samples onto finite-dimensional Euclidean spaces. We then approximate the mapping from the finite-dimensional input space to the finite-dimensional output space using kernel methods. Below, we briefly review the PCA technique.

The primary function of PCA is to extract the principal features of the data. High-dimensional data often suffers from the curse of dimensionality, and dimensionality reduction—achieved by identifying and retaining the most significant information—serves as an effective remedy. This constitutes the central role of PCA. In our setting, we employ PCA to reduce the samples from an infinite-dimensional space to a finite-dimensional one. The PCA algorithm applied to a random input $x$ in $\X$ seeks to minimize the reconstruction error $\be\l[\l\|(I-P)x\r\|_{\X}^2\r]$ over $\Pi_{d_\X}$, the set of all orthogonal projections $P$ with rank $d_\X$.  Given the covariance operator of $x$ defined as $\Sigma_x:=\be[x\otimes x]$, there exist eigenvalue-eigenvector pairs $\{\lambda_i^{d_\X}, \phi_i^{d_\X}\}_{i\geq1}$ satisfying $\langle\Sigma_x \phi_i^{d_\X}, \phi_j^{d_\X}\rangle_\X=\lambda_i^{d_\X}\delta_{ij}$ and $\lambda_1^{d_\X}\geq\lambda_2^{d_\X}\geq\cdots\geq0$, where $\delta_{ij}=1$ if $i=j$, otherwise $0$. It can be shown that the optimal PCA projection is given by
\begin{equation*}
    P_{d_\X}^{\X}=\mathop{\arg\min}_{P\in\Pi_{d_\X}}\be\l[\l\|(I-P)x\r\|_{\X}^2\r]=\mathcal{D}_{d_\X}^{\X}\circ\mathcal{E}_{d_\X}^{\X},
\end{equation*}
where the encoder $\mathcal{E}_{d_\X}^{\X}:\X\to\mathbb{R}^{d_\X}$ is defined as
\begin{equation*}
    \mathcal{E}_{d_\X}^{\X}(x):=\l(\l\langle x,\phi_i^\X\r\rangle_{\X}\r)_{i=1}^{d_\X},
\end{equation*}
and the decoder $\mathcal{D}_{d_\X}^{\X}:\mathbb{R}^{d_\X}\to\X$ is defined as 
\begin{equation*}
    \mathcal{D}_{d_\X}^{\X}(\eta):=\sum_{i=1}^{d_\X}\eta_i\phi_i^\X=\l(\mathcal{E}_{d_\X}^{\X}\r)^*(\eta).
\end{equation*}
It then follows that $\be\l[\l\|(I-P_{d_\X}^{\X})x\r\|_{\X}^2\r]=\sum_{i>d_\X}\lambda_i^{d_\X}$, see \citet[Theorem 3.8]{lanthaler2022error}.

In practice, it is usually difficult to obtain $\Sigma_x$ directly,  so we typically use the empirical covariance operator $\Sigma_x^{T}=\frac{1}{T}\sum_{i=1}^T x_i\otimes x_i$ as a substitute, thus deriving the empirical PCA. Following the same procedure, we naturally obtain the empirical encoder $\widehat{\mathcal{E}}_{d_\X}^{\X}$ and empirical decoder $\widehat{\mathcal{D}}_{d_\X}^{\X}$. Similarly, by replacing  $\Sigma_x^{T}$, the input random variable $x$, and the rank $d_\X$ with $\Sigma_y^T:=\frac{1}{T}\sum_{i=1}^T y_i\otimes y_i$, the output random variable $y$ and $d_\Y$, respectively, we  apply empirical PCA to $y$ in $\Y$ with rank $d_\Y$, and obtain the empirical encoder $\widehat{\mathcal{E}}_{d_\Y}^{\Y}$ and empirical decoder $\widehat{\mathcal{D}}_{d_\Y}^{\Y}$.

\begin{figure}
\centering

\tikzset{every picture/.style={line width=0.75pt}} 

\begin{tikzpicture}[x=0.75pt,y=0.75pt,yscale=-1,xscale=1]

\draw [color={rgb, 255:red, 208; green, 2; blue, 27 }  ,draw opacity=1 ]   (140,85) -- (253,85) ;
\draw [shift={(255,85)}, rotate = 180] [color={rgb, 255:red, 208; green, 2; blue, 27 }  ,draw opacity=1 ][line width=0.75]    (10.93,-3.29) .. controls (6.95,-1.4) and (3.31,-0.3) .. (0,0) .. controls (3.31,0.3) and (6.95,1.4) .. (10.93,3.29)   ;
\draw    (132.5,99) -- (133.47,161) ;
\draw [shift={(133.5,163)}, rotate = 269.1] [color={rgb, 255:red, 0; green, 0; blue, 0 }  ][line width=0.75]    (10.93,-3.29) .. controls (6.95,-1.4) and (3.31,-0.3) .. (0,0) .. controls (3.31,0.3) and (6.95,1.4) .. (10.93,3.29)   ;
\draw    (122.5,163) -- (121.53,101) ;
\draw [shift={(121.5,99)}, rotate = 89.1] [color={rgb, 255:red, 0; green, 0; blue, 0 }  ][line width=0.75]    (10.93,-3.29) .. controls (6.95,-1.4) and (3.31,-0.3) .. (0,0) .. controls (3.31,0.3) and (6.95,1.4) .. (10.93,3.29)   ;
\draw    (140,180) -- (253,180) ;
\draw [shift={(255,180)}, rotate = 180] [color={rgb, 255:red, 0; green, 0; blue, 0 }  ][line width=0.75]    (10.93,-3.29) .. controls (6.95,-1.4) and (3.31,-0.3) .. (0,0) .. controls (3.31,0.3) and (6.95,1.4) .. (10.93,3.29)   ;
\draw    (276.5,99) -- (277.47,161) ;
\draw [shift={(277.5,163)}, rotate = 269.1] [color={rgb, 255:red, 0; green, 0; blue, 0 }  ][line width=0.75]    (10.93,-3.29) .. controls (6.95,-1.4) and (3.31,-0.3) .. (0,0) .. controls (3.31,0.3) and (6.95,1.4) .. (10.93,3.29)   ;
\draw    (266.5,163) -- (266.5,101) ;
\draw [shift={(266.5,99)}, rotate = 90] [color={rgb, 255:red, 0; green, 0; blue, 0 }  ][line width=0.75]    (10.93,-3.29) .. controls (6.95,-1.4) and (3.31,-0.3) .. (0,0) .. controls (3.31,0.3) and (6.95,1.4) .. (10.93,3.29)   ;

\draw (120,76.4) node [anchor=north west][inner sep=0.75pt]    {$\mathcal{X}$};
\draw (264,77.4) node [anchor=north west][inner sep=0.75pt]    {$\mathcal{Y}$};
\draw (114,166.4) node [anchor=north west][inner sep=0.75pt]    {$\mathbb{R}^{d_{\mathcal{X}}}$};
\draw (261,169.4) node [anchor=north west][inner sep=0.75pt]    {$\mathbb{R}^{d_\mathcal{Y}}$};
\draw (190,63.4) node [anchor=north west][inner sep=0.75pt]    {$\textcolor[rgb]{0.82,0.01,0.11}{h^{\dagger }}$};
\draw (137,119.4) node [anchor=north west][inner sep=0.75pt]  [font=\normalsize]  {$\widehat{\mathcal{E}}_{d_{\mathcal{X}}}^{\mathcal{X}}$};
\draw (92,119.4) node [anchor=north west][inner sep=0.75pt]  [font=\normalsize]  {$\widehat{\mathcal{D}}_{d_{\mathcal{X}}}^{\mathcal{X}}$};
\draw (281,119.4) node [anchor=north west][inner sep=0.75pt]  [font=\normalsize]  {$\widehat{\mathcal{E}}_{d_{\mathcal{Y}}}^{\mathcal{Y}}$};
\draw (237,119.4) node [anchor=north west][inner sep=0.75pt]  [font=\normalsize]  {$\widehat{\mathcal{D}}_{d_{\mathcal{Y}}}^{\mathcal{Y}}$};
\draw (190,156.4) node [anchor=north west][inner sep=0.75pt]  [color={rgb, 255:red, 208; green, 2; blue, 27 }  ,opacity=1 ]  {$f^{\dagger }$};

\end{tikzpicture}
\caption{Schematic diagram of PCA encoder-decoder framework}

\label{fig2}

\end{figure}

We now formulate the estimator as $h:=\widehat{\mathcal{D}}_{d_\Y}^{\Y}\circ f\circ\widehat{\mathcal{E}}_{d_\X}^{\X}$. It is then natural to choose $f^\dagger=\widehat{\mathcal{E}}_{d_\Y}^{\Y}\circ h^\dagger\circ\widehat{\mathcal{D}}_{d_\X}^{\X}$. This formulation is illustrated schematically in Figure \ref{fig2}. \citet{bhattacharya2021model} and \citet{lanthaler2023operator} represent $f$ using neural networks. In contrast, we represent $f$ in an RKHS induced by a matrix-valued kernel $k:\mathbb{R}^{d_\X}\times\mathbb{R}^{d_\X}\to\B(\mathbb{R}^{d_\Y})$. Specifically, we consider kernels of the form $k(u, v)=\phi(\|u-v\|_{\mathbb{R}^{d_\X}})I_{\mathbb R^{d_\Y}}$, where $\phi:[0, \infty)\to\mathbb{R}$ is a radial function, such that $k$ is positive definite for $d_\X>0$.  This property can be equivalently characterized by requiring $t\mapsto\phi(\sqrt{t})$ to be completely monotone.  Notable examples satisfying this condition include the inverse multiquadrics $\phi(x) = \l(c^2 + x^2\r)^{-\beta}$ and the Gaussian kernel $\phi(x) = e^{-\alpha |x|^2}$ for any $c > 0$, $\beta > 0$, and $\alpha > 0$; see \citet{wendland2004scattered}. Let $\H_k$ denote the RKHS induced by the matrix-valued kernel $k$.  With a slight abuse of notation, we define the prediction error as
\[
\E(f)=\E(h):=\be\l[\l\|h(x)-y\r\|_{\Y}^2\r]=\be\l[\l\|\widehat{\mathcal{D}}_{d_\Y}^{\Y}\circ f\circ\widehat{\mathcal{E}}_{d_\X}^{\X}(x)-y\r\|_{\Y}^2\r], \quad \forall f \in \H_K,
\]
where $h = \widehat{\mathcal{D}}_{d_\Y}^{\Y}\circ f\circ\widehat{\mathcal{E}}_{d_\X}^{\X}$. Using the identity  $(\widehat{\mathcal{D}}_{d_\X}^{\X})^*=\widehat{\mathcal{E}}_{d_\X}^{\X}$,
we compute the Fr$\acute{e}$chet derivative of $\E(f)$ in $\H_k$, and obtain
\[
\nabla\E(f)= 2 \be\l[\phi\l(\l\|\widehat{\mathcal{E}}_{d_\X}^{\X}(x)-\cdot\r\|\r)\widehat{\mathcal{E}}_{d_\Y}^{\Y}\l(\widehat{\mathcal{D}}_{d_\Y}^{\Y}\circ f\circ\widehat{\mathcal{E}}_{d_\X}^{\X}x-y\r)\r].
\]
Based on samples, we derive the regularized SGD iteration with $f_1=0$, and
\[
f_{t+1}=f_{t}- \eta\l(
\phi\l(\l\|\widehat{\mathcal{E}}_{d_\X}^{\X}(x_t)-\cdot\r\|\r)\l(f_t\l(\widehat{\mathcal{E}}_{d_\X}^{\X}x_t\r)-\widehat{\mathcal{E}}_{d_\Y}^{\Y}y_t\r)+\lambda f_t\r),
\]
where $\eta$ is the step size. This can be interpreted as an SGD scheme based on the samples $\{\widehat{\mathcal{E}}_{d_\X}^{\X}x_t, \widehat{\mathcal{E}}_{d_\Y}^{\Y}y_t\}_{t=1}^T$ in $\H_k$. Accordingly, we define $h_1=0$ and 
\begin{equation} \label{sgd-pca}
    h_{t+1}=h_t-\eta\l(\phi\l(\l\|\widehat{\mathcal{E}}_{d_\X}^{\X}(x_t-\cdot)\r\|\r)\widehat{P}^{\Y}_{d_\Y}(h_t(x_t)-y_t)+\lambda h_t\r),
\end{equation}
where $\widehat{P}^{\Y}_{d_\Y}:=\widehat{\mathcal{D}}_{d_\Y}^{\Y}\circ\widehat{\mathcal{E}}_{d_\Y}^{\Y}$ is the empirical projection operator.

Under suitable assumptions, we can assert that, $\phi\l(\l\|\widehat{\mathcal{E}}_{d_\X}^{\X}(x-\cdot)\r\|\r)\widehat{P}^{\Y}_{d_\Y}$  converges to $\phi(\|x-\cdot\|_\X)I_\Y$ as $d_\X$ and $d_\Y$ tend to $\infty$. Hence, the SGD iteration \eqref{sgd-pca} can be approached by the following scheme with 
$\widetilde{h}_1=0$, and  
\[
\widetilde{h}_{t+1}=\widetilde{h}_t-\eta\l(\phi(\|x_t-\cdot\|_\X)(\widetilde{h}_t(x_t)-y_t)+\lambda \widetilde{h}_t\r),
\]
which is the setting analyzed in this paper. 

To summarize, using the PCA encoder-decoder as a concrete example, we see that our analysis can seamlessly align with the encoder-decoder framework. Rigorous proofs will be provided in our future work.

\section{Detailed Setup of the Numerical Experiments}\label{Detailed setup}

We describe the matrix realization used in the numerical experiments. The input variable is represented through the first $m$ basis functions
\[
e_k(u)=\sqrt{2}\cos(k\pi u),\qquad u\in[0,1],\quad 1\le k\le m,
\]
and generated by
\[
x=\sum_{k=1}^m \sqrt{\mu_k}\,\xi_k e_k,
\qquad
\mu_k=k^{-\alpha},
\]
where $\xi_k\sim\mathcal N(0,1)$ are independent. We consider the scalar kernel
\[
\mathcal K(x,x')
=
\sum_{k\ge1}\nu_k\langle x,e_k\rangle_{L^2}\langle x',e_k\rangle_{L^2},
\qquad
\nu_k=k^{-\beta},
\]
and the operator-valued kernel $K(x,x')=\mathcal K(x,x')I_{\Y}$ with $\Y=L^2[0,1]$.

Let $\{\psi_k\}_{k\ge1}$ denote the canonical orthonormal basis of $\H_{\mathcal K}$,
\[
\psi_k=\sqrt{\nu_k}\langle \cdot,e_k\rangle_{L^2},
\]
so that
\[
\phi(x)=\mathcal K(x,\cdot)=\sum_{k=1}^m \sqrt{\mu_k\nu_k}\,\xi_k\psi_k.
\]
The covariance operator $C:\H_{\mathcal K}\to\H_{\mathcal K}$ is diagonal in this basis:
\[
C=\sum_{k=1}^m \mu_k\nu_k\,\psi_k\otimes\psi_k.
\]

We choose the target operator as
\[
H^\dagger=S^\dagger C^r,
\qquad
S^\dagger\psi_k=e_k,\quad 1\le k\le m.
\]
Then
\[
h^\dagger(x)=H^\dagger\phi(x),
\qquad
\epsilon_t=\frac{\sigma}{\sqrt m}\sum_{k=1}^m \zeta_k^t e_k,
\]
where $\zeta_k^t\sim\mathcal N(0,1)$ are independent, and hence
\[
x_t=\sum_{k=1}^m \sqrt{\mu_k}\,\xi_k^t e_k,
\qquad
y_t=\sum_{k=1}^m\left((\mu_k\nu_k)^{r+1/2}\xi_k^t+\frac{\sigma}{\sqrt m}\zeta_k^t\right)e_k.
\]

Writing the estimator in matrix form as
\[
H_t=\sum_{k_1,k_2=1}^m H^{(t)}_{k_1,k_2}\,e_{k_1}\otimes\psi_{k_2},
\qquad
H^{(t)}=\bigl(H^{(t)}_{k_1,k_2}\bigr)_{1\le k_1,k_2\le m},
\]
and introducing
\[
A=\mathrm{diag}\bigl(\sqrt{\mu_1\nu_1},\dots,\sqrt{\mu_m\nu_m}\bigr),
\qquad
B=\mathrm{diag}\bigl((\mu_1\nu_1)^{r+1/2},\dots,(\mu_m\nu_m)^{r+1/2}\bigr),
\]
together with
\[
z^t=A\xi^t,
\qquad
y^t=B\xi^t+\frac{\sigma}{\sqrt m}\zeta^t,
\]
where
\[
\xi^t=(\xi_1^t,\dots,\xi_m^t)^\top,
\qquad
\zeta^t=(\zeta_1^t,\dots,\zeta_m^t)^\top,
\]
the regularized SGD iteration becomes
\[
H^{(1)}=0,
\qquad
H^{(t+1)}
=
(1-\eta_t\lambda_t)H^{(t)}
-
\eta_t\bigl(H^{(t)}z^t-y^t\bigr)(z^t)^\top.
\]

Finally, letting
\[
H_0=\mathrm{diag}(\mu_1^r\nu_1^r,\dots,\mu_m^r\nu_m^r),
\]
the prediction and estimation errors are computed by
\[
\E(h_{t+1})-\E(h^\dagger)=\|(H^{(t+1)}-H_0)A\|_{\mathrm F}^2,
\qquad
\|h_{t+1}-h^\dagger\|_{\H}^2=\|H^{(t+1)}-H_0\|_{\mathrm F}^2,
\]
where $\|\cdot\|_{\mathrm{F}}$ denotes the Frobenius norm of a matrix.

This completes the numerical setup used for the experiments in the main text, for which we take
\[
r=1.0,\qquad m=200,\qquad \alpha=0.3,\qquad \beta=0.9,\qquad \sigma=0.1.
\]

\bibliographystyle{plainnat}
\bibliography{reference}

\end{appendix}

\end{document}